\documentclass{article} % For LaTeX2e
\usepackage{iclr2026_conference,times}

\usepackage{amsthm}
\usepackage{amsmath,amsfonts,bm}

\def\eqref#1{equation~\ref{#1}}
\def\1{\bm{1}}

\def\mD{{\bm{D}}}

\def\mL{{\bm{L}}}

\DeclareMathAlphabet{\mathsfit}{\encodingdefault}{\sfdefault}{m}{sl}
\SetMathAlphabet{\mathsfit}{bold}{\encodingdefault}{\sfdefault}{bx}{n}

\newcommand{\softmax}{\mathrm{softmax}}

\newcommand{\KL}{D_{\mathrm{KL}}}

\def\wrt{\mbox{w.r.t.\space}}
\def\ie{\mbox{\textit{i.e.}}}
\def\eg{\mbox{\textit{e.g.}}}

\newtheorem{thm}{Theorem}

\newtheorem{lemma}{Lemma}
\newtheorem{remark}{Remark}
\newtheorem{ass}{Assumption}

\usepackage{amsmath}

  \def\mD{{\mathcal D}}

  \def\mL{{\mathcal L}}

  \DeclareMathAlphabet\mathbfcal{OMS}{cmsy}{b}{n}
  \def\0{{\bf 0}}
  \def\1{{\bf 1}}

  \def\bx{{\bf x}}

  \def\bx{{\bf x}}

\def\eg{\emph{e.g.}} 
\def\ie{\emph{i.e.}} 
 
\def\etc{\emph{etc.}} 
\def\wrt{{w.r.t.~}}

\usepackage{hyperref}
\hypersetup{
    colorlinks=true,
    linkcolor=red,
    citecolor=teal,
    urlcolor=cyan,
}
\usepackage{url}
\usepackage{xspace}
\usepackage{threeparttable}
\usepackage{multicol,multirow}
\usepackage{booktabs} % for professional tables
\usepackage{graphicx}
\usepackage[table]{xcolor}   % 导言区引入
\usepackage{wrapfig}
\usepackage{subfigure}

\usepackage{pifont}
\let\oldding\ding%
\renewcommand{\ding}[2][1]{\scalebox{#1}{\oldding{#2}}}%

\usepackage{algorithm}
\usepackage[algo2e]{algorithm2e}

\usepackage{listings}
\usepackage{caption}  % 让 lstlisting 的标题样式更好看（可选）
\usepackage{enumitem}
\definecolor{codeblue}{rgb}{0.25,0.5,0.5}
\newcommand{\smallbullet}{\raisebox{0.25ex}{\scalebox{0.5}{$\bullet$}}}
\usepackage{etoc}    
\etocdepthtag.toc{mtchapter}     % 导言区加
\usepackage{xcolor}              % colors
\usepackage{tcolorbox}           % 盒子环境

\usepackage[colorinlistoftodos,prependcaption,textsize=small]{todonotes}

\lstdefinestyle{pseudopy}{
  language=Python,
  basicstyle=\fontsize{7pt}{7pt}\ttfamily\bfseries,
  commentstyle=\color{codeblue}\fontsize{7pt}{7pt},
  keywordstyle=\fontsize{7pt}{7pt},
  columns=fullflexible,
  keepspaces=true,
  breaklines=true,
  showstringspaces=false,
  frame=tb,              % 上下细框（按需）
  numbersep=6pt,
  belowskip=0.6\baselineskip,
  aboveskip=0.6\baselineskip
}

\newcommand{\blue}[1]{{\color{black}#1}}

\newcommand{\methodname}{ZeroSiam\xspace}
\newcommand{\drop}[1]{\textbf{\textcolor{purple!60}{#1}}}

\title{ZeroSiam: An Efficient Asymmetry for Test-Time Entropy Optimization without Collapse}
\author{%
Guohao Chen$^{1}$\thanks{Equal contribution. $^\dagger$Corresponding author. Source code at: \href{https://github.com/Cascol-Chen/ZeroSiam}{https://github.com/Cascol-Chen/ZeroSiam}.}~~,~
Shuaicheng Niu$^{12*}$,~
Deyu Chen$^{3}$,~
Jiahao Yang$^{3}$,~
Zitian Zhang$^{3}$~ \\
\textbf{Mingkui Tan$^{3*}$,~
Pengcheng Wu$^{1}$,~
Zhiqi Shen$^{1\dagger}$}~
\\
\texttt{\{guohao.chen, shuaicheng.niu, zqshen\}@ntu.edu.sg}
\\
Nanyang Technological University$^{1}$ ~
Joint WeBank-NTU Research Institute on Fintech$^{2}$ \\
South China University of Technology$^{3}$ \\
}

\iclrfinalcopy % Uncomment for camera-ready version, but NOT for submission.
\begin{document}

\maketitle

% 同时在vision llm领域起效，在各种challenging 的tiny模型上都表现很好，甚至于可以将一个已崩溃的模型救活
% tiny模型under high risk collapse
\begin{abstract}
Test-time entropy minimization helps adapt a model to novel environments and incentivize its reasoning capability, unleashing the model's potential during inference by allowing it to evolve and improve in real-time using its own predictions, achieving promising performance.
However, pure entropy minimization can favor non-generalizable shortcuts, such as inflating the logit norm and driving all predictions to a dominant class to reduce entropy, risking collapsed solutions (\eg, constant one-hot outputs) that trivially minimize the objective without meaningful learning.
% In this paper, we reveal asymmetry as a key mechanism for collapse prevention and introduce \methodname--an efficient asymmetric Siamese architecture tailored for test-time entropy minimization. \methodname efficiently achieves asymmetry with a learnable predictor and a stop-gradient operator before the classifier. 
\blue{In this paper, we reveal asymmetry as a key mechanism for collapse prevention and introduce \methodname--an efficient asymmetric Siamese architecture tailored for test-time entropy minimization. \methodname prevents collapse through asymmetric divergence alignment, efficiently achieved by a learnable predictor and a stop-gradient operator before the classifier.}
% In this paper, we introduce \methodname, an efficient asymmetric Siamese architecture tailored for test-time entropy minimization. \methodname prevents collapse through asymmetric divergence alignment, which is efficiently achieved by a learnable predictor and a stop-gradient operator before the classifier.
% We provide empirical and theoretical evidence that \methodname not only prevents collapse solutions, but also absorbs and regularizes biased learning signals, enhancing performance even when no collapse occurs.
We provide empirical and theoretical evidence that \methodname not only prevents collapse, but also regularizes biased learning signals, enhancing performance even when no collapse occurs. 
Despite its simplicity, extensive results show that \methodname performs more stably over prior methods using negligible overhead, demonstrating efficacy on both vision adaptation and large language model reasoning tasks across challenging test scenarios and diverse models, including particularly collapse-prone tiny models.

\end{abstract}

% 熵最小化方法有坍塌风险、为了提升Tent的稳定性，现有方法引入多次增广一致性、样本筛选、predefined uncertainty mass之类的方法，计算代价高或依赖仔细的selection threshold调试。但trival解仍存在，因此现有方法跨模型、跨场景仍出现坍塌现象。
% In self-supervised learning xxx
% from achitecture design remains underexplored
% 两者的崩溃问题是能联系，但是Simsiam都是aligning-based self-supervised，怎么用到entropy上
% 是怎么受Simsiam启发的，minimal additions, stop-gradient和非对称孪生网络结构
% simple, efficient, and stable TTA without collapse . Minimal overhead (zero augmentations, zero extra backbone forward passes).
% create asymetric learning structure to prevent collapse with minimal cost.
% Why it matters
\section{Introduction}

% Entropy Minimization 入手, 各个里边都有, RL, LLM里边都有 (Entropy 重要)
% 待解决: Test-time Entropy什么时候引入
% Test-Time Entropy, noisy, TTA, RL, 都有collapse风险
% TTA怎么做,启发式,可能发生剃掉了, 还是没有一个很好的解决方式
% Self-supervised的做法,解决特征层面,多分支,不好 不方便涌过来 (no augmentation, single-branch) 他们还不行
% 我们想怎么解决, 我们提出ZeroSiam

% LLM xx alignment

% % Entropy minimization is a fundamental principle that encourages models to make confident predictions by reducing output uncertainty. It has proven effective across domains, from semi-supervised learning and domain adaptation to computer vision, natural language processing, and speech, where it improves robustness under limited labels, distribution shifts, and noisy inputs. Its simplicity and versatility make it a cornerstone for building adaptive and reliable learning systems.

Entropy measures the uncertainty of a model’s predictions. Since its emergence, entropy optimization has been widely adopted as an auxiliary objective alongside supervised, reinforcement signals, \etc~To be specific, in semi-supervised learning it enforces low-entropy predictions on unlabeled data to refine decision boundaries~\citep{grandvalet2004semi,lee2013pseudo,sajjadi2016regularization}; in domain adaptation it mitigates distribution shifts by promoting confident outputs~\citep{jiang2007instance,long2016unsupervised,morerio2017minimal}; and in reinforcement learning it is often maximized to encourage exploration or minimized to ensure deterministic policies~\citep{ziebart2008maximum,ahmed2019understanding}, among others, showing its remarkable success in boosting the learning effectiveness.

% % in language models it helps stabilize training and improve prediction consistency;
% % These demonstrate that entropy-based objectives have achieved remarkable success in scenarios with limited labels or noisy signals, significantly boosting the effectiveness of learning.

Of late, \textbf{test-time} entropy minimization has attracted growing interest, as it operates purely unsupervised during inference, without relying on any ground-truth supervision. For example, in test-time adaptation (TTA)~\citep{wang2021tent,lee2024deyo} it encourages confident predictions to mitigate domain shifts on unseen data streams; in large language models, it helps calibrate uncertainty and improve prediction consistency for alignment~\citep{Hu2025TLM,jang2025self}. These advances highlight the potential of entropy minimization in unleashing model power during inference, enabling adaptation to novel or out-of-distribution domains and moving toward greater intelligence.

% (cgh version)
% \cgh{However, since entropy minimization reaches minima when the model produces one-hot outputs regardless of inputs, relying solely on entropy as the learning objective can lead to collapsed solutions (\eg, constant one-hot outputs) and degraded performance. This issue becomes even more pronounced during testing, where models are often sensitive or uncertain when faced with noisy real-world data, domain shifts, or novel environments, and minimizing prediction entropy leads to further bias and error accumulation. For instance, after test-time adaptation, models can deteriorate to predicting all samples to a single class under challenging wild testing scenarios~\citep{niu2023sar}.}

However, entropy minimization naturally drives the model to increase the maximum predicted logit, regardless of whether it aligns with the true label. In practice, testing often involves noisy real-world data, domain shifts, or novel environments, where models often tend to be highly sensitive and uncertain, and the predicted maximum class is frequently incorrect. In this sense, minimizing entropy alone does not ensure that the corresponding supervised loss (\eg, cross-entropy) on the target task is also reduced. Instead, the model can easily exploit a shortcut solution by producing near one-hot outputs for all inputs, which trivially minimizes entropy but fails to capture meaningful predictions, leading to degraded performance. For instance, after TTA, models can deteriorate to predicting all samples to a single class under challenging wild testing scenarios~\citep{niu2023sar}.

Various TTA methods have been proposed to tackle the above issue, such as filtering unreliable gradients using predefined thresholds~\citep{niu2022eata,lee2024deyo,Hu2025TLM}, or mitigating disturbances through optimization of a sharpness-aware loss surface~\citep{niu2023sar}. However, the use of heuristic thresholds for gradient selection remains problematic, as such thresholds are difficult to define and generalize across domains. In this sense, filtering or suppressing noisy gradients can only be partial, and the model remains exposed to trivial solutions by optimizing with the remaining gradients, for which entropy can still be minimized by collapsing into constant one-hot outputs regardless of the inputs. As a result, it is unavoidable that these methods still risk collapse during deployment, particularly under prolonged or more challenging testing scenarios, as in Tables~\ref{tab:imagenet-c-mix-shift} \& \ref{tab:imagenet-c-blindspot}.

% In this paper, we resolve the collapse risk of test-time entropy minimization by drawing inspiration from negative-free self-supervised learning (SSL)~\citep{chen2021exploring,zhang2022avoid}, 
\blue{In this paper, we reveal asymmetry as a key mechanism for collapse prevention, which was introduced in the negative-free self-supervised learning (SSL)~\citep{chen2021exploring,zhang2022avoid} that seeks to learn meaningful representations by maximizing the similarity between two augmented views of the same image. Negative-free SSL uses an \textit{asymmetric structure} to prevent the outputs of \textit{two branches} from collapsing toward the same constant, resolving collapse from architectural design. However, asymmetric structure is tailored for representation learning during pre-training, and naively applying it to our context is infeasible or inferior, as: 1) test-time entropy minimization typically has only one prediction branch and optimizes entropy instead of similarity; 2) traditional Siamese design requires extra backbone passes, which impairs efficiency for our test-time learning.}
% In this paper, we reveal asymmetry as a key mechanism for collapse prevention, which was originally developed in the negative-free self-supervised learning (SSL)~\citep{chen2021exploring,zhang2022avoid}—learning meaningful representations by maximizing the similarity between two augmented views of the same image. Negative-free SSL introduces an \textit{asymmetric structure} to prevent the \textit{two branches} from collapsing toward the same constant, resolving collapse from the architecture design. However, asymmetric structure is originally developed for representation learning, and naively applying asymmetry to our context is infeasible or inferior, as: 1) test-time entropy minimization typically has only one prediction branch and optimizes entropy instead of similarity; 2) traditional Siamese design requires extra backbone passes, which impairs efficiency for our test-time learning.

\begin{figure*}[t]
    \centering
    \vspace{-0.1in}
    \includegraphics[width=1.\linewidth]{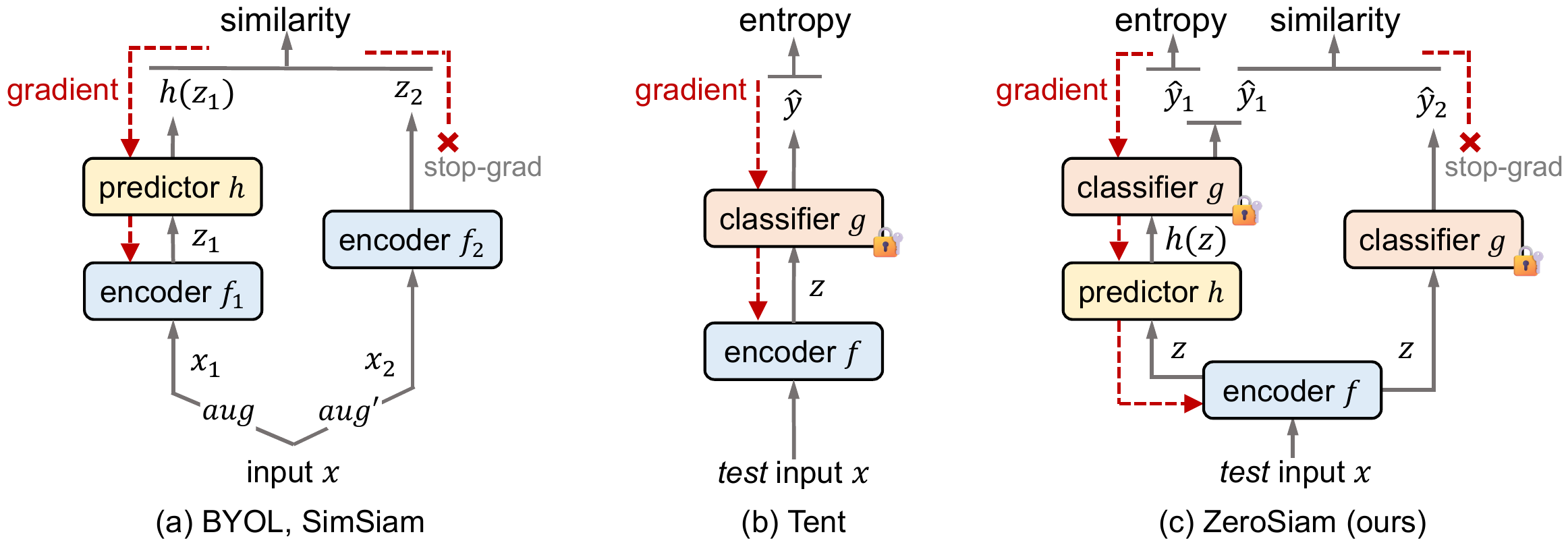}
    \vspace{-0.19in}
    \caption{Comparisons on architectures. (a) Alignment-oriented SSL methods (BYOL~\citep{grill2020bootstrap}, SimSiam~\citep{chen2021exploring}). (b) Test-time entropy minimization (Tent)~\citep{wang2021tent}. (c) Our \methodname, which designs a minimal asymmetry for entropy minimization with a lightweight predictor and a stop-gradient branch—without augmentations, extra encoder passes, or teacher models—to substantially enhance learning stability and boost performance while retaining efficiency.
    % `GP' and `SGP' denote gradient path and stop-gradient path, respectively.
    % The predictor in \methodname is initialized as identity to enable fully TTA.
    }
    \label{fig:architecture-comparisons}
    \vspace{-0.15in}
\end{figure*}

% Therefore, we design \methodname, demonstrating how asymmetry can be efficiently implemented in test-time entropy minimization, without using augmentations, extra backbone passes, and teacher models. The core idea in \methodname is to decouple a prediction into two asymmetric outputs based on the same feature: an online branch with a learnable predictor and a target branch with stop-gradient, where we minimize entropy on the online branch to learn discrimination, while performing asymmetric predictor–target alignment to prevent collapsed constants. We investigate both the empirical and theoretical behaviors of \methodname during TTA in Section~\ref{sec:empric_theory_evidence}, highlighting that our architecture can both avoid collapse and regularize biased learning signals at test time, which improves the performance of entropy optimization even when no collapse occurs. \methodname achieves notable stability and effectiveness with a single extra predictor, requiring negligible overhead (see Table~\ref{tab:efficiency_comparison}). Despite its simplicity, \methodname achieves state-of-the-art results across diverse architectures and a wide range of test scenarios, \eg, adapting reliably even with \textit{incorrect} pseudo labels (see Table~\ref{tab:imagenet-c-blindspot}).

\blue{Therefore, we design \methodname, demonstrating how asymmetry can be efficiently implemented in test-time entropy minimization without using augmentations, extra backbone passes, and teacher models.} To embed asymmetry within a single backbone pass, \methodname decouples a prediction into two asymmetric outputs based on the same feature: an online branch with a learnable predictor and a target branch with stop-gradient. As shown in Figure~\ref{fig:architecture-comparisons}, we minimize entropy on the online branch to learn discriminative features, while performing asymmetric predictor–target alignment to prevent collapsed constants. 
% We investigate \methodname's behaviors empirically and theoretically during TTA in Section~\ref{sec:empric_theory_evidence}, highlighting that our architecture can both prevent collapse and regularize biased shortcut learning signals at test time, which improves the performance of entropy optimization even when no collapse occurs.
We investigate \methodname's behaviors empirically and theoretically during TTA in Section~\ref{sec:empric_theory_evidence}, demonstrating that our architecture avoids collapse and is also meaningful in \blue{regularizing biased shortcut learning signals} at testing, which improves the performance of entropy optimization even when no collapse occurs. 
% \methodname achieves stable and effective entropy minimization using only a single predictor, introducing negligible overhead (see Table~\ref{tab:efficiency_comparison}).
\blue{\methodname achieves notable stability and effectiveness with a single extra predictor, introducing negligible overhead (see Table~\ref{tab:efficiency_comparison}).}
Despite its simplicity, \methodname yields superior robustness compared to prior methods across diverse architectures and a wide range of test scenarios, \eg, adapting reliably even with \textit{incorrect} pseudo labels (see Table~\ref{tab:imagenet-c-blindspot}).

% Therefore, we design \methodname, an efficient asymmetric Siamese architecture tailored for test-time entropy minimization, using zero augmentations and extra backbone passes. To embed asymmetry within a single backbone pass, \methodname decouples a prediction into two asymmetric outputs based on the same feature: an online branch with a learnable predictor and a target branch with stop-gradient. As shown in Figure~\ref{fig:architecture-comparisons}, we minimize entropy on the online branch to learn discriminative features, while performing asymmetric predictor–target alignment to prevent collapsed constant solutions. We investigate \methodname's behaviors empirically and theoretically during TTA in Section~\ref{sec:empric_theory_evidence}, demonstrating that our architecture avoids collapse and is also meaningful in absorbing and regularizing biased learning signals at testing, which improves the performance of entropy optimization even when no collapse occurs. \methodname achieves stable and effective entropy minimization using only a single predictor, introducing negligible overhead (see Table~\ref{tab:efficiency_comparison}). Despite its simplicity, \methodname yields superior robustness compared to prior methods across diverse architectures and a wide range of test scenarios, \eg, adapting reliably even with \textit{incorrect} pseudo labels (see Table~\ref{tab:imagenet-c-blindspot}).

\textbf{Main Novelty and Contributions}:
% Our \textbf{main novelty and contributions} are summarized as follows:
% 1) We propose \methodname, a simple yet effective asymmetric architecture for test-time entropy minimization without collapse, while requiring zero augmentations, no extra backbone passes, and no teacher models to achieve efficient online deployment.
\blue{1) We are the first to study asymmetric structure in TTA and propose \methodname, revealing how asymmetry can be efficiently implemented in test-time entropy minimization for collapse prevention without using augmentations, extra backbone passes, and teacher models.}
2) We provide empirical and theoretical insights into \methodname's behaviors (Section~\ref{sec:empric_theory_evidence}), demonstrating that \methodname also helps absorb and regulate non-generalizable shortcut learning signals at testing, which enhances test-time learning performance even when no collapse occurs.
3) Extensive experiments on both language and vision tasks across both transformer and CNN models with varying sizes and a wide range of challenging test scenarios verify our efficacy.

\section{Preliminary and Problem Statement}

We briefly revisit test-time entropy minimization in this section for the convenience of our method presentation and put detailed related
work discussions into Appendix~\ref{sec:related-work} due to page limits.

\textbf{Test-Time Entropy Minimization} 
Formally, given any model with its encoder $f(\cdot;\theta_f)$ and classifier $g(\cdot;\theta_g)$ trained on source data $\mD_{train}$, test-time entropy minimization conducts model learning directly from the testing data $\mD_{test} = \{\bx_j\}_{j=1}^{M}$ by optimizing 
\begin{equation}\label{eq:entropy_minimization}
\min_{\tilde{\theta}{\subset}\{\theta_f,\theta_g\}}\mL (\bx; \theta_f, \theta_g) = - \sum_c p(\hat{y}_c) \log p(\hat{y}_c),~~~\text{where}~~~p(\hat{y}_c)=g(f(\bx;\theta_f); \theta_g)[c].  
\end{equation}
Here, $p(\hat{y}_c)$ denotes the predicted probability for class $c$, and $\tilde{\theta}$ represents the learnable parameters. 
Based on Eqn.~(\ref{eq:entropy_minimization}), test-time adaptation methods~\citep{wang2021tent,zhang2021memo,marsden2024universal,zhang2025come} often leverage it to adapt a pretrained model to novel domains under potential distribution shifts, \textit{known as} out-of-distribution generalization. In natural language processing, it has also been employed for large language model alignment and improving predictive performance~\citep{Hu2025TLM,jang2025self}. 
By using the model’s own predictive entropy as a self-supervised signal during inference, Eqn.~(\ref{eq:entropy_minimization}) requires neither source data nor modifications to the training process, making it a practical and lightweight solution for real-world deployment.

% % entropy minimization公式，可以用来干嘛，比如说TTA，比如说reasoning
% % 很多用途，我们focus vision tta, LLM tta
% % 还可以讲一下很难选，为什么之前的方法没有根本性地解决崩溃
% We focus on the problem of fully \textbf{t}est-\textbf{t}ime \textbf{a}daptation (TTA). (Entropy Minimization) 
% % Formally, given any model $f(\cdot;\theta)$ trained on source data $\mD_{train} = \{(\bx_i,y_i)\}_{i=1}^{N}$, 
% Formally, given any model with its encoder $f(\cdot;\theta_f)$ and classifier $g(\cdot;\theta_g)$ trained on source data $\mD_{train} = \{(\bx_i,y_i)\}_{i=1}^{N}$, 
% fully TTA adapts $g\circ f$ to testing data $\mD_{test} = \{\bx_j\}_{j=1}^{M}$ with potential distributions shifts from $\mD_{train}$ online by minimizing some unsupervised objectives $\mathcal{L}_u$ on test data, \ie, $\min_{\tilde{\theta}}\mL_{u} (g\circ f(\bx))$,
% where $\tilde{\theta}{\subset}\{\theta_f,\theta_g\}$ is learnable parameters during TTA. In fully TTA, $\mL_{u}$ is typically prediction entropy~\citep{wang2021tent, niu2023sar}, marginal entropy~\citep{zhang2021memo}, and other entropy-based variants~\citep{marsden2024universal,zhang2025come}. 
% By using the model's own predicted entropy as TTA signals, fully TTA requires no source data or modifications to training, making it practical for real-world deployment.
% % By using the model's own predicted entropy as adaptation signals, fully TTA requires no source data or retraining, making it simple and practical for real-world deployment.

\textbf{Problem Statement and Motivation} As shown in Figure~\ref{fig:motivation-figures} (c-d), model learning with unsupervised entropy minimization on test data may favor non-generalizable shortcuts, such as (i) inflating the logit norm to reduce entropy, or (ii) aligning all logits towards a dominant mode. Such updates converge toward collapsed trivial solutions (\eg, constant one-hot outputs), causing performance degradation. 
This phenomenon becomes particularly severe in challenging test scenarios or when using weaker base models (\eg, ConvNeXt-Tiny), where the lower source accuracy makes the resulting entropy gradients substantially less reliable, thereby increasing the risk of collapse. Existing methods mainly seek to improve TTA stability with heuristic thresholds to remove some unreliable updates~\citep{niu2023sar,lee2024deyo}, but 
\textit{the objective is yet trivially minimized by collapsed solutions}. 
% but the \textit{objective still admits collapsed constant solutions}.
% \textit{yet trivial solutions persist in the entropy loss landscape}. 
Consequently, prior methods remain exposed to collapse and are sensitive to architectures and domains (see Tables~\ref{tab:imagenet-c-label-shift}-\ref{tab:imagenet-c-blindspot}).
% , as verified in Tables~\ref{tab:imagenet-c-label-shift}-\ref{tab:imagenet-c-blindspot}. 
This gap motivates us to ask: 
\textit{How can we design an efficient, theoretically grounded entropy optimization mechanism that inherently avoids undesired trivial solutions?}
% \textit{How can we design an efficient, theoretically grounded TTA mechanism that inherently disfavors trivial solutions by construction?}

% \makebox[\linewidth][s]{theoretically grounded TTA mechanism that inherently disfavors trivial solutions by construction?}}
% \textit{How can we design a theoretically grounded TTA mechanism that inherently rules out trivial solutions at minimal cost?}

% non-generalizable and undesirable shortcuts

% 熵最小化是unbounded的，容易学到一些难泛化的知识，logit 范数膨胀、常量预测，导致TTA过程的稳定性较差。我们希望引入一种极简的Siamese网络，不引入额外数据增广、额外前向传播，并给出理论证明，从而高效地、有理论保障地解决collapse问题

% developing a minimal Siamese architecture for asymmetry learning in entropy-based TTA, namely \methodname. 

% In this paper, we propose an efficient asymmetric Siamese tailored for entropy-based TTA to prevent collapse, namely \methodname. 

% Based on the above analyses, our central idea to prevent collapse risk is removing trivial minima from construction for stable adaptation. To this end, we design \methodname as a minimal asymmetric Siamese network, inspired by the asymmetric stabilization principle from self-supervised learning, and tailor it to address the unsupervised TTA.

% constant prediction bias
% norm-inflation bias
% 展示四个图Acc, 偏置捕捉（LogitNorm不一定要放）, logitnorm, |R|/|M|, w/ ZeroSiam & w/o ZeroSiam， ，参考SAR
% 展示predictor能够捕捉norm放大、能捕捉偏置（label shift和normal下不太一样）
% nice property，可以在噪声数据上调lr
% 这里也还是写从Figure2 motivated出发进行设计
% \section{Method}
% 只能说we bridge
% Intuition SSL，假设unsupervised learning也可以利用asymmetric来打破鞍点
\section{\methodname: Stable Adaptation via Minimal Asymmetric Network}
We achieve the goal of efficient test-time entropy optimization without collapse by designing a lightweight asymmetric Siamese architecture for entropy minimization, namely \methodname. We draw inspiration from the traditional asymmetric design for similarity learning in SSL, but extend its compatibility to the single-branch entropy minimization via a single learnable predictor—without requiring augmentations and extra backbone passes.
% —substantially enhancing stability while maintaining efficiency.
We depict the design choice of our \methodname in Section~\ref{sec:zerosiam-architecture}, and provide further theoretical and empirical evidence of \methodname's stability for TTA in Section~\ref{sec:empric_theory_evidence}.
The overall details of \methodname are illustrated in Figure~\ref{fig:architecture-comparisons} (c) and Algorithm~\ref{alg:zerosiam-code}.

% why it works感觉还是要再强调一下
\subsection{Minimal Asymmetric Siamese Architecture for Entropy Minimization}\label{sec:zerosiam-architecture}

% test-time entropy minimization tends to inflate the logit norm and drive all predictions toward a dominant class to trivially minimize the entropy objective, leading to collapsed constant solutions. SSL~\citep{chen2021exploring} prevents collapse constants in consistency learning by asymmetry: with a non-identity predictor $h$ on one branch, and a stop-gradient operation on the other (Figure~\ref{fig:architecture-comparisons}). 
% This ensures that degenerate constant outputs incur non-zero alignment loss, which prevents \textit{two branches} from collapsing toward the same constant, showing a promising solution. However, constructing asymmetry in test-time entropy optimization remains an open question due to its single-branch nature, non-pairwise objective, and efficiency requirements for online update. 

Unlike cross-entropy training supervised by ground-truth labels, test-time entropy minimization relies on the model’s own predicted labels for self-training. This, however, creates a shortcut solution, where the objective can be minimized simply by producing one-hot outputs irrespective of the input. 
\blue{As shown in Figure~\ref{fig:motivation-figures} (c-d), pure entropy minimization tends to inflate the logit norm and drive all predictions toward a dominant class to reduce the entropy naively, leading to model collapse into degenerate solutions (\ie, constant one-hot outputs). Similar issues appear in negative-free SSL, where the model learns trivial solutions via constant representations regardless of inputs to maximize representation similarity. SSL~\citep{chen2021exploring} prevents such collapse by introducing an asymmetry: with a non-identity predictor $h$ on one branch, and a stop-gradient operation on the other (see Figure~\ref{fig:architecture-comparisons}). 
This ensures that degenerate constant outputs incur non-zero alignment loss, preventing the output of \textit{two branches} from collapsing toward the same constant, showing a promising solution.}

Motivated by this, we hypothesize that an asymmetric mechanism, which remains under-explored in entropy optimization, could also prevent undesired trivial solutions and stabilize self-training. 
However, constructing asymmetry in test-time entropy optimization remains an open question: 
% (1)~Unlike contrastive SSL, unsupervised TTA (\eg, Tent) relies on \textit{one prediction branch}, making it non-trivial to introduce a Siamese asymmetric structure;
% (2) Entropy-based TTA optimizes entropy rather than similarity, offering no pairwise stop-gradient alignment that enables asymmetry as in SSL;
% (3)~Traditional Siamese designs necessitate extra backbone passes, which are incompatible for efficient online TTA.
(1)~Unlike contrastive SSL, entropy-based learning relies on \textit{one prediction branch}, making it non-trivial to introduce a Siamese asymmetric structure; 
(2) Its objective optimizes entropy rather than similarity, offering no pairwise stop-gradient alignment that enables asymmetry as in SSL;
(3)~Traditional Siamese designs require extra backbone passes, hindering the efficiency for our test-time learning.
% which are incompatible with our efficient online learning.
% (2)~entropy loss also operates on a single prediction, offering no pairwise stop-gradient alignment teacher that allows asymmetry as in SSL; 
% (2)~Entropy optimization also operates on a single prediction distribution, offering no pairwise stop-gradient alignment that enables asymmetry as in SSL; 
% (3)~naive Siamese designs would introduce additional forward passes or augmentations, undermining the efficiency required in online TTA. 

\begin{figure}[t]
\vspace{-0.15in}
    \centering
    \begin{minipage}{1.0\textwidth}
        \centering
\begin{algorithm}[H]\small
         \KwIn{Test samples $\mD_{test}=\{\bx_j\}_{j=1}^{M}$, encoder $f(\cdot;\theta_f)$, classifier $g(\cdot;\theta_g)$, inserted predictor $h(\cdot;\theta_h)$}
         
         % \KwOut{Predictions $\{\hat{p_c}_j\}_{j=1}^{M}$.}
         \KwOut{Predictions $\{p^r_j\}_{j=1}^{M}$.}
         
         Initialize predictor $h(\cdot;\theta_h)=\text{Identity}$ \tcp*{warm start for fully test-time learning}
         \For{$\bx_j \in \mD_{test}$}{
          Compute encoder feature: $z = f(\bx_j;\theta_f)$ \tcp*{augmentation-free, single-pass}
          
          Compute target~~branch outputs $u^r = g(z;\theta_g)$, $p^r=\softmax(u^r)$ \tcp*{stop-gradient}

          Compute online branch outputs $u^o = g(h(z;\theta_h);\theta_g)$, $p^o=\softmax(u^o)$ ;

          Update $\theta_h$ and normalization params $\tilde{\theta} \subset\theta_f$ via Eqn.~(\ref{eq:zerosiam-loss}).
         }
         \caption{
         \textit{ZeroSiam}: Test-Time Asymmetric Entropy Minimization.
         % \textit{ZeroSiam}: Minimal Asymmetric Siamese for TTA.
         % The pipeline of proposed \methodname.
         }
         \label{alg:zerosiam-code}
\end{algorithm}
    \end{minipage}%
\end{figure}

To bridge these gaps, our key idea is to embed asymmetry within a single forward pass. We achieve this by inserting a predictor and the stop-gradient operator before the classifier, which decouples a prediction into two \emph{asymmetric} views from the same feature: an online branch (through the predictor) for \textit{optimizing entropy}, and a target branch (the original logits) for \textit{asymmetric stop-gradient alignment}. In this way, it establishes an efficient asymmetric Siamese for entropy optimization, using no augmentations or extra backbone passes, yet still exploiting asymmetric predictor–target alignment to prevent collapsed constants and stabilize adaptation. Formally, let $f(\cdot;\theta_f)$ denote the encoder, $g(\cdot;\theta_g)$ the classifier, and $h(\cdot;\theta_h)$ a lightweight \textit{linear} predictor. Given a test sample $\bx$, \methodname computes the encoder feature $z=f(\bx;\theta_f)$ once, then defines two asymmetric branches:
% To overcome these challenges, our key idea is to embed asymmetry within a single forward pass. We achieve this by inserting a predictor and the stop-gradient operator before the classifier, which aims to decouple a prediction into two \emph{asymmetric} views from the same feature: an online branch (through the predictor) for optimizing entropy, and a target branch (the original logits) for stop-gradient alignment. In this way, it establishes an efficient asymmetric Siamese for entropy optimization, using no augmentations or extra backbone passes, yet still exploiting asymmetric predictor–target alignment to prevent collapsed constant solutions and stabilize adaptation. Formally, let $f(\cdot;\theta_f)$ denote the encoder, $g(\cdot;\theta_g)$ the classifier, and $h(\cdot;\theta_h)$ a lightweight \textit{linear} predictor. Given a test sample $\bx$, \methodname computes the encoder feature $z=f(\bx;\theta_f)$ once, then define two asymmetric branches:
\begin{align}
    u^r &= g(z;\theta_g), 
    && \text{(target branch, stop-gradient)} \\
    u^o &= g(h(z;\theta_h);\theta_g),
    && \text{(online branch)}.
\end{align}
Let $p^r = \text{softmax}(u^r)$ and $p^o = \text{softmax}(u^o)$.
The \methodname\ objective combines entropy minimization on the online branch 
with an alignment regularizer to the target branch:
\begin{align}
    \mathcal{L} 
    = H(p^o) + \alpha \, D\!\left(p^o \,\|\, \mathrm{sg}[p^r]\right),
    \label{eq:zerosiam-loss}
\end{align}
where $H(p) = -\sum_c p_c \log p_c$ is the prediction entropy, $D(\cdot\|\cdot)$ is a divergence (\eg, $\ell_2$ or KL)\footnote{\methodname uses symmetric KL loss as $D(\cdot\|\cdot)$ to penalize both under- and over-coverage of modes, \ie, $D(p\|q) = \KL(p\|q) + \KL(q\|p)$, and $\alpha$ is fixed to 1. Results with KL and reverse KL are put in Table~\ref{tab:zerosiam-generality}.}, 
and $\mathrm{sg}[\cdot]$ denotes stop-gradient. 
% \mata{Following Tent~\citep{wang2021tent}, we update only normalization parameters in $\theta_f$ and the predictor parameters $\theta_h$ during TTA. MOVE TO IMPL DETAILS? LLM also only norm layers?}
Here, $\theta_h$ is initialized as an identity to ensure a warm start, which quickly diverges during online learning, as shown in Figure~\ref{fig:motivation-figures} (a), creating the asymmetry necessary to prevent collapsed constant solutions. 
% breaks the stability of
% Interestingly, \blue{Table~\ref{tab:predictor_design} shows that even a \textit{randomly initialized} predictor helps prevent collapse in Tent~\citep{wang2021tent}}, highlighting the value of asymmetry itself, despite still underperforming a learnable predictor. 
\blue{Even when using a \textit{randomly initialized} predictor to introduce asymmetry, Table~\ref{tab:predictor_design} shows that \methodname helps prevent collapse in Tent~\citep{wang2021tent}, revealing asymmetry's value.}
% In fact, as shown in Table~X, we find that even a randomly initialized predictor can help prevent collapse in Tent, while a learnable predictor yields a significant further boost. 
In essence, the asymmetric predictor–stop-gradient alignment inherently rules out collapsed constant solutions as valid minima, thereby avoiding collapse and improving stability during online entropy minimization.
% by construction
Unlike BYOL and SimSiam, which rely on augmentations or extra backbone passes, \methodname shows that asymmetry can also be instantiated within pure online entropy minimization in a minimal and efficient manner.

% Empirical evidence of \methodname's stabilization effects. (a-c) record the ODD accuracy, logits $L_2$ norm, and center dominance in model predictions under a mild test scenario~\citep{wang2021tent}. (d) records the Frobenius distance between $\theta_h$ and the identity matrix under non-i.i.d. streams with varying imbalance ratios~\citep{niu2023sar}. Center dominance is calculated by $\|\overline{u}\|/\|u\|$ following~\citep{zhang2022avoid}. Experiments are conducted on ImageNet-C (Snow, level 5) with ResNet50-GN. For fair comparisons, \methodname and Tent use the same learning rate configuration.

\begin{figure*}[t]
    \centering
    % \vspace{-0.3in}
    \includegraphics[width=1.\linewidth]{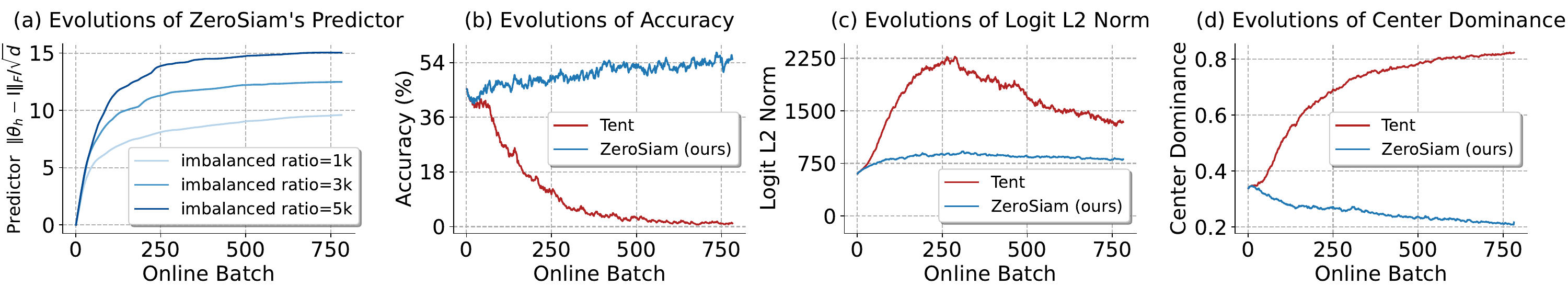}
    \vspace{-0.2in}
    \caption{Empirical evidence of \methodname's stabilization effects. (a) records the Frobenius distance between $\theta_h$ and the identity matrix under the non-i.i.d. streams with varying imbalance ratios~\citep{niu2023sar}. (b-d) record the ODD accuracy, logits $L_2$ norm, and center dominance in model predictions under a mild test scenario~\citep{wang2021tent}. Center dominance is calculated by $\|\overline{u}\|/\|u\|$ following~\citep{zhang2022avoid}. Experiments are conducted on ImageNet-C (Snow, level 5) with ResNet50-GN. For fair comparisons, \methodname and Tent use the same learning rate configuration.
    % (\ie, differing only in extra predictor and similarity loss).
    }
    \label{fig:motivation-figures}
    \vspace{-0.1in}
\end{figure*}
\subsection{How \methodname Resists Collapse: Empirical and Theoretical Insights}\label{sec:empric_theory_evidence}
% \methodname aims to establish a stable and efficient architecture for entropy-based TTA via asymmetry.
Although asymmetric designs (\eg, SimSiam for SSL) are known to prevent collapsed trivial solutions, test-time entropy minimization poses new challenges like logit norm inflation and single-class dominance. 
\blue{In this section, we rethink TTA stability from shortcuts, and find that \methodname not only inherits the anti-collapse effect of asymmetry, but also \textit{plays a broader role} in absorbing biased shortcuts and regularizing dynamics during entropy minimization. We present empirical evidence showing these stabilizing effects, and further explain \methodname's behaviors with theoretical insights.}

% We find that \methodname not only inherits the anti-collapse effect of asymmetry, \ie, eliminating trivial solutions, but also \textit{plays a broader role} in absorbing bias and regularizing learning dynamics during entropy minimization. 

\blue{\textbf{Rethinking TTA Stability from Shortcuts} During unsupervised TTA, the model can exploit shortcuts (\eg, assigning all predictions to a dominant class) to naively reduce the objective, resulting in model collapse, as shown in Figure~\ref{fig:motivation-figures}. Therefore, to improve TTA stability, we seek to identify what shortcuts are and suppress the exploitation of shortcuts during unsupervised learning. 
Instead of identifying shortcuts specifically for each objective, as in Figures~\ref{fig:motivation-figures} \&~\ref{fig:motivation-figures-suppl}, our key insight is that models with different sizes and architectures exploit similar shortcut patterns under the same objective. 
In this sense, we \textit{define shortcuts in an objective as patterns learnable by an extremely simple network (\eg, the predictor $h$)}, which naively reduce the loss value without improving generalization.
% In this sense, we propose to \textit{define shortcuts in an objective as patterns learnable by an extremely simple network}, which naively reduce the loss value without acquiring generalizable knowledge.
}
\vspace{-0.15in}

\textbf{Empirical Evidence of \methodname's Stabilization Effects} 
In Figure~\ref{fig:motivation-figures}, we visualize the effects of our \methodname under the unsupervised TTA. We conduct experiments with ResNet50-GN on ImageNet-C, and record the evolution of OOD accuracy, representative failure modes of entropy-based TTA, and changes in predictor parameters $\theta_h$ (\ie, a simple FC layer, learning preliminarily non-generalizable modes due to capacity). From the results, we highlight two key observations:

\textit{\textbf{Observation 1: Predictor as an Effective Absorber of Biased \blue{Shortcut} Signals.}} 
Figure~\ref{fig:motivation-figures} (a) shows that the predictor parameters $\theta_h$ deviate more rapidly and substantially from the identity mapping when facing online data streams with a higher label imbalance ratio~\citep{niu2023sar}, which produces more \blue{shortcut} signals for entropy minimization that may result in collapse. 
% Here, the large deviation actually acts as a shortcut direction that reduces the target entropy loss to trivial solutions (see also Theorem~\ref{proofs: thm1}). 
\blue{In this sense, this pronounced deviation thus suggests an active absorption of biased shortcut signals within predictor $h$ (an extremely simple network)}, adaptively posing stronger alignment regularization in Eqn.~(\ref{eq:zerosiam-loss}) to stabilize entropy adaptation under more collapse-prone TTA scenarios, as shown in Figure~\ref{fig:similarity_loss}.

% where such bias is then exposed in the online branch, resulting in a high alignment loss to the target branch, as shown in Figure~\ref{fig:similarity_loss} in Appendix.

% \textit{\textbf{Observation 1: Predictor as an Effective Absorber of Biased Learning Signals.}} Figure~\ref{fig:motivation-figures} (a) shows that the predictor parameters $\theta_h$ deviate more rapidly and substantially from the identity mapping when facing online data streams with increasingly imbalanced label distributions~\citep{niu2023sar}, which produces more biased signals for entropy minimization by naively mapping most inputs to a dominant class. Since the predictor $h$ is a single FC layer, it can only capture shortcuts that trivially reduce entropy along the collapsing direction (see Theorem~\ref{proofs: thm1}). This pronounced deviation thus implies an active absorption of biased learning signals in predictor $h$, which results in a higher alignment loss under biased streams (see Figure~\ref{fig:similarity_loss}), inducing an adaptive regularization strength.

% This pronounced deviation thus implies an active absorption of biased learning signals in predictor $h$, where such bias is then exposed in the online branch, resulting in a high alignment loss to the target branch (see Figure~\ref{fig:similarity_loss}).

% Meanwhile
\textit{\textbf{Observation 2: Asymmetry Suppresses Learning Non-Generalizable Collapse Modes.}} In Figure~\ref{fig:motivation-figures} (c), the logit $L_2$ norm under Tent suffers from a rapid inflation, whereas \methodname grows slowly and then stabilizes. In Figure~\ref{fig:motivation-figures} (d), Tent increasingly aligns logits toward a dominant mode (higher center dominance), while our \methodname suppresses such dominance over time.
% Consequently, Tent collapses into non-generalizable shortcuts, whereas \methodname mitigates them and maintains strong generalization, as reflected in the contrasting accuracy curves in Figure~\ref{fig:motivation-figures} (d).
Consequently, Tent collapses into non-generalizable shortcuts, whereas \methodname mitigates them and achieves better generalization as shown in Figure~\ref{fig:motivation-figures} (b), similarly even when no collapse occurs (see Figure~\ref{fig:motivation-figures-suppl}).

\blue{Overall, Figure~\ref{fig:motivation-figures} shows that \methodname absorbs and regulates non-generalizable shortcuts during TTA beyond collapse prevention. This reveals a distinct advantage of \methodname, and provides TTA-specific evidence on the role of asymmetry in generalization that extends prior findings from SSL.}

% Overall, we observe that the lightweight, linear design of the predictor $h$ makes it especially sensitive and effective to absorb non-generalizable shortcuts induced by entropy minimization. Meanwhile, the asymmetry in \methodname helps regulate such biases, yielding improved TTA efficacy and stability. This reveals a distinct advantage of \methodname for the unsupervised TTA, and provides TTA-specific evidence on the role of asymmetry in preventing collapse that extends prior self-supervised findings.

\textbf{Theoretical Insights of \methodname's Mechanism} We provide a theoretical analysis to explain \methodname's behaviors, showing that (i) the predictor enhances the online branch’s dynamism in exploring the parameter space, thereby facilitating more efficient entropy optimization;  (ii) the alignment term explicitly regulates the biased learning signals induced by entropy minimization, preventing the\linebreak model from being exposed to collapsed trivial solutions and meanwhile prompting performance.

% and (ii) the asymmetry between two branches regulates the biased learning signals absorbed by the predictor, suppressing non-generalizable shortcuts in entropy minimization and ensuring stability against collapse.

% the asymmetry in our \methodname (i) functions as a gradient filter that suppresses non-generalizable shortcuts in the unsupervised entropy minimization, and (ii) implicitly defines a trust region anchored by the target branch to resist collapse.
% imposes a per-step proximal constraint anchored by the target branch to mitigate collapse.

% \textbf{\textit{(Gradient Analysis)}}

% theoretical evidence: 1) asymmetric as gradient filtering. 2) target branch as trust region to mitigate collapse

\begin{thm} \textbf{(Optimization and Stability of \methodname)} \label{proofs: thm1}
    Consider the \methodname objective $\mathcal{L} = H(p^o) + \alpha \, D\!\left(p^o \,\|\, \mathrm{sg}[p^r]\right)$, where $H(\cdot)$ denotes the entropy loss, $D(\cdot)$ the alignment regularizer, and $p^o, p^r \in \Delta^{|\mathcal{C}|-1}$ are the probability distributions induced by the online and target branches. Given the encoder $f$, classifier $g$ and predictor $h$. Under Assumptions~\ref{proofs: ass1} and~\ref{proofs: ass2}, the following hold:  \\ 
    (1) For $\alpha = 0$, the entropy variation satisfies $|\Delta H(p^o)| > |\Delta H(p^r)|$, and the Hessian of \(H(p^o)\) attains its minimal eigenvalue along collapse directions $\mathbf{v}$:
    \begin{equation*}
        \lambda_{\min}\!\left(\nabla^2 H(p^o)\right) = \mathbf{v}^{\top}\nabla^2 H(p^o)\mathbf{v} .
    \end{equation*} 
    (2) For $\alpha > 0$, the predictor $h$ serves as a filtering mechanism that suppresses gradient update directions corresponding to over-amplified logits, and the system converges to a stable equilibrium: there exists $h_{\min} > 0$ such that
    \begin{equation*}
        H(p^o) > h_{\min},~p^o \to p^r .
    \end{equation*}
\end{thm}

\begin{remark}
    Theorem \ref{proofs: thm1} characterizes \methodname's optimization dynamics and convergence behavior. Initially, the online and target branches coincide within the non-collapse region, so entropy minimization dominates. In this regime, the online branch exhibits a stronger tendency toward collapse than the target branch in both magnitude and directional dominance (\textit{\textbf{Conclusion (1)}}). As optimization proceeds, the alignment term $D(\cdot)$ constrains the predictor $h$ to filter/suppress the gradient components that drive $p^o$ away from $p^r$, i.e., regulating the biased learning signals that may result in collapse, guiding the system toward a stable and non-collapsing equilibrium (\textit{\textbf{Conclusion (2)}}). \blue{This shows that the predictor in the online branch purposely converts biased signals (\eg, shortcuts) into explicit discrepancies, where these biased signals are then penalized by the alignment loss.}

    % As optimization proceeds, the alignment regularizer $D(\cdot)$ constrains the \mata{predictor $h$ to absorb the gradient components that drive $p^o$ away from $p^r$ (\ie, biased learning signals that may result in collapse), thereby filtering these gradients for the remaining encoder optimization and guiding the system toward a stable and non-collapsing equilibrium (\textit{\textbf{Conclusion (2)}}).}
    
    % As optimization proceeds, the alignment regularizer $D(\cdot)$ constrains the predictor $h$ to filter the gradient components that drive $p^o$ away from $p^r$, guiding the system toward a stable and non-collapsing equilibrium (\textit{\textbf{Conclusion (2)}}).
    % \mata{indicating $h$ absorbs biased learning signals that may resulting in collapse, thereby}
\end{remark}

\begin{wraptable}{r}{0.3\linewidth}
\vspace{-0.55in}
\caption{Efficiency comparison for processing 50,000 images (Gaussian noise, level 5 on ImageNet-C) via an RTX 3090 GPU on ViT-Base.}
\vspace{-0.13in}
\label{tab:efficiency_comparison}
\newcommand{\tabincell}[2]{\begin{tabular}{@{}#1@{}}#2\end{tabular}}
\begin{center}
\begin{threeparttable}
    \resizebox{1.0\linewidth}{!}{
 	\begin{tabular}{lr}
 	 Method & GPU time \\
 	\midrule
        Tent~\citep{wang2021tent} & 193 secs  \\
        SAR~\citep{niu2023sar} & 382 secs  \\
        EATA~\citep{niu2022eata} & 197 secs  \\
        COME~\citep{zhang2025come} & 300 secs \\
        DeYO~\citep{lee2024deyo} & 280 secs \\
    \midrule
         \methodname (ours) & 193 secs   \\
	\end{tabular}
	}
\end{threeparttable}
\end{center}
\vspace{-0.25in}
\end{wraptable}

\section{Comparisons with SOTAs}\label{sec:main_exp}

\textbf{Dataset and Methods} For image classification, we conduct experiments based on ImageNet-C~\citep{hendrycks2019benchmarking}, a large-scale benchmark for out-of-distribution generalization. It contains 15 types of 4 main categories (noise, blur, weather, digital) corrupted images and each type has 5 severity levels. For natural language reasoning, we assess logical reasoning capabilities on Math-500~\citep{lightman2023let}, CollegeMath~\citep{tang2024mathscale}, AIME24, and Minerva~\citep{lewkowycz2022solving}, ranging from high-school mathematics to advanced competition problems.
We compare with the following methods. EATA~\citep{niu2022eata}, SAR~\citep{niu2023sar}, DeYO~\citep{lee2024deyo}, and CETA~\citep{yang2024ceta} are entropy-based methods with sample selection or update reweighting. COME~\citep{zhang2025come} predefines an uncertainty mass, implemented based \linebreak on DeYO for comparisons. TLM~\citep{Hu2025TLM} refines the prediction entropy on input sequences.

% We compare with the following state-of-the-art methods. EATA~\citep{niu2022eata}, DeYO~\citep{lee2024deyo}, and CETA~\citep{yang2024ceta} are entropy-based methods with active sample selection or update reweighting. SAR~\citep{niu2023sar} optimizes the sharpness of the entropy surface. COME~\citep{zhang2025come} defines an uncertainty mass on entropy, where we implement it based on DeYO for comparisons. TLM~\citep{Hu2025TLM} refines the prediction entropy on input sequences.

\textbf{Models and Implementation Details} We conduct experiments on ResNet50-GN, ViT-Base, ViT-Small, ConvNeXt-Tiny, and Swin-Tiny that are obtained from \texttt{timm}~\citep{rw2019timm} for image classification, and Llama3.1-8B-Instruct~\citep{dubey2024llama} for natural language reasoning. We use symmetric KL as the divergence term in Eqn.~(\ref{eq:zerosiam-loss}), and $\alpha$ is fixed to 1 without tuning. We update predictor params $\theta_h$, and the affine parameters in norm layers for the vision task per Tent~\citep{wang2021tent}, and the LoRA~\citep{hu2022lora} parameters for LLM. More details are put in Appendix~\ref{sec:more_details}.

% For trainable parameters of \methodname during TTA, we update predictor params $\theta_h$, and also the affine parameters in normalization layers for image classification following Tent~\citep{wang2021tent}, and the LoRA~\citep{hu2022lora} parameters for large language models. More details are put in Appendix~\ref{sec:more_details}.

\subsection{Robustness to Corruption under Various Wild Test Settings}
% \subsection{Robustness under Various Wild Test Settings}

In this section, we evaluate the robustness of \methodname under the wild test scenarios as established by SAR~\citep{niu2023sar}, which simulate practical TTA scenarios in real-world applications, including 1) \textbf{{online imbalanced label distribution shifts}}: the samples in a non-i.i.d. data stream come in class order (\textbf{in Table~\ref{tab:imagenet-c-label-shift}}), 2) \textbf{{mixed distribution shifts}}: the samples in online data stream are obtained from a mixture of 15 corruption types (a total of 15$\times$50,000 images) in ImageNet-C, \ie, a prolonged TTA scenario with multiple concurrent shifts (\textbf{in Table~\ref{tab:imagenet-c-mix-shift}}), and 3) \textbf{{batch size = 1}}: the samples come one-by-one due to the scarcity of data in online streams (\textbf{in Table~\ref{tab:imagenet-c-bs1}}). Compared to SAR, we extend the evaluations using more models, including ViT-Small, ConvNeXt-Tiny, and Swin-Tiny.

From the results, we have the following general observations: 1) \textit{Sensitivity on model choice:} While sample selection methods like DeYO perform competitively on ResNet50-GN and ViT-Base, they remain unstable and \textit{sensitive to model choices}, particularly with weaker models where pseudo labels are highly unreliable, \eg, the accuracy of~0.1\% (DeYO) \textit{vs.} 26.5\% (\methodname) on 
\begin{wraptable}{r}{0.45\textwidth}
  \setlength{\tabcolsep}{6pt}
  \vspace{-0.03in}
    \caption{Comparisons with SOTAs on ImageNet-C (level 5) under \textbf{\textsc{mixture of 15 corruption types}} \wrt \textbf{Acc (\%)}.}
    \vspace{-0.14in}
    \label{tab:imagenet-c-mix-shift}
\newcommand{\tabincell}[2]{\begin{tabular}{@{}#1@{}}#2\end{tabular}}
 \begin{center}
     \begin{threeparttable}
      \resizebox{\linewidth}{!}{ % 缩放到 0.4\textwidth
        \begin{tabular}{l|ccccc|>{\columncolor{blue!8}}c}
          \multicolumn{1}{l|}{Method} & \multicolumn{1}{c}{R-50} & \multicolumn{1}{c}{Vit-B} & \multicolumn{1}{c}{Vit-S} & \multicolumn{1}{c}{Con-T} & \multicolumn{1}{c}{Swi-T} & \multicolumn{1}{|c}{Avg.} \\
          \midrule
          NoAdapt & 30.6 & 29.9 & 22.9 & 34.7 & 31.3 & 29.9 \\
          Tent & \drop{13.4} & \drop{16.5} & \drop{7.1} & \drop{27.8} & \drop{6.7} & 14.3 \\
          SAR & 38.1 & 55.7 & 41.3 & 36.5 & \drop{25.9} & 39.5 \\
          EATA & 38.3 & 57.1 & 36.9 & 41.1 & 41.3 & 42.9 \\
          COME & \drop{30.0} & 58.5 & 40.9 & \drop{23.7} & \drop{19.1} & 34.4 \\
          DeYO & 38.6 & 59.4 & 35.9 & \drop{27.8} & \drop{29.0} & 38.1 \\
          \rowcolor{black!8} \textit{ZeroSiam} & 40.4 & 57.1 & 41.7 & 42.4 & 39.3 & \textbf{44.2} \\
        \end{tabular}
      }
    \end{threeparttable}
	 \end{center}
\vspace{-0.19in}
\end{wraptable}
\textit{noise} corruptions with ViT-Small under the label shift scenario in Table~\ref{tab:imagenet-c-label-shift}, and SAR and COME also fail to perform TTA (worse than NoAdapt) on Swin-Tiny under mixed shifts in Table~\ref{tab:imagenet-c-mix-shift}; 
2) \textit{Sensitivity on test scenario choice:} For example, compared to results in Table~\ref{tab:imagenet-c-label-shift} under label shifts, EATA's performance drops significantly on Swin-Tiny under TTA with a single sample in Table~\ref{tab:imagenet-c-bs1}, \eg, 32.9\%$\rightarrow$25.3\% regarding the accuracy on \textit{blur} corruptions. This highlights the instability in prior methods and that achieving \textit{robust TTA across diverse architectures and test scenarios remains highly non-trivial}; 3) Our \methodname is simple yet effective. \methodname neither relies on careful sample selection strategies as in SAR or DeYO, nor requires source data as in EATA, and it achieves consistently stable and high performance across models and scenarios, markedly outperforming prior methods, \eg, 52.9\% (\methodname) \textit{vs.} 38.8\% (SAR) in terms of average accuracy under label shifts in Table~\ref{tab:imagenet-c-label-shift}, and 
an average gain of +8.5\% over DeYO under the batch size of 1 in Table~\ref{tab:imagenet-c-bs1}, suggesting our effectiveness.

\subsection{Effectiveness for Incentivizing Reasoning Capability Online}
A statistical reasoning model may fail to generalize to the diverse reasoning tasks in real-world applications, while retraining the model offline is both expensive and time-consuming. We further explore adaptively incentivizing the reasoning capability in a large language model \textit{online}, by minimizing the entropy of its own predicted tokens. From Table~\ref{tab:math-reasoning}, we have the following observations: 1) Entropy optimization demonstrates a great potential in enhancing the model's reasoning online, \ie, all prior methods achieve an accuracy gain by +3.34\% on AIME24. 2) However, for a comprehensive benchmark like CollegeMath, covering algebra, calculus, probability, differential equations, \textit{etc}, existing methods may suffer from overfitting and poor generalization, \eg, -0.83\% on Tent and -0.75\% on SAR. 3) \methodname enhances generalization for online reasoning incentivization with its predictor and asymmetric alignment design, which regularizes non-generalizable shortcuts as discussed in Section~\ref{sec:empric_theory_evidence}, yielding significant further gains, \eg, +10.00\% on AIME24, +3.40\% on Math-500, and +3.94\% on average. This suggests the potential of our \methodname to boost both the perception robustness and reasoning capability in a model during inference for greater intelligence.

% Resistance to noise adaptation.
% Resistance to learning from noise. Models first adapt on varying amounts of pure Gaussian noise, then perform TTA on ImageNet-C (Fog, level 5) with ViT-Base under label shifts.
% then run TTA on ImageNet-C (Fog, level 5) with ViT-B under label shifts.

% The model is evaluated post/during TTA for blind-spot adaptation/other scenarios. Thus, in Table~x, \methodname may achieve a higher performance on several domains compared to Table xx.
\subsection{Robustness to Adversarial Adaptation Scenarios (Stress Test)}
% \subsection{Robustness to Stress-Testing Adaptation Scenarios}

% rule out trivial collapsed solutions
% indicating reduced reliance on perfect sample selection and a wider safety margin for real-world TTA.
%  its suitability for safety-critical deployments.
% This property is crucial for ensuring safe deployment of TTA in challenging environments.
% as a safer choice for deployment in challenging real-world settings.
% This robustness highlights a critical safety advantage of \methodname for deployment in unpredictable real-world environments.

% Comparisons with state-of-the-art methods on ImageNet-C (severity level 5) under \textbf{\textsc{online imbalanced label shifts}} (imbalance ratio = $\infty$)  regarding \textbf{Accuracy (\%)}. ``$\mathcal{N}$/$\mathcal{B}$/$\mathcal{W}$/$\mathcal{D}$" are short for \textit{Noise/Blur/Weather/Digital} corruptions. \drop{RED} marks results worse than NoAdapt.
\begin{table*}[t]
  \setlength{\tabcolsep}{6pt}
  \renewcommand{\arraystretch}{1.02}
    \caption{Comparisons with SOTAs on ImageNet-C (level 5) under \textbf{\textsc{online label shifts}} (imbalance ratio=$\infty$) \wrt \textbf{Acc(\%)}. ``$\mathcal{N}$/$\mathcal{B}$/$\mathcal{W}$/$\mathcal{D}$" are short for \textit{Noise/Blur/Weather/Digital} corruptions. \drop{RED} marks results worse than NoAdapt. \textbf{Detailed results of each corruption are put in Appendix.}
    }
    \vspace{-0.15in}
  \label{tab:imagenet-c-label-shift}
  \newcommand{\tabincell}[2]{\begin{tabular}{@{}#1@{}}#2\end{tabular}}
  \begin{center}
  \begin{threeparttable}
  \LARGE
    \resizebox{1.0\linewidth}{!}{
      \begin{tabular}{l|cccc|cccc|cccc|cccc|cccc|>{\columncolor{blue!8}}c}
        \multicolumn{1}{c}{} & \multicolumn{4}{c}{ResNet50-GN} & \multicolumn{4}{c}{ViT-Base} & \multicolumn{4}{c}{ViT-Small} & \multicolumn{4}{c}{ConvNeXt-Tiny} & \multicolumn{4}{c}{Swin-Tiny} \\
        Method & $\mathcal{N}$ & $\mathcal{B}$ & $\mathcal{W}$ & $\mathcal{D}$ & $\mathcal{N}$ & $\mathcal{B}$ & $\mathcal{W}$ & $\mathcal{D}$ & $\mathcal{N}$ & $\mathcal{B}$ & $\mathcal{W}$ & $\mathcal{D}$ & $\mathcal{N}$ & $\mathcal{B}$ & $\mathcal{W}$ & $\mathcal{D}$ & $\mathcal{N}$ & $\mathcal{B}$ & $\mathcal{W}$ & $\mathcal{D}$ & Avg. \\
        \midrule
        NoAdapt & 18.6 & 19.4 & 47.7 & 33.9 & 8.1 & 28.4 & 36.1 & 41.6 & 1.8 & 23.4 & 27.7 & 33.5 & 23.7 & 24.8 & 52.1 & 35.4 & 25.8 & 21.7 & 50.5 & 25.9 & 29.0 \\
        Tent & \drop{2.9} & \drop{14.6} & \drop{27.7} & 38.0 & 22.9 & 54.3 & 41.3 & 64.7 & \drop{0.3} & 40.5 & 38.5 & 48.2 & \drop{22.8} & \drop{22.9} & \drop{50.2} & \drop{34.4} & \drop{13.6} & \drop{19.4} & \drop{34.9} & \drop{25.3} & 30.9 \\
        SAR & 35.0 & 24.9 & 47.8 & 40.6 & 46.2 & 55.8 & 61.4 & 65.8 & \drop{1.6} & 42.2 & 46.7 & 52.8 & 35.5 & \drop{24.2} & \drop{43.3} & 38.0 & 29.2 & \drop{19.0} & \drop{39.4} & 26.7 & 38.8 \\
        EATA & 27.8 & 20.4 & 42.7 & 34.6 & 35.7 & 47.4 & 61.6 & 51.6 & 16.3 & 42.2 & 54.8 & 54.2 & 35.4 & 29.0 & 54.9 & 42.2 & 35.6 & 32.9 & 53.4 & 41.1 & 40.7 \\
        COME & \drop{18.4} & \drop{19.2} & 47.7 & \drop{33.4} & 49.8 & 59.2 & 69.8 & 67.5 & \drop{0.2} & 39.5 & 56.0 & 59.6 & 43.1 & \drop{24.4} & 62.9 & 44.9 & 42.9 & 23.9 & 59.7 & 27.7 & 42.5 \\
        % DeYO & 43.7 & 23.2 & 54.2 & 54.5 & 48.0 & 56.6 & 71.2 & 69.9 & \drop{0.1} & 41.4 & 59.4 & 60.7 & 36.0 & \drop{19.0} & \drop{43.8} & \drop{34.8} & 41.9 & 26.8 & 56.6 & 46.8 & 44.4 \\
        DeYO & 43.7 & 23.2 & 54.2 & 54.5 & 48.0 & 56.6 & 71.2 & 69.9 & \drop{0.1} & 41.4 & 59.4 & 60.7 & 36.0 & \drop{19.0} & \drop{43.8} & \drop{34.8} & 36.5 & 22.7 & 52.1 & 34.7 & 43.1 \\
        \rowcolor{black!8} \textit{ZeroSiam} & 43.7 & 41.2 & 62.1 & 57.3 & 52.8 & 60.1 & 71.2 & 69.6 & 26.5 & 50.4 & 62.0 & 60.8 & 43.0 & 40.0 & 62.4 & 53.8 & 45.3 & 42.2 & 59.3 & 53.6 & \textbf{52.9} \\
      \end{tabular}
    }
  \end{threeparttable}
  \end{center}
  \vspace{-0.03in}
\end{table*}

  % \caption{Comparisons with state-of-the-art methods on ImageNet-C (severity level 5) with \textbf{\textsc{Batch Size=1}} regarding \textbf{Accuracy (\%)}. ``$\mathcal{N}$/$\mathcal{B}$/$\mathcal{W}$/$\mathcal{D}$" denote \textit{Noise/Blur/Weather/Digital} corruptions.}
\begin{table*}[t]
  \setlength{\tabcolsep}{6pt}
  \renewcommand{\arraystretch}{1.02}
  \caption{Comparisons with SOTAs on ImageNet-C (level 5) with \textbf{\textsc{Batch Size=1}} \wrt \textbf{Acc (\%)}.}
  \vspace{-0.15in}
  \label{tab:imagenet-c-bs1}
  \newcommand{\tabincell}[2]{\begin{tabular}{@{}#1@{}}#2\end{tabular}}
  \begin{center}
  \begin{threeparttable}
  \LARGE
    \resizebox{1.0\linewidth}{!}{
      \begin{tabular}{l|cccc|cccc|cccc|cccc|cccc|>{\columncolor{blue!8}}c}
        \multicolumn{1}{c}{} & \multicolumn{4}{c}{ResNet50-GN} & \multicolumn{4}{c}{ViT-Base} & \multicolumn{4}{c}{ViT-Small} & \multicolumn{4}{c}{ConvNeXt-Tiny} & \multicolumn{4}{c}{Swin-Tiny} \\
        Method & $\mathcal{N}$ & $\mathcal{B}$ & $\mathcal{W}$ & $\mathcal{D}$ & $\mathcal{N}$ & $\mathcal{B}$ & $\mathcal{W}$ & $\mathcal{D}$ & $\mathcal{N}$ & $\mathcal{B}$ & $\mathcal{W}$ & $\mathcal{D}$ & $\mathcal{N}$ & $\mathcal{B}$ & $\mathcal{W}$ & $\mathcal{D}$ & $\mathcal{N}$ & $\mathcal{B}$ & $\mathcal{W}$ & $\mathcal{D}$ & Avg. \\
        \midrule
        NoAdapt & 18.6 & 19.4 & 47.7 & 33.9 & 8.1 & 28.4 & 36.1 & 41.7 & 1.8 & 23.3 & 27.8 & 33.6 & 23.7 & 24.6 & 52.2 & 35.3 & 25.9 & 21.6 & 50.6 & 25.8 & 29.0 \\
        Tent & \drop{2.6} & \drop{13.3} & \drop{27.5} & 37.9 & 28.8 & 51.7 & 43.7 & 61.9 & 23.7 & 24.6 & 52.2 & 35.3 & 32.0 & \drop{23.9} & \drop{51.9} & 36.5 & \drop{24.2} & 23.0 & \drop{42.1} & 26.7 & 33.2 \\
        SAR & 24.3 & 23.2 & \drop{46.9} & 40.6 & 39.6 & 53.7 & 63.1 & 64.6 & 25.9 & \drop{21.6} & 50.6 & \drop{25.8} & 31.2 & 25.3 & \drop{51.0} & 35.8 & \drop{25.2} & 22.8 & \drop{47.0} & 26.3 & 37.2 \\
        EATA & 26.3 & 23.3 & 50.3 & 43.5 & 29.8 & 43.7 & 52.3 & 56.4 & 2.2 & 30.7 & 36.9 & 43.6 & 28.0 & 25.1 & \drop{51.8} & 36.4 & 33.0 & 25.3 & \drop{50.6} & 30.8 & 36.0 \\
        COME & \drop{18.4} & \drop{19.3} & \drop{47.6} & \drop{33.5} & 51.7 & 55.1 & 70.6 & 69.6 & \drop{1.5} & 41.0 & 56.7 & 59.6 & 41.8 & 25.5 & 62.0 & 43.5 & 39.7 & 25.3 & 58.3 & 38.6 & 43.0 \\
        DeYO & 43.2 & 28.4 & 50.2 & 55.5 & 53.7 & 59.0 & 71.7 & 70.4 & \drop{0.5} & 42.4 & 59.1 & 60.4 & 37.7 & \drop{21.3} & \drop{44.2} & 35.3 & 36.4 & 24.6 & \drop{50.1} & 35.3 & 44.0 \\
        \rowcolor{black!8} \textit{ZeroSiam} & 42.6 & 40.8 & 62.9 & 58.6 & 52.5 & 59.9 & 70.9 & 69.6 & 27.1 & 50.1 & 61.4 & 60.5 & 42.2 & 39.3 & 61.7 & 52.8 & 44.9 & 41.5 & 58.8 & 52.9 & \textbf{52.5} \\
      \end{tabular}
    }
  \end{threeparttable}
  \end{center}
  \vspace{-0.12in}
\end{table*}

% !!subfigure
% \begin{figure}[t]
%     \centering
%     \hfill
%     \subfigure{\includegraphics[width=0.45\linewidth]{noise_adapt_v4.pdf}}\vspace{-0.05in}
%     \hfill\subfigure{\includegraphics[width=0.40\linewidth]{Heatmap_v7.pdf}}\vspace{-0.05in}
%     \hfill
%     \vspace{-0.05in}
%     \caption{.}
%     \label{fig:ablation-study}
%     \vspace{-0.1in}
% \end{figure}

\textbf{Robustness under Blind-Spot Adaptation~} Existing entropy-based TTA methods largely rely on reliable sample selection strategies to stabilize adaptation~\citep{niu2023sar,lee2024deyo}. However, in many real-world scenarios, adaptation inevitably involves numerous erroneous pseudo labels that cannot be identified, especially when the model's initial accuracy is very low. This raises a critical concern: How safe is existing TTA in the real world? To stress-test this vulnerability, we construct a \textit{blind-spot subset} consisting only of samples misclassified by the NoAdapt model for each domain and adapt exclusively on this subset, before evaluation on the full dataset.  As shown in Table~\ref{tab:imagenet-c-blindspot}, prior methods suffer from severe instability, often collapse and underperform the NoAdapt baseline after TTA, \eg, DeYO deteriorates accuracy on \textit{12} out of 20 scenarios. 
% This is because sample selection only seeks to filter some errors, but remains exposed to trivial solutions. 
In contrast, \methodname introduces an efficient asymmetry to prevent undesired trivial solutions from optimization, delivering consistent accuracy gains \textit{even when adapting with incorrect labels} \eg, the average accuracy of 29.0\% (NoAdapt) \textit{vs.} 52.0\% (\methodname), drastically breaking prior performance ceiling ($\approx$33.6\%). This underscores \methodname as a more principled and robust solution for real-world TTA.

\begin{wrapfigure}{r}{0.33\textwidth}
\vspace{-0.18in}
\centering
\includegraphics[width=.95\linewidth]{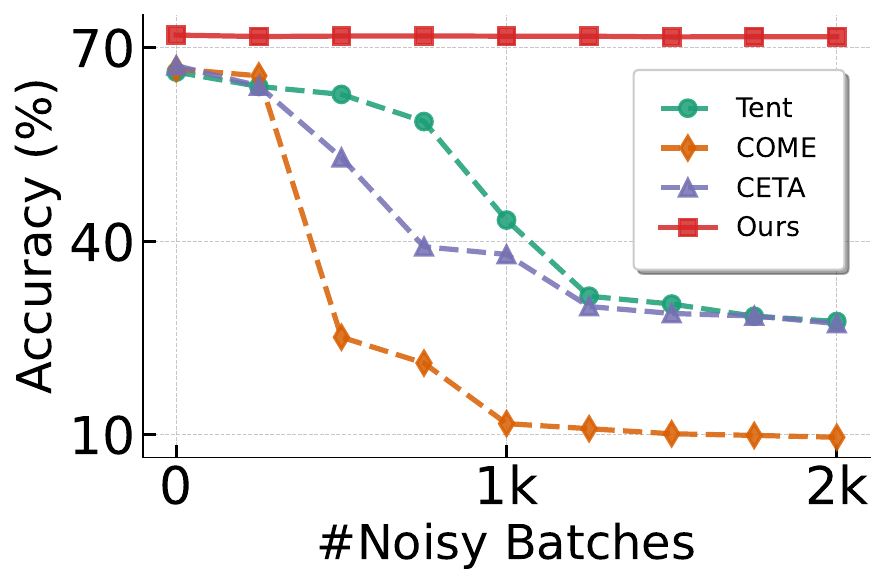}
\vspace{-0.11in}
\caption{Resistance to learning from noise. Models pre-adapt on $N$ pure Gaussian noise, then run TTA on ImageNet-C (level 5).}
\label{fig:pure_noise}
\vspace{-0.2in}
\end{wrapfigure}

% \textbf{Robustness to Non-Semantic Corruptions.}
\textbf{Resistance to Learning from Noise}
In dynamic real-world applications, models may frequently encounter test data that are severely corrupted and \textit{non-semantic}, such as extreme occluded frames and pure sensor noise where no valid label exists. Minimizing entropy on these data can then be misled to ``confidently learn nonsense'', risking model degeneration. To simulate this, we first pre-adapt models on varying amounts of pure Gaussian noise, then perform TTA on ImageNet-C (Fog, level 5) with ViT-Base under the mid scenario~\citep{wang2021tent}. As shown in Figure~\ref{fig:pure_noise}, existing entropy-based methods degrade sharply post-adaptation to noise, \eg, 67.3\%$\rightarrow$27.2\% on CETA, indicating an overfit to non-semantic patterns. In contrast, \methodname absorbs and regularizes non-generalizable shortcuts through the asymmetric alignment (c.f. Section~\ref{sec:empric_theory_evidence}), maintaining consistently high and stable accuracy ($\approx72\%$) regardless of the exposure to noisy inputs. This robustness implies that \methodname can be safely deployed even when streams contain meaningless or corrupted inputs, a property crucial for reliable deployment in real-world scenarios.
\begin{table*}[t]
\setlength{\tabcolsep}{8pt}
    % \vspace{-0.05in}
\caption{Comparisons with SOTAs for \textbf{\textsc{online adaptive reasoning}} in mathematical reasoning benchmarks regarding \textbf{Accuracy (\%)}.}
    \vspace{-0.1in}
    \label{tab:math-reasoning}
\newcommand{\tabincell}[2]{\begin{tabular}{@{}#1@{}}#2\end{tabular}}
 \begin{center}
     \begin{threeparttable}
      \resizebox{0.9\linewidth}{!}{ % 缩放到 0.4\textwidth
        \begin{tabular}{l|cccc|>{\columncolor{blue!8}}c}
          \multicolumn{1}{l|}{Model+Method} & \multicolumn{1}{c}{Math-500} & \multicolumn{1}{c}{CollegeMath} & \multicolumn{1}{c}{AIME24} & \multicolumn{1}{c}{Minerva} & \multicolumn{1}{|c}{Average} \\
          \midrule
          Llama3.1-8B~\citep{dubey2024llama} & 49.20 & 25.00 & 3.33 & 20.96 & 24.62 \\
          $\bullet$~~Tent~\citep{wang2021tent} & 50.00$_{(+0.80)}$ & 24.17$_{(-0.83)}$ & 6.67$_{(+3.34)}$ & 20.59$_{(-0.37)}$ & 25.36$_{(+0.74)}$ \\
          $\bullet$~~SAR~\citep{niu2023sar} & 49.20$_{(+0.00)}$ &	24.25$_{(-0.75)}$ &	6.67$_{(+3.34)}$ &	21.32$_{(+0.36)}$ &	25.36$_{(+0.74)}$  \\
          $\bullet$~~EATA~\citep{niu2022eata} & 48.80$_{(-0.40)}$ & 24.83$_{(-0.17)}$ & 6.67$_{(+3.34)}$ & 20.59$_{(-0.37)}$ & 25.22$_{(+0.60)}$ \\
          $\bullet$~~COME~\citep{zhang2025come} & 49.80$_{(+0.60)}$ & 25.42$_{(+0.42)}$ & 6.67$_{(+3.34)}$ & 22.74$_{(+1.78)}$ & 26.16$_{(+1.54)}$ \\
          $\bullet$~~TLM~\citep{Hu2025TLM} & 50.00$_{(+0.80)}$ & 25.58$_{(+0.58)}$ & 6.67$_{(+3.34)}$ & 19.49$_{(-1.47)}$ & 25.44$_{(+0.82)}$ \\
          % \rowcolor{black!8} $\bullet$~~\textit{ZeroSiam} (ours) & 52.60$_{(+3.40)}$ & 25.67$_{(+0.67)}$ & 6.67$_{(+3.34)}$ & 22.43$_{(+1.47)}$ & \textbf{26.84$\boldsymbol{_{(+2.22)}}$} \\
          \rowcolor{black!8} $\bullet$~~\textit{ZeroSiam} (ours) & 52.60$_{(+3.40)}$ & 26.25$_{(+1.25)}$ & 13.33$_{(+10.00)}$ & 22.06$_{(+1.10)}$ & \textbf{28.56$_{(+3.94)}$}
        \end{tabular}
      }
    \end{threeparttable}
	 \end{center}
\vspace{-0.1in}
\end{table*}

% Comparisons with state-of-the-art methods on ImageNet-C (severity level 5) on the \textbf{\textsc{blind-spot subset}} (samples initially misclassified by the NoAdapt model) with \textbf{\textsc{Batch Size=1}} regarding \textbf{Accuracy (\%)}. ``$\mathcal{N}$/$\mathcal{B}$/$\mathcal{W}$/$\mathcal{D}$" are short for \textit{Noise/Blur/Weather/Digital} corruptions.
\begin{table*}[t]
  \setlength{\tabcolsep}{6pt}
  \renewcommand{\arraystretch}{1.02}
  \caption{Comparisons with SOTAs on ImageNet-C (level 5)'s \textbf{\textsc{blind-spot subset}} (samples initially misclassified by the NoAdapt model) with \textbf{\textsc{Batch Size=1}} regarding \textbf{Accuracy (\%)}. 
    }
    \vspace{-0.15in}
  \label{tab:imagenet-c-blindspot}
  \newcommand{\tabincell}[2]{\begin{tabular}{@{}#1@{}}#2\end{tabular}}
  \begin{center}
  \begin{threeparttable}
  \LARGE
    \resizebox{1.0\linewidth}{!}{
      \begin{tabular}{l|cccc|cccc|cccc|cccc|cccc|>{\columncolor{blue!8}}c}
        \multicolumn{1}{c}{} & \multicolumn{4}{c}{ResNet50-GN} & \multicolumn{4}{c}{ViT-Base} & \multicolumn{4}{c}{ViT-Small} & \multicolumn{4}{c}{ConvNeXt-Tiny} & \multicolumn{4}{c}{Swin-Tiny} \\
        Method & $\mathcal{N}$ & $\mathcal{B}$ & $\mathcal{W}$ & $\mathcal{D}$ & $\mathcal{N}$ & $\mathcal{B}$ & $\mathcal{W}$ & $\mathcal{D}$ & $\mathcal{N}$ & $\mathcal{B}$ & $\mathcal{W}$ & $\mathcal{D}$ & $\mathcal{N}$ & $\mathcal{B}$ & $\mathcal{W}$ & $\mathcal{D}$ & $\mathcal{N}$ & $\mathcal{B}$ & $\mathcal{W}$ & $\mathcal{D}$ & Avg. \\
        \midrule
        NoAdapt & 18.6 & 19.4 & 47.7 & 33.9 & 8.1 & 28.4 & 36.1 & 41.7 & 1.8 & 23.3 & 27.8 & 33.6 & 23.7 & 24.6 & 52.2 & 35.3 & 25.9 & 21.6 & 50.6 & 25.8 & 29.0 \\
        Tent & \drop{0.2} & \drop{5.2} & \drop{16.8} & \drop{28.1} & \drop{0.2} & 42.9 & \drop{20.0} & 52.0 & \drop{0.1} & 40.0 & 36.3 & 42.8 & \drop{17.9} & \drop{19.7} & \drop{36.3} & \drop{31.0} & \drop{0.2} & \drop{8.0} & \drop{22.2} & \drop{13.0} & 21.7 \\
        SAR & \drop{17.9} & \drop{18.4} & \drop{35.6} & \drop{31.6} & 31.1 & 55.0 & 54.0 & 56.4 & 2.0 & 30.4 & 42.0 & 43.3 & 26.4 & 24.6 & \drop{50.7} & \drop{35.1} & 27.0 & 22.2 & \drop{49.3} & \drop{23.8} & 33.8 \\
        EATA & 19.5 & \drop{19.0} & \drop{44.2} & 40.2 & 26.3 & 43.5 & 53.2 & 57.5 & 2.0 & 26.5 & 32.8 & 41.7 & 25.8 & 24.7 & \drop{51.6} & 35.7 & 27.4 & 22.4 & \drop{49.1} & 26.8 & 33.5 \\
        COME & 18.6 & 19.4 & \drop{47.6} & \drop{33.8} & \drop{0.1} & 44.8 & 70.9 & 69.7 & \drop{1.0} & 37.0 & 55.9 & 46.2 & 41.4 & \drop{16.4} & \drop{31.1} & 36.1 & 40.5 & \drop{9.5} & \drop{25.3} & 26.9 & 33.6 \\
        DeYO & \drop{0.5} & \drop{9.4} & \drop{21.0} & 40.0 & \drop{0.2} & 47.2 & 72.6 & 71.4 & \drop{0.2} & 38.2 & 60.2 & 61.7 & 35.2 & \drop{12.2} & \drop{35.0} & \drop{33.0} & \drop{0.2} & \drop{8.8} & \drop{25.7} & \drop{23.9} & 29.8 \\
        \rowcolor{black!8} \textit{ZeroSiam} & 39.9 & 36.6 & 53.9 & 52.7 & 53.0 & 60.4 & 71.6 & 70.2 & 26.1 & 49.8 & 61.8 & 61.1 & 43.3 & 40.3 & 61.9 & 55.0 & 46.5 & 42.1 & 58.8 & 54.7 & \textbf{52.0} \\
      \end{tabular}
    }
  \end{threeparttable}
  \end{center}
  \vspace{-0.2in}
\end{table*}

% \begin{wraptable}{r}{0.55\textwidth}
% \setlength{\tabcolsep}{3pt}
% \vspace{-0.15in}
%     \caption{Comparisons with state-of-the-arts on ImageNet-C (level 5) under \textbf{\textsc{mixture of 15 corruption types}} regarding \textbf{Accuracy (\%)}.}
%     \label{tab:imagenet-c-mix-shift}
% \newcommand{\tabincell}[2]{\begin{tabular}{@{}#1@{}}#2\end{tabular}}
%  \begin{center}
%      \begin{threeparttable}
%       \resizebox{\linewidth}{!}{ % 缩放到 0.4\textwidth
%         \begin{tabular}{l|cccc|>{\columncolor{blue!8}}c}
%           \multicolumn{1}{l|}{Method} & \multicolumn{1}{c}{\textit{Math-500}} & \multicolumn{1}{c}{\textit{Col-Math}} & \multicolumn{1}{c}{\textit{AIME24}} & \multicolumn{1}{c}{\textit{Minerva}} & \multicolumn{1}{|c}{Avg.} \\
%           \midrule
%           Llama3.1-8B & 49.20 & 25.00 & 3.33 & 20.96 & 24.62 \\
%           $\bullet$~~Tent & 50.00 & 24.17 & 6.67 & 20.59 & 25.36 \\
%           $\bullet$~~EATA &  &  &  & &  \\
%           $\bullet$~~COME &  &  &  & &  \\
%           $\bullet$~~TLM & 50.00 & 25.58 & 6.67 & 19.49 &  \\
%           \rowcolor{black!8} $\bullet$~~\textit{ZeroSiam} & 52.60 & 25.67 & 6.67 & 22.43 & \\
%         \end{tabular}
%       }
%     \end{threeparttable}
% 	 \end{center}
% \vspace{-0.15in}
% \end{wraptable}

\subsection{Additional Discussions}

% why efficiency is important
% \textbf{Comparisons on Efficiency} 
% Computational efficiency is critical for online deployment in resource-limited or latency-sensitive applications. \methodname achieves collapse prevention using \textit{only} a single FC layer to create an asymmetric structure, without requiring augmentations and extra backbone passes. As in Tables~\ref{tab:efficiency_comparison} \& \ref{tab:imagenet-c-label-shift}, \methodname significantly improves stability while matching the latency of Tent, with a lower latency compared to prior state-of-the-art methods, \eg, \textbf{193s} (\methodname) \textit{vs.} 280s (DeYO). These results highlight the advantages of \methodname~\wrt both TTA stability and efficiency.

\textbf{Efficacy under Mild Test Setting} We further evaluate \methodname in the mild setting~\citep{wang2021tent}, where data comes with shuffled labels. As shown in Table~\ref{tab:imagenet-c-mild}, \methodname delivers superior performance compared to prior methods, \eg, the average accuracy of 50.1\% (\methodname) \textit{vs.} 43.0\% (COME), suggesting our efficacy across a wide range of scenarios. Moreover, as in Table~\ref{tab:efficiency_comparison}, \methodname also achieves a lower latency compared to prior SOTAs, \eg, 193s (\methodname) \textit{vs.} 280s (DeYO), highlighting the advantages of \methodname regarding both TTA stability and efficiency.
\blue{We also offer more results under diverse adaptation epochs over the entire test set in Appendix~\ref{sec:suppl:different-epochs}.}

% \methodname is not only effective under challenging test scenarios, but also enhances TTA performance significantly on the mild test scenario~\citep{wang2021tent}, where data comes with shuffled labels. As shown in Table~\ref{tab:imagenet-c-mild-details}, \methodname demonstrates superior performance on 4 out of 5 models and also achieves comparative results on ViT-Base, \eg, the average accuracy of 48.6\% (Ours) \vs 41.7\% (DeYO) on ResNet50-GN, and 62.6\% (Ours) \vs 62.8\% (DeYO) on ViT-Base. This suggests that our \methodname's design helps improve the stability and efficacy of TTA across a wide range of scenarios. 

% \textbf{Design Choices of Predictors}
\textbf{Ablation on Predictor Design}
% \methodname uses a single learnable FC layer as the predictor for the online branch. Here, 
We further examine other predictor designs in Table~\ref{tab:predictor_design}. Interestingly, replacing the learnable predictor with a fixed, randomly initialized predictor ($\theta_h=\mathbf{I}+0.1\,\mathbf{W}$, $\mathbf{W}\!\sim\!\mathcal{N}(0,\mathbf{I})$) already improves upon Tent, \ie, 47.3\%$\rightarrow$60.7\% \wrt accuracy, confirming the value of breaking symmetry alone in collapse mitigation. However, making the predictor deeper and nonlinear brings no further gain (64.0\% vs. 64.1\%), since the predictor in \methodname is to absorb biased shortcuts for regularization rather than learning rich representations, which do not need large representation capacity.
% for regularization in asymmetric alignment rather than learning rich representations
% therefore its efficacy is insensitive to capacity. 
Therefore, we use the single linear predictor for all other experiments.

% State-of-the-arts
\textbf{Integration with Prior Methods} We further evaluate the efficacy of our \methodname when integrated with state-of-the-art methods. From Table~\ref{tab:zerosiam-integration}, \methodname consistently enhances the performance of prior methods, \eg, improving the average accuracy by +8.4\% on EATA and +7.4\% on DeYO, suggesting the efficacy of our asymmetric architecture as a plug-and-play component. However, incorporating with EATA and DeYO does not ensure an improvement over \methodname. This is because while EATA and DeYO aim to filter partial noisy samples, \methodname adapts reliably even with incorrect pseudo labels, as verified in Table~\ref{tab:imagenet-c-blindspot}. Thus, \methodname may not directly benefit from the traditional sample selection design and encourages more exploration. We leave it for future work.

\begin{wrapfigure}{r}{0.35\textwidth}
\vspace{-0.17in}
\centering
\includegraphics[width=1.0\linewidth]{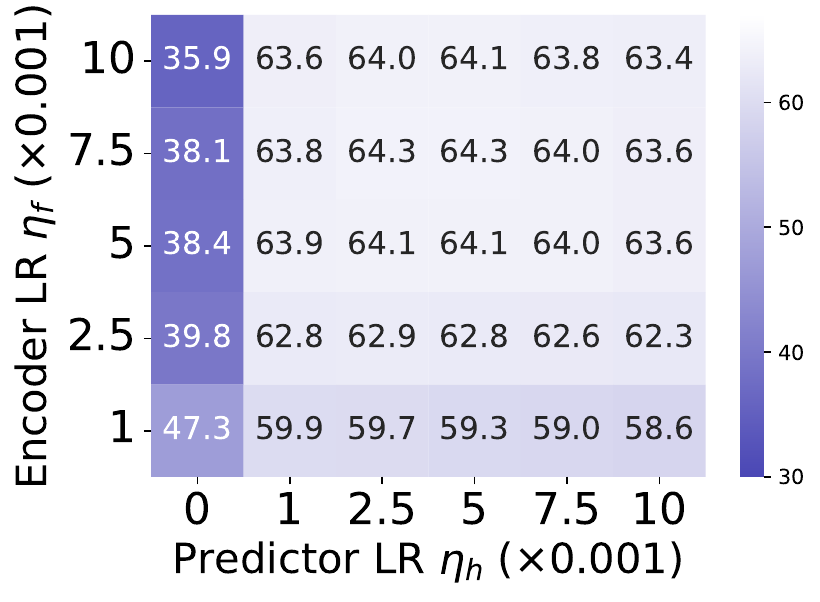}
\vspace{-0.26in}
\caption{Sensitivity to learning rates. Results are reported on Im-ageNet-C (level 5) with ViT-Base under label shifts \wrt Accuracy.}
\vspace{-0.2in}
\label{fig:hyperparameter_sensitivity}
\end{wrapfigure}

\textbf{Sensitivity to Learning Rates~} 
We further examine the sensitivity of \methodname\ to learning rates. When the predictor learning rate is set to zero, the predictor becomes a frozen identity and \methodname degenerates to Tent. From Figure~\ref{fig:hyperparameter_sensitivity}, we have the following observations: 1) our \methodname consistently outperforms Tent across a wide range of encoder learning rates, \ie, by $11.3\%{-}28.2\%$ in accuracy, suggesting the benefit of our asymmetric design. 2) For a fixed encoder learning rate, varying \(\eta_h\in\{1,2.5,5,7.5,10\}\times0.001\) changes accuracy by only $\approx1\%$, highlighting our effectiveness does not rely on careful hyperparameter tuning. 
3) overall, \methodname achieves both high accuracy and a broad performance plateau under diverse learning rate configurations, \eg, above 62.3\% when $\eta_f\in[2.5,10]{\times}10^{-3}$ and $\eta_h\in[1,10]{\times}10^{-3}$, consistently outperforming the accuracy of 58.0\% in SAR. These results suggest that \methodname\ is not only effective but also robust to learning rate choices, making it practical and easy to use for online deployment.
\blue{We also provide more sensitivity analysis of learning rates on Swin-Tiny and further introduce a data-free approach for determining predictor learning rate $\eta_h$ in Appendix~\ref{sec:suppl:learning-rate-on-swin} and~\ref{sec:suppl:data-free}, respectively.}

% We also demonstrate a potential strategy for \methodname to select learning rates in a \textit{data-free} manner (Figure~X in Appendix), by using random Gaussian noise as the proxy to track TTA dynamics.
%  a wide range of learning rate choices

\begin{table*}[t]
  \setlength{\tabcolsep}{6pt}
  \renewcommand{\arraystretch}{1.02}
  \caption{Comparisons with state-of-the-art methods on ImageNet-C (severity level 5) under a \textbf{\textsc{mild scenario}} regarding \textbf{Accuracy (\%)}. ``$\mathcal{N}$/$\mathcal{B}$/$\mathcal{W}$/$\mathcal{D}$" denote \textit{Noise/Blur/Weather/Digital} corruptions.}
  \vspace{-0.15in}
  \label{tab:imagenet-c-mild}
  \newcommand{\tabincell}[2]{\begin{tabular}{@{}#1@{}}#2\end{tabular}}
  \begin{center}
  \begin{threeparttable}
  \LARGE
    \resizebox{1.0\linewidth}{!}{
      \begin{tabular}{l|cccc|cccc|cccc|cccc|cccc|>{\columncolor{blue!8}}c}
        \multicolumn{1}{c}{} & \multicolumn{4}{c}{ResNet50-GN} & \multicolumn{4}{c}{ViT-Base} & \multicolumn{4}{c}{ViT-Small} & \multicolumn{4}{c}{ConvNeXt-Tiny} & \multicolumn{4}{c}{Swin-Tiny} \\
        Method & $\mathcal{N}$ & $\mathcal{B}$ & $\mathcal{W}$ & $\mathcal{D}$ & $\mathcal{N}$ & $\mathcal{B}$ & $\mathcal{W}$ & $\mathcal{D}$ & $\mathcal{N}$ & $\mathcal{B}$ & $\mathcal{W}$ & $\mathcal{D}$ & $\mathcal{N}$ & $\mathcal{B}$ & $\mathcal{W}$ & $\mathcal{D}$ & $\mathcal{N}$ & $\mathcal{B}$ & $\mathcal{W}$ & $\mathcal{D}$ & Avg. \\
        \midrule
        NoAdapt & 18.6 & 19.4 & 47.7 & 33.9 & 8.1 & 28.4 & 36.1 & 41.6 & 1.8 & 23.4 & 27.7 & 33.5 & 23.7 & 24.8 & 52.1 & 35.4 & 25.8 & 21.7 & 50.5 & 25.9 & 29.0 \\
        Tent & \drop{5.5} & \drop{17.7} & \drop{33.1} & 39.0 & 29.0 & 51.3 & 44.7 & 62.1 & \drop{0.5} & 37.2 & 41.4 & 47.2 & 32.1 & \drop{24.0} & \drop{51.9} & 36.6 & \drop{24.2} & 23.1 & \drop{42.6} & 26.9 & 33.5 \\
        SAR & 29.9 & 24.6 & \drop{41.2} & 41.8 & 34.7 & 52.4 & 62.5 & 62.6 & \drop{1.5} & 38.6 & 44.1 & 49.0 & 34.2 & 25.0 & \drop{45.4} & 37.8 & 32.4 & 25.1 & \drop{47.3} & 29.0 & 38.0 \\
        EATA & 38.6 & 32.8 & 56.9 & 51.4 & 50.1 & 57.4 & 68.2 & 67.2 & 22.2 & 44.7 & 56.4 & 56.2 & 36.2 & 30.6 & 55.5 & 42.6 & 41.8 & 37.5 & 56.5 & 47.1 & 47.5 \\
        COME & \drop{18.4} & \drop{19.3} & \drop{47.6} & \drop{33.6} & 51.6 & 58.9 & 70.1 & 69.0 & \drop{0.3} & 41.2 & 56.0 & 58.4 & 41.1 & 26.3 & 60.5 & 43.4 & 42.2 & 25.5 & 58.7 & 38.6 & 43.0 \\
        DEYO & 40.8 & 30.4 & \drop{44.0} & 51.4 & 53.3 & 55.0 & 71.1 & 69.6 & \drop{0.2} & 44.7 & 57.8 & 59.0 & 37.3 & \drop{22.1} & \drop{44.3} & 36.2 & 36.1 & 24.6 & 51.0 & 41.1 & 43.5 \\
        \rowcolor{black!8} \textit{ZeroSiam} & 41.6 & 37.1 & 59.7 & 54.2 & 51.0 & 58.4 & 69.6 & 68.3 & 24.2 & 48.1 & 59.6 & 58.8 & 40.2 & 35.4 & 59.5 & 49.0 & 43.1 & 38.2 & 56.6 & 49.8 & \textbf{50.1} \\
      \end{tabular}
    }
  \end{threeparttable}
  \end{center}
  % \vspace{-0.15in}
\end{table*}

\begin{table}[t]
    \centering
    \begin{minipage}[t]{0.29\textwidth}
    \setlength{\tabcolsep}{5pt}
    % \vspace{-0.05in}
    \caption{Ablation on predictor design in \methodname on ImageNet-C (level 5) with ViT-Base under label shifts.}
    \vspace{-0.1in}
    \label{tab:predictor_design}
    \newcommand{\tabincell}[2]{\begin{tabular}{@{}#1@{}}#2\end{tabular}}
    \begin{center}
    \begin{threeparttable}
        \resizebox{1.0\linewidth}{!}{
     	\begin{tabular}{llr}
     	 Method & Predictor & Acc.\\
     	\midrule
            Tent & - & 47.3  \\
            % \rowcolor{black!8} 
            \methodname & FC & \textbf{64.1}  \\
                \smallbullet~ Exp1 & \textit{fixed} random FC & 60.7 \\
            \smallbullet~ Exp2 & FC+ReLU+FC & 64.0 \\
            % \smallbullet~ Exp3 & Transformer Layer & 64.0 \\
    	\end{tabular}
    	}
    \end{threeparttable}
    \end{center}
    \end{minipage}
    \hfill
    \begin{minipage}[t]{0.32\textwidth}
    \setlength{\tabcolsep}{3.5pt}
    \caption{\methodname's efficacy when used with prior methods. Experiments are on ImageNet-C (level 5) under label shifts.}
    \vspace{-0.1in}
        \label{tab:zerosiam-integration}
    \newcommand{\tabincell}[2]{\begin{tabular}{@{}#1@{}}#2\end{tabular}}
     \begin{center}
         \begin{threeparttable}
          \resizebox{\linewidth}{!}{ % 缩放到 0.4\textwidth
            \begin{tabular}{l|ccc|c}
              \multicolumn{1}{l|}{Method} & \multicolumn{1}{c}{ViT-B} & \multicolumn{1}{c}{R-50} & \multicolumn{1}{c}{Con-T} & \multicolumn{1}{|c}{Avg.} \\
              \midrule
              EATA & 49.4 & 31.6 & 40.7 & 40.6 \\
                \rowcolor{black!8} 
                \textbf{+}~\textit{\methodname} & 57.1 & 48.3 & 41.7 & 49.0 \\
               DeYO & 62.3 & 43.9 & 33.2 & 46.5 \\
            \rowcolor{black!8}
            \textbf{+}~\textit{\methodname} & 65.0 & 50.7 & 46.1 & 53.9 \\
            \end{tabular}
          }
        \end{threeparttable}
    	 \end{center}
    \end{minipage}
    \hfill
    \begin{minipage}[t]{0.34\textwidth}
    \setlength{\tabcolsep}{4.2pt}
    \caption{Generality of \methodname across objective and divergence choices on ImageNet-C (level 5) with ViT-Base under label shifts.}
    \vspace{-0.1in}
    \label{tab:zerosiam-generality}
\newcommand{\tabincell}[2]{\begin{tabular}{@{}#1@{}}#2\end{tabular}}
 \begin{center}
     \begin{threeparttable}
      \resizebox{\linewidth}{!}{ % 缩放到 0.4\textwidth
        \begin{tabular}{l|cccccc}
          \multicolumn{1}{l|}{Loss} & \multicolumn{1}{c}{N/A} & \multicolumn{1}{c}{sKL} & \multicolumn{1}{c}{KL} & \multicolumn{1}{c}{rKL} & \multicolumn{1}{c}{JS} & \multicolumn{1}{c}{MSE}\\
          \midrule
          SLR & 44.0 & 62.2 & 63.2 & 61.0 & 63.4 & 61.5 \\
          CE & 43.2 & 61.6 & 61.4 & 61.9 & 61.4 & 62.1 \\
          $-p^2$ & 45.8 & 63.3 & 63.4 & 62.1 & 63.0 & 63.6 \\
        \rowcolor{black!8} 
        Tent & 47.3 & 64.1 & 64.0 & 64.1 & 63.8 & 64.1 \\
        \end{tabular}
      }
    \end{threeparttable}
	 \end{center}

    \end{minipage}
    \vspace{-0.1in}
\end{table}

% \textbf{Importance of Stop-Gradient} Applying stop-gradient on the target branch provides a stable anchor for the online branch that helps prevent collapse. As shown in Table~\ref{tab:stop-gradient-importance}, removing stop-gradient drastically reduces the accuracy from 51.6\%$\rightarrow$20.5\% on ResNet50-GN and 64.1\%$\rightarrow$38.5\% on Vit-Base under imbalance label shifts. Such results align with prior findings in SimSiam~\citep{chen2021exploring} and BYOL~\citep{grill2020bootstrap}, underscoring stop-gradient as a key component in asymmetry.

% !!wrap
% \begin{wraptable}{r}{0.32\linewidth}
% \vspace{-0.17in}
% \caption{Ablation on predictor design in \methodname. Results are reported on ImageNet-C (level 5) with ViT-B under label shifts.}
% \vspace{-0.1in}
% \label{tab:predictor_design}
% \newcommand{\tabincell}[2]{\begin{tabular}{@{}#1@{}}#2\end{tabular}}
% \begin{center}
% \begin{threeparttable}
%     \resizebox{1.0\linewidth}{!}{
%  	\begin{tabular}{llr}
%  	 Method & Predictor & Acc.\\
%  	\midrule
%         Tent & - & 47.3  \\
%         \methodname & FC & \textbf{64.1}  \\
%             \smallbullet~ Exp1 & \textit{fixed} random FC & 60.7 \\
%         \smallbullet~ Exp2 & FC+ReLU+FC & 64.0 \\
%         % \smallbullet~ Exp3 & Transformer Layer & 64.0 \\
% 	\end{tabular}
% 	}
% \end{threeparttable}
% \end{center}
% \vspace{-0.2in}
% \end{wraptable}

\blue{
\textbf{Sensitivity of $\alpha$ in \methodname} To avoid heavy tuning in practical TTA use cases, we consistently set $\alpha=1.0$ in \methodname. From the perspective of learning objective, $\alpha = 1.0$ strikes a balanced and interpretable interaction between the two objectives: 
1) comparable gradient scales: both entropy loss and the divergence loss operate in \textit{probability space} and are also \textit{per-sample normalized}, which yield comparable gradient scales for TTA, without any one dominating optimization;
2) simplified and interpretable formulations: consider Eqn.~(\ref{eq:zerosiam-loss}) with $\alpha=1.0$ detailed as $H(p^{o})+D_{KL}(p^{o}||\text{sg}[p^r])+D_{KL}(\text{sg}[p^r]||p^{o})$. By expanding the terms, we derive an overall formulation of $-p^{o}\log \mathrm{sg}[p^r] - \mathrm{sg}[p^r]\log p^{o}$, which expands to a symmetric cross-entropy between online and target branches. This symmetry encourages the two branches to exchange reliable signals while preventing either from drifting too far—an effect we find important for stable adaptation.

\begin{table}[h]
    \vspace{-0.02in}
    \caption{\blue{Sensitivity of $\alpha$ in \methodname. We report \textbf{Accuracy (\%)} on ImageNet-C across severities with ViT-Base under \textbf{\textsc{online imbalanced label shifts}}, where the imbalance ratio is $\infty$.}}
     \vspace{-0.06in}
    \label{tab:sensitivity_alpha}
\newcommand{\tabincell}[2]{\begin{tabular}{@{}#1@{}}#2\end{tabular}}
 \begin{center}
 \begin{threeparttable}
    \resizebox{0.85\linewidth}{!}{
 	\begin{tabular}{l|cc|ccccc}
 	 Method & No adapt & Tent & $\alpha=0.5$ & $\alpha=1.0$ & $\alpha=1.5$ & $\alpha=2.0$ & $\alpha=2.5$ \\
 	\midrule
    \blue{Severity Level=1} & \blue{67.9} & \blue{76.8} & \blue{78.8} & \blue{78.9} & \blue{78.6} & \blue{78.0} & \blue{77.0}  \\
    \blue{Severity Level=3} & \blue{53.8} & \blue{69.5} & \blue{73.9} & \blue{74.1} & \blue{73.8} & \blue{73.2} & \blue{71.8}  \\
    Severity Level=5 & 29.9 & 47.3 & 63.9 & 64.1 & 63.9 & 63.3 & 62.4  \\	\end{tabular}
	}
	 \end{threeparttable}
	 \end{center}
  \vspace{-0.07in}
\end{table}

we provide a sensitivity analysis of $\alpha$ under different domain shift severities. As shown in Table~\ref{tab:sensitivity_alpha}, by aligning the gradient scale of both terms, ZeroSiam achieves consistently the best results with $\alpha=1$ across severities of domain shift, outperforming Tent's accuracy by 4.6\% (at level 3) and 16.8\% (at level 5). Extensive results across various test scenarios (Tables~\ref{tab:imagenet-c-label-shift}-\ref{tab:imagenet-c-mild}) also validate the efficacy of this setup, making it a reasonable default to avoid careful hyperparameter tuning in TTA.

}

\textbf{Generality of \methodname's architecture} 
\methodname introduces a general asymmetric architecture without limiting the use of self-training objectives\footnote{In Table~\ref{tab:zerosiam-generality}, SLR denotes the soft likelihood ratio~\citep{marsden2024universal}, CE denotes cross-entropy loss with its own predicted labels, and $-p^2$ denotes the negative squared-probability loss, \ie, $-\sum_k p^2_k$.}
and divergence terms. As shown in Table~\ref{tab:zerosiam-generality}, \methodname remains effective across a wide range of combinations and consistently outperforms the single-branch baselines, \eg, 44.0\% (SLR)  \textit{vs.} 63.2\% (SLR+KL). This arises because collapsed solutions and non-generalizable shortcuts remain pervasive across many unsupervised objectives, 
% where \methodname helps regulate biases and prevent collapsed solutions, suggesting our versatility.
suggesting a promising direction for extending \methodname to a wider range of tasks in the future.
\blue{For entropy minimization, we use symmetric KL as the divergence with $\alpha=1$ as discussed in Table~\ref{tab:sensitivity_alpha}.}

\section{Conclusion}

%%% Short VERSION %%%

\blue{In this paper, we explore asymmetry as a new principled test-time entropy minimization mechanism that improves learning stability by inherently avoiding collapsed trivial solutions.} To this end, we devise a lightweight Siamese architecture for asymmetric entropy minimization (\textit{termed} \methodname) to prevent the model from being exposed to collapsed solutions, using an injected learnable predictor. The predictor converts collapsed solutions into explicit discrepancies, which are then eliminated by an alignment term. We empirically and theoretically demonstrate that \methodname not only prevents collapse but also absorbs and regularizes biased \blue{shortcut} signals during learning, thereby boosting the performance even when collapse does not occur.
Extensive experiments verify that \methodname performs more stably over prior methods with negligible overhead, proving effective in both vision adaptation and LLM reasoning tasks across challenging models and test scenarios.

\section*{Acknowledgment}
% This research was supported by Ministry of Education, Singapore, under its Academic Research Fund Tier 1.
This research was supported, in part, by the Joint WeBank-NTU Research Institute on Fintech, Nanyang Technological University, Singapore, and Ministry of Education, Singapore, under its Academic Research Fund Tier 1.

\section*{Ethics Statement}

This research focuses on improving the stability of test-time entropy minimization through an asymmetric Siamese architecture, to advance robustness and reliability in vision and language models for academic and socially beneficial purposes. The proposed approach is evaluated solely on widely used, publicly available models and benchmark datasets under their respective licenses. The study does not involve human subjects or sensitive personal data, and thus does not raise privacy, fairness, or discrimination concerns. No conflicts of interest or external sponsorship influence this work, and all experiments were conducted in accordance with established standards of research integrity.

\section*{Reproducibility Statement}

In this work, we implement all methods (all compared methods and our \methodname) with different models (ResNet50-GN, ViT-Base, ViT-Small, ConvNeXt-Tiny, Swin-Tiny, and Llama3.1-8B-Instruct) on the ImageNet-C, Math-500, College-Math, AIME-24, and Minerva datasets. Reproducing all the results in our paper depends on the following three aspects:
\begin{enumerate}[leftmargin=*]
    \item \textbf{\textsc{Datasets.}} The first paragraph of Section~\ref{sec:main_exp} and Appendix~\ref{sec:more_dataset_details} provide the details of the adopted datasets and the download \texttt{URL}s.
    
    \item \textbf{\textsc{Models.}} All adopted models (with the pre-trained weights) for test-time adaptation are publicly available. Specifically, all vision models: \{ResNet50-GN, ViT-Base, ViT-Small, ConvNeXt-Tiny, Swin-Tiny\} are from \texttt{timm} repository~\citep{rw2019timm}, and Llama3.1-8B-Instruct is from its official repository. Appendix~\ref{sec:more_exp_details} provides the download \texttt{URL}s of them.
    
    \item \textbf{\textsc{Protocols of each method.}} The second paragraph of Section~\ref{sec:main_exp} and Appendix~\ref{sec:more_exp_details} provides the implementation details of all compared methods and our \methodname. We reproduce all compared methods based on the code from their official GitHub, for which the download \texttt{url} is provided (in Appendix~\ref{sec:more_exp_details}) following each method introduction.
\end{enumerate}

\bibliography{iclr2026_conference}
\bibliographystyle{iclr2026_conference}

\clearpage
\appendix
% \section{Appendix}
\clearpage
\setcounter{table}{0}
\setcounter{figure}{0} 
\renewcommand{\thetable}{\Alph{table}}
\renewcommand{\thefigure}{\Alph{figure}}
\renewcommand\theHtable{Appendix.\thetable}
\renewcommand\theHfigure{Appendix.\thefigure}
\newtcolorbox[auto counter]{mybox}[1][]{
  title=My Box \thetcbcounter,
  label=mybox:\thetcbcounter,
  #1
}

\newpage
\appendix
\begin{leftline}
	{
		\LARGE{\textsc{Appendix}}
	}
\end{leftline}

\etocdepthtag.toc{mtappendix}
\etocsettagdepth{mtchapter}{none}
\etocsettagdepth{mtappendix}{subsection}

{
    \hypersetup{linkcolor=black}
    \footnotesize\tableofcontents
}
\clearpage

% (entropy minimization for TTA for RL, 给一点reasoning的related work)
% test-time entropy minimization: tta, ttrl
\section{Related Work}\label{sec:related-work}
In this section, we connect test-time entropy minimization and self-supervised learning through their shared challenge of collapse, and relate our approach to these areas.

\textbf{Test-Time Entropy minimization} seeks to promote confident predictions by minimizing the prediction entropy on test data. 
% \textbf{Test-Time Entropy minimization} seeks to reduce the model uncertainty by minimizing entropy on test data. 
Tent~\citep{wang2021tent} first exploits this scheme for test-time model adaptation to novel and out-of-distribution environments. Unlike \textit{test-time training} approaches~\citep{sun2020test,liu2021ttt++,bartler2022mt3,gandelsman2022tmae}, which train an extra auxiliary self-supervised branch to produce learning signals from test data, test-time entropy minimization~\citep{wang2021tent,niu2022eata,marsden2024universal,lee2024deyo,zhang2025come,hu2025beyond,zhang2025otvp} removes reliance on the source training process and adapts an arbitrary model on-the-fly by using its own predicted labels for self-training, rendering it practical in real-world applications~\citep{niu2024foa,liang2025comprehensive}. Recently, test-time entropy optimization has attracted growing interest in various fields, such as incentivizing knowledge in large language models during inference~\citep{jang2025self,zuo2025ttrl,Hu2025TLM}, and adapting a fixed policy network to evolving environments in real-time~\citep{xu2025test}.
% , showing great potential. 

Nevertheless, as highlighted by~\citet{niu2023sar}, entropy minimization is inherently unstable and prone to collapse, \eg, converging towards constant one-hot outputs that trivially minimize the entropy loss without meaningful predictions. To improve stability, subsequent methods explored strategies such as sample filtering~\citep{niu2023sar,lee2024deyo,Hu2025TLM}, uncertainty quantification~\citep{zhang2025come}, and parameter selection~\citep{chen2024cross,marsden2024universal,zhang2025dpcore}. Despite these advances, collapsed constant solutions remain valid optima in prior methods, so stability often hinges on heuristics, thresholds, or extra compute, leaving methods sensitive to architectures and domains, where robust architectural design is largely underexplored. Our \methodname fills this gap by introducing asymmetry that rules out trivial minima in test-time entropy minimization, while remaining efficient by avoiding augmentations and extra encoder passes.

\textbf{Self-supervised Learning}
aims to learn discriminative and consistent representations by creating pretext tasks such as contrastive learning~\citep{oord2018representation,saunshi2019theoretical,he2020momentum,chen2020simple}. 
A central challenge in SSL is trivial collapse, where all outputs degenerate to a constant to maximize representation similarity. Contrastive methods such as SimCLR~\citep{chen2020simple} and MoCo~\citep{he2020momentum} prevent collapse by pushing apart negative pairs, but rely on large batch sizes, memory banks, or careful negative mining strategies. To remove dependency on negative samples, BYOL~\citep{grill2020bootstrap} and SimSiam~\citep{chen2021exploring} introduce asymmetric architectures, preventing collapse with a \textit{predictor–stop-gradient asymmetry} to break the stability of the collapsed constant solution. Although subsequent methods such as Barlow Twins~\citep{zbontar2021barlow} and VicReg~\citep{bardes2022vicreg} achieve stability without symmetry-breaking design by maximizing feature diversity within a batch, they implicitly assume large batches and i.i.d. streams, which often fail during testing. Inspired by SimSiam, we revisit asymmetry in the context of single-branch, entropy-based TTA, and show how it can be instantiated in a minimal and efficient manner. 
% to substantially improve TTA stability.

% 丢附录去
% \textbf{Reinforcement Learning.} Entropy collapse is also a common challenge within reinforcement learning, where policies can degenerate (\eg, become deterministic and lose diversity) due to biased or noisy advantage estimation, insufficient exploration, or overly aggressive updates~\citep{mnih2016asynchronous,haarnoja2018soft,ahmed2019understanding}. To address this, trust-region methods such as TRPO~\citep{schulman2015trust} and PPO~\citep{schulman2017proximal} constrain each update step in policy space through KL-divergence penalties or clipping, ensuring gradual and stable improvement. Similar principles appear in value-based methods~\citep{mnih2015human,schwarzer2023bigger,nauman2024bigger}, where target networks provide a stable anchor to stabilize bootstrapped updates against rapidly changing estimates. These approaches share with \methodname the key idea of \textit{structurally constraining updates to counter collapse}. Specifically, \methodname's online branch with gradient updates resembles a fast-updating policy (\ie, with extra updates in the predictor), while the stop-gradient target branch serves as an anchor, effectively forming a lightweight trust region in the entropy minimization landscape. This connection situates \methodname as a general collapse-prevention mechanism, linking stabilization strategies across TTA, self-supervised learning, and reinforcement learning.

\clearpage
\section{Theoretical Analysis}\label{suppl:sec:proofs}
\textbf{Notations.} Let $\mathcal{L} = H(p^o) + \alpha \, D\!\left(p^o \,\|\, \mathrm{sg}[p^r]\right)$ be the optimization objective of \methodname, where $H(\cdot)$ is the entropy loss, $D(\cdot)$ the alignment regularizer, and $p^o, p^r \in \Delta^{|\mathcal{C}|-1}$ are the probability distributions induced by the online and target branches.
Let the encoder $f$ and classifier $g$ be pretrained, and $h$ a lightweight \textit{linear} predictor. At initialization, we set $h$ as the identity mapping $\mathbf{I}$, so that the online and target branches coincide and both lie in the non-collapse region $\{ p \in \Delta^{|\mathcal{C}|-1} | \min_{c} p_c \ge \delta >0 \}$.

\subsection{Assumptions}
\begin{ass} (\textbf{Lipschitz continuity of the encoder)} \label{proofs: ass1}
    Let $f(\cdot;\theta_{f})$ be an encoder parameterized by $\theta_{f}$. We assume $f$ is Lipschitz continuous with respect to its parameters. Specifically, there exists a constant $L_f > 0$ such that for any $\theta_{1}, \theta_{2}$ and any input $\mathbf{x}$ in the domain,
    \begin{equation*}
        \| f_{\theta_{1}}(\mathbf{x}) - f_{\theta_{2}}(\mathbf{x}) \|_{2} \le L_{f} \| \theta_{1} - \theta_{2} \| .
    \end{equation*}
\end{ass}

Assumption \ref{proofs: ass1} bounds the sensitivity of the encoder $f$ to parameter perturbations, thereby ensuring stability for subsequent optimization analysis, and is commonly employed in prior theoretical studies on representation learning and model convergence.

% Given that only the affine parameters of the normalization layers are optimized, this assumption is well justified.

\begin{ass} (\textbf{Standard optimization assumptions}) \label{proofs: ass2}
    Let $\mathcal{L}(\cdot)$ denote the objective function, $\phi = (\theta_{h}, \theta_{f})$ the set of optimization parameters, and $\eta_{h}, \eta_{f}$ the learning rates for the predictor and encoder, respectively. Our analysis relies on the following standard assumptions, which are common in the optimization literature:\\
    (1) \textit{Smoothness}: $\mathcal{L}$ is $\rho$-smooth. Specifically, there exists $\rho > 0$ such that for all $\phi_{1}, \phi_{2}$, 
    % $\| \nabla_{\phi}\mathcal{L}(\phi_{1}) - \nabla_{\phi}\mathcal{L}(\phi_{2}) \| \leq \rho \| \phi_{1} - \phi_{2} \|$ ;\\
     \begin{equation*}
         \| \nabla_{\phi}\mathcal{L}(\phi_{1}) - \nabla_{\phi}\mathcal{L}(\phi_{2}) \| \le \rho \| \phi_{1} - \phi_{2} \| ;
     \end{equation*}
    % 2) (Lipschitz Continuity): $\mathcal{L}$ is Lipschitz continuous with respect to the input;\\
    (2) \textit{Descent condition}: The learning rates satisfy $\max(\eta_{h}, \eta_{f}) < \frac{1}{\rho}$, ensuring monotonic descent.
    % 4) (Non-degenerate Initialization): $\min_k y_{n,k}^{(0)} \geq \delta > 0$ and $y_n^{(0)}$ is not uniform, guaranteeing $\|\nabla_\phi H\| > 0$.
\end{ass}

Assumption~\ref{proofs: ass2} constrains the appropriate range of learning rates to guarantee monotonic descent of $\mathcal{L}(\cdot)$, which is a common assumption in optimization research.

\subsection{Proof of Lemma \ref{proofs: lemma1}}

\begin{lemma} \textbf{(Gradient dominance in entropy descent)} \label{proofs: lemma1}
    Let $H(\cdot)$ denote the entropy function, $g(\cdot;\theta_g), h(\cdot;\theta_h)$ and $f(\cdot;\theta_f)$ denote the classifier, the predictor and the encoder, $\phi = (\theta_{h}, \theta_{f})$ the set of optimization parameters, and $\eta_{h}, \eta_{f}$ the learning rates for the predictor and encoder, respectively. Then the discrete change in entropy of the target branch satisfies:
    \begin{equation*}
        \left \| \nabla_{\phi}H(p^{o}(t)) \right \| \ge \frac{1}{L_{H} \left \| g \right \|_2 L_{f} \cdot \max \left ( \eta_{h},\eta_{f} \right )} \cdot \left | H(p^{r}(t+1)) - H(p^{r}(t)) \right | ,
    \end{equation*}
    where $L_{H}, L_{f}$ are the Lipschitz constants of $H(\cdot)$ and $f$, $\| g \|_{2}$ denotes the spectral norm of the classifier weight vector, and $t$ is the optimization iteration step.
\end{lemma}

\begin{proof}
    Consider the online and target branches:
    \begin{equation*}
        \begin{cases}
          p^o = \text{softmax}(g(h(f(\mathbf{x};\theta_f);\theta_h);\theta_g)) & \text{(online branch)} , \\
          p^r = \text{softmax}(g(f(\mathbf{x};\theta_f);\theta_g)) & \text{(target branch)} .
        \end{cases}
    \end{equation*}
    By design, the encoder $f$ of the target branch is synchronized with the online branch after each update:
    \begin{equation*}
        \theta_{f}^{r} (t+1) = \theta_{f}^{o} (t+1) = \theta_{f}^{o} (t) - \eta_{f} \nabla_{\theta_{f}}H(p^{o}(t)) .
    \end{equation*}
    Thus, we bound the entropy change of the target branch $|H(p^{r}(t+1)) - H(p^{r}(t))|$ by tracing the propagation of parameter updates through the online branch.
    
    According to the Assumption \ref{proofs: ass1}, the encoder $f$ is $L_f$-Lipschitz continuous \wrt $\theta_f$, we have
    \begin{equation}
        \| \mathbf{z}(t+1) - \mathbf{z}(t) \|_{2} = \| f(\mathbf{x};\theta_{f}(t+1)) - f(\mathbf{x};\theta_{f}(t)) \|_{2} \le L_{f} \| \theta_{f}(t+1) - \theta_{f}(t) \| .
    \label{eqn: l_f}
    \end{equation}
    where $\mathbf{z} = f(\mathbf{x};\theta_f)$ is the output feature of the encoder.
    
    Note that in the non-collapse region $\{ p \in \Delta^{|\mathcal{C}|-1} | \min_{c} p_c \ge \delta >0 \}$, the entropy function $H(\cdot)$ is Lipschitz continuous with constant $L_{H} = \sqrt{|\mathcal{C}|} (|\log \delta| + 1)$, thus:
    \begin{equation}
        |H(p(t+1)) - H(p(t))| \le L_{H} \| p(t+1) - p(t) \|_{2} .
    \label{eqn: l_H}
    \end{equation}
    
    Since the pre-trained classifier $g$ is fixed, and the softmax function is $1$-Lipschitz, combining Eqn. (\ref{eqn: l_f}) and Eqn. (\ref{eqn: l_H}), we have the following chain:
    \begin{equation*}
        \begin{aligned}
            |H(p^{r}(t+1)) - H(p^{r}(t))| &\le L_{H} \| p^{r}(t+1) - p^{r}(t) \|_{2} \\
            &\le L_{H} \| g \|_{2} \| \mathbf{z}^{r}(t+1) - \mathbf{z}^{r}(t) \|_{2} \\
            &\le L_{H} \| g \|_{2} L_{f} \| \theta_{f}^{r}(t+1) - \theta_{f}^{r}(t) \| \\
            &= L_{H} \| g \|_{2} L_{f} \eta_{f} \| \nabla_{\theta_{f}}H(p^{o}(t)) \| ,
        \end{aligned}
    \end{equation*}
    where $\| g \|_{2}$ is the spectral norm of the classifier weight vector.
    
    Let $\phi = (\theta_{h}, \theta_{f})$ denote the set of optimization parameters, and $\eta_{h}, \eta_{f}$ the learning rates for the predictor and encoder. Since:
    \begin{gather*}
        \eta_{f} \le \max(\eta_{h}, \eta_{f}) , \\
        \| \nabla_{\theta_{f}}H(p^{o}(t)) \| \le 
        \begin{Vmatrix}
        \begin{bmatrix}
         \nabla_{\theta_{f}}H(p^{o}(t)) \\
         \nabla_{\theta_{h}}H(p^{o}(t))
        \end{bmatrix}
        \end{Vmatrix}
        = \| \nabla_{\phi}H(p^{o}(t)) \| ,
    \end{gather*}
    we obtain:
    \begin{equation*}
        \left \| \nabla_{\phi}H(p^{o}(t)) \right \| \ge \frac{1}{L_{H} \left \| g \right \|_2 L_{f} \cdot \max \left ( \eta_{h},\eta_{f} \right )} \cdot \left | H(p^{r}(t+1)) - H(p^{r}(t)) \right | .
    \end{equation*}
\end{proof}

\subsection{Proof of Lemma \ref{proofs: lemma2}}

\begin{lemma} \textbf{(Collapse Direction with Maximal Negative Curvature)} \label{proofs: lemma2}
    Let $H(\cdot)$ denote the entropy function, $u = g(h(f(\mathbf{x};\theta_f);\theta_h);\theta_g)$ the uncertainty logits, and $p = \text{softmax}(u)$ the induced probability distribution. For any class $c \in \mathcal{C}$, there exists a direction $\mathbf{v} = (\mathbf{v}_h, \mathbf{v}_f) \in \mathbb{R}^{\dim(\phi)}$ such that:\\
    (1) $p_c$ increases monotonically along $\mathbf{v}$; \\
    (2) As $p_c \to 1$, it holds that $\mathbf{v}^\top \nabla^{2}_{\phi}H \mathbf{v} \to \lambda_{\min} (\nabla^{2}_{\phi}H)$; \\
    (3) For any $\mathbf{w} \bot \mathbf{v}$, we have $\mathbf{v}^\top \nabla^{2}_{\phi}H \mathbf{v} \le \mathbf{w}^\top \nabla^{2}_{\phi}H \mathbf{w}$.
\end{lemma}

\begin{proof}
    (1) 
    Let $\mathbf{z} = f(\mathbf{x};\theta_f) \in \mathbb{R}^d$ be the encoder output, $u = g(h(\mathbf{z};\theta_h);\theta_g) \in \mathbb{R}^{|\mathcal{C}|}$ the uncertainty logits, and $\mathbf{W}, \mathbf{P}$ the weight vector of $g, h$, respectively. Define $\mathbf{v} = (\mathbf{v}_h, \mathbf{v}_f)$ such that:
    \begin{equation*}
        \begin{cases}
         \mathbf{v}_{h} = a \cdot \mathbf{W}_k^\top r , \\
         \mathbf{v}_{f} = \nabla_{\theta_{f}} (r^\top \mathbf{z}) ;
        \end{cases}
    \end{equation*}
    where $r$ is a feature direction satisfying $r^\top \mathbf{z} \ne 0$ and $a > 0$. $\mathbf{v}_{f}$ is the direction in normalization layer parameters that amplifies the feature component along $r$.

    This joint direction $\mathbf{v} = (\mathbf{v}_h, \mathbf{v}_f)$ induces a change in logits:
    \begin{equation*}
        \Delta u = \mathbf{W} (\mathbf{v}_{h}) \mathbf{z} + \mathbf{W} \mathbf{P} (\nabla_{\theta_{f}} \mathbf{z} \mathbf{v}_{f}) \approx a (r^\top \mathbf{z}) \mathbf{W} \mathbf{W}_{k}^\top + \mathbf{W} \mathbf{P} r .
    \end{equation*}
    By choosing $r^\top \mathbf{z} > 0$ and sufficiently large $a$, we ensure:
    \begin{equation*}
        [\Delta u]_c > 0,~~~[\Delta u]_j = \mathcal{O}(1)~~\text{for}~~j \ne c .
    \end{equation*}
    Thus, moving along $\mathbf{v}$ increases $u_c$, and hence $p_c = \text{softmax}(u_c)$, monotonically.

    (2) 
    The Hessian of entropy $H(\cdot)$ \wrt logits $u$ is:
    \begin{equation*}
        [\nabla^2_u H]_{ij} = \frac{\partial^2 H}{\partial u_i \partial u_j} = p_i (\delta_{ij} - p_j)(H - \log p_i - 1) .
    \end{equation*}
    As $p_c \to 1$, let $\mu = 1 - p_c \to 0^+$, and for $p_j \approx \frac{\mu}{|\mathcal{C}|-1}$ for $j \ne c$. Then in the Hessian matrix
    \begin{equation*}
        \begin{cases}
          \frac{\partial^2 H}{\partial u_c^2} = p_c (1 - p_c)(H - \log p_c - 1) , & \text{Diagonal for}~c , \\
          \frac{\partial^2 H}{\partial u_j^2} = p_j (1 - p_j)(H - \log p_j - 1) , & \text{Diagonal for}~j \ne c , \\
          \frac{\partial^2 H}{\partial u_i \partial u_j} = - p_i p_j (H - \log p_i -\log p_j - 1) , & \text{Off-diagonal} ,
        \end{cases}
    \end{equation*}
    we obtain the following approximation:
    \begin{equation*}
        \begin{cases}
         \frac{\partial^2 H}{\partial u_c^2} \approx (1 - \mu) \mu (H + \mu - 1) \approx \mu (\mu \log \frac{|\mathcal{C}|-1}{\mu} + \mu - 1) \approx - \mu , \\
         \frac{\partial^2 H}{\partial u_j^2} \approx \frac{\mu}{|\mathcal{C}|-1} (\mu \log \frac{|\mathcal{C}|-1}{\mu} - \log \frac{\mu}{|\mathcal{C}|-1} - 1) \approx \frac{\mu}{|\mathcal{C}|-1} \log \frac{|\mathcal{C}|-1}{\mu} > 0 , \\
         \frac{\partial^2 H}{\partial u_i \partial u_j} \approx - p_i p_j (-\log p_j) > 0 .
        \end{cases}
    \end{equation*}
    That is, as $p_c \to 1$, we have
    \begin{equation*}
        \mathbf{v}^\top \nabla^{2}_{\phi}H \mathbf{v} \to \lambda_{\min} (\nabla^{2}_{\phi}H) .
    \end{equation*}

    (3) 
    The parameter space Hessian quadratic form along $\mathbf{v}$ is:
    \begin{equation*}
        \mathbf{v}^\top \nabla^{2}_{\phi} H \mathbf{v} = \Delta u^\top \nabla^{2}_{u} H \Delta u + \sum_{i} \frac{\partial H}{\partial u_i} \mathbf{v}^\top \nabla_{\phi}^2 u_i \mathbf{v} .
    \end{equation*}
    As $p_c \to 1$, $\nabla_{u} H \to 0$, so the higher-order terms vanishes. Since $\Delta u = \frac{\partial u}{\partial \phi} \propto e_c$, we have:
    \begin{equation*}
        \mathbf{v}^\top \nabla^{2}_{\phi} H \mathbf{v} \propto e_c^\top \nabla^{2}_{u} H e_c = \frac{\partial^2 H}{\partial u_c^2} < 0 .
    \end{equation*}
    For any $\mathbf{w} \bot \mathbf{v}$, $\Delta u_{\mathbf{w}} = \frac{\partial u}{\partial \phi} \mathbf{w}$ is not parallel to $e_c$, so:
    \begin{equation*}
        \mathbf{w}^\top \nabla^{2}_{\phi}H \mathbf{w} \ge \mathbf{v}^\top \nabla^{2}_{\phi}H \mathbf{v} .
    \end{equation*}
\end{proof}

\subsection{Proof of Theorem \ref{proofs: thm1}}

\textbf{Theorem \ref{proofs: thm1}. \textit{(Optimization and Stability of \methodname)}}
\emph{
    Consider the \methodname objective $\mathcal{L} = H(p^o) + \alpha \, D\!\left(p^o \,\|\, \mathrm{sg}[p^r]\right)$, where $H(\cdot)$ denotes the entropy loss, $D(\cdot)$ the alignment regularizer, and $p^o, p^r \in \Delta^{|\mathcal{C}|-1}$ are the probability distributions induced by the online and target branches. Given the encoder $f$, classifier $g$ and predictor $h$. Under Assumptions~\ref{proofs: ass1} and~\ref{proofs: ass2}, the following hold:  \\
    % (1) For $\lambda = 0$, the entropy variation satisfies $|\Delta H(p_o)| > |\Delta H(p_r)|$; \\
    % (2) For $\lambda = 0$, it holds that
    % \begin{equation}
    %     \lambda_{\min}\!\left(\nabla^2 H(p_o)\right) = \mathbf{v}^{\top}\nabla^2 H(p_o)\mathbf{v}, \quad \mathbf{v} \in \mathcal{V}_{\text{collapse}} ,
    % \end{equation}
    % where $\mathcal{V}_{\text{collapse}}$ denotes collapse directions; \\  
    (1) For $\alpha = 0$, the entropy variation satisfies $|\Delta H(p^o)| > |\Delta H(p^r)|$, and the Hessian of \(H(p^o)\) attains its minimal eigenvalue along collapse directions $\mathbf{v}$:
    \begin{equation*}
        \lambda_{\min}\!\left(\nabla^2 H(p^o)\right) = \mathbf{v}^{\top}\nabla^2 H(p^o)\mathbf{v} .
    \end{equation*} 
    (2) For $\alpha > 0$, the predictor $h$ serves as a filtering mechanism that suppresses gradient update directions corresponding to over-amplified logits, and the system converges to a stable equilibrium: there exists $h_{\min} > 0$ such that
    \begin{equation*}
        H(p^o) > h_{\min},~p^o \to p^r .
    \end{equation*}
}

\begin{proof}
    (1) 
    By Lemma \ref{proofs: lemma1}, the target branch entropy variation is determined by the gradient of the online branch:
    \begin{equation}
        \| \nabla_{\phi}H(p^{o}(t)) \| \ge \frac{1}{L \cdot \eta_{\max}} \cdot | H(p^{r}(t+1)) - H(p^{r}(t)) | ,
    \label{eqn: delta_Hr}
    \end{equation}
    where $L = L_{H} \| g \|_2 L_{f}$ and $\eta_{\max} = \max (\eta_{h},\eta_{f})$.
    
    Under the Assumption \ref{proofs: ass2} that $\mathcal{L} = H(\cdot)$ is $\rho$-smoothness and $\eta_{\max} < \frac{1}{\rho}$, the descent lemma gives:
    \begin{equation*}
        H(p^o(t+1)) \le H(p^o(t)) - \nabla_{\phi} H^\top D_{\eta} \nabla_{\phi} H + \frac{\rho}{2} \| D_{\eta} \nabla_{\phi} H \|^2 ,
    \end{equation*}
    where $D_{\eta} = \text{diag}(\eta_f \mathbf{I}_{d_f}, \eta_h \mathbf{I}_{d_h})$. So we have:
    \begin{equation}
        |\Delta H^o| = H(p^o(t+1)) - H(p^o(t)) \ge \eta_{\min} \| \nabla_{\phi} H \|^2,~\eta_{\min} = \min (\eta_{h},\eta_{f}) .
    \label{eqn: delta_Ho}
    \end{equation}
    Combining Eqn. (\ref{eqn: delta_Hr}) and Eqn. (\ref{eqn: delta_Ho}), we obtain:
    \begin{equation*}
        \frac{|\Delta H^o|}{|\Delta H^r|} \ge \frac{\eta_{\min} \| \nabla_{\phi} H \|^2}{L \cdot \eta_{\max} \| \nabla_{\phi}H \|} = \frac{\eta_{\min}}{\eta_{\max}} \cdot \frac{\| \nabla_{\phi}H \|}{L} .
    \end{equation*}
    Since $L, \| \nabla_{\phi}H \| > 0$ and bounded in non-collapse regions, achieving $\frac{|\Delta H^o|}{|\Delta H^r|} > 1$ requires:
    \begin{equation*}
        \frac{\eta_{\min}}{\eta_{\max}} > \frac{L}{\| \nabla_{\phi}H \|} .
    \end{equation*}
    In practice, we set $\eta_h = k \cdot \eta_f$ with $k > 1$, so there always exists a $k$ such that
    \begin{equation*}
        |\Delta H(p^o)| > |\Delta H(p^r)| .
    \end{equation*}

    We further construct the direction $\mathbf{v} = (\mathbf{v}_h, \mathbf{v}_f) \in \mathbb{R}^{\dim(\phi)}$ according to Lemma \ref{proofs: lemma2}, the corresponding directional derivative is given by:
    \begin{equation*}
        D_{\mathbf{v}} H = \nabla_{\phi} H^\top \mathbf{v} \propto [\nabla_{u} H]_c ,
    \end{equation*}
    where $u$ is the uncertainty logits. And the second-order directional derivative is:
    \begin{equation*}
        D_{\mathbf{v}}^2 H = \mathbf{v}^\top \nabla_{\phi}^2 H \mathbf{v} = \lambda_{\min} < 0 .
    \end{equation*}
    This implies that along $\mathbf{v}$, the loss decreases fastest at second order~\citep{jin2017escape}.

    Moreover, by Lemma \ref{proofs: lemma2}, $p_c$ increases monotonically along $\mathbf{v}$, thus $[\nabla_{u} H]_c = p_c (H - \log p_c - 1)$ increases synchronously, creating a self-reinforcing feedback loop that accelerates collapse.

    (2) 
    Consider the optimization objective:
    \begin{equation*}
        \mathcal{L} = H(p^o) + \alpha \, D\!(p^o \,\|\, \mathrm{sg}[p^r]) , 
    \end{equation*}
    where $D\!(p^o \,\|\, \mathrm{sg}[p^r]) = \KL \!(p^o \,\|\, \mathrm{sg}[p^r]) + \KL \!(\mathrm{sg}[p^r] \,\|\, p^o)$ are symmetric KL loss.

    Under the stop-gradient setting on the target branch, the gradient of the shared encoder $f$ is solely determined by the online branch. We have:
    \begin{equation*}
        \nabla_{\mathbf{z}} \mathcal{L} = (\frac{\partial u^o}{\partial \mathbf{z}}) \nabla_{u^o} \mathcal{L} = \mathbf{P}^\top \mathbf{W}^\top [\nabla_{u^o} H(p^o) + 2 \alpha (p^o - p^r)] ,
    \end{equation*}
    where $\mathbf{z} = f(\mathbf{x};\theta_f)$ is the encoder output, $\mathbf{W}, \mathbf{P}$ the weight vector of the classifier $g$ and the predictor $h$, respectively.
    % And $h$ is updated to minimize $\mathcal{L}$ with gradient:
    % \begin{equation*}
    %     \nabla_{\theta_{h}} \mathcal{L} = \mathbf{z} \mathbf{W}^\top [\nabla_{u^o} H(p^o) + 2 \alpha (p^o - p^r)] .
    % \end{equation*}
    Consider a parameter update unit direction $\mathbf{v_{h}}$, $\| \mathbf{v_{h}} \|_2 = 1$, then:
    \begin{equation*}
        <\nabla_{\mathbf{z}} \mathcal{L}, \mathbf{v_{h}}> = \mathbf{v_{h}}^\top \nabla_{\mathbf{z}} \mathcal{L} = \mathbf{v_{h}}^\top (\frac{\partial u^o}{\partial \mathbf{z}}) \nabla_{u^o} \mathcal{L} = \mathbf{P}^\top \mathbf{W}^\top [\nabla_{u^o} H(p^o) + 2 \alpha (p^o - p^r)]
    \end{equation*}
    During the minimization of $\mathcal{L}$, the predictor $h$ encourages the update direction that ensuring
    \begin{equation*}
        p^o \to p^r ,
    \end{equation*}
    and suppresses gradient update directions corresponding to over-amplified logits.
    
    We further analyze the convergence properties of the optimization system.

    By the definition of entropy and KL divergence, we have:
    \begin{equation*}
        H(p^o) = -\sum_{c=1}^{|\mathcal{C}|} p^o_c \log p^o_c = -\sum_{c=1}^{|\mathcal{C}|}  p^o_c \log p^r_c - \KL (p^o || p^r) .
    \end{equation*}
    Applying Gibbs’ inequality, $-\sum_{c=1}^{|\mathcal{C}|}  p^o_c \log p^r_c \ge H(p^r)$, we obtain:
    \begin{equation*}
        H(p^o) \ge H(p^r) - \KL(p^o || p^r) .
    \end{equation*}
    Since $D_{\text{SKL}}(p^o || p^r) \ge \KL(p^o || p^r)$, it follows that:
    \begin{equation}
        H(p^o) \ge H(p^r) - D_{\text{SKL}}(p^o || p^r) .
    \label{eqn: bound_hpo}
    \end{equation}
    When the target branch is still in the non-collapse region $\{ p^r \in \Delta^{|\mathcal{C}|-1} | \min_{c} p^r_c \ge \delta >0 \}$, the entropy satisfies:
    \begin{equation}
        H(p^r) \ge -\log (1 - (1 - |\mathcal{C}|) \delta) >0 .
    \label{eqn: bound_hpc}
    \end{equation}
    Define the Lyapunov function $V(t) = \mathcal{L}(t)$. Since $H(p^o) \ge 0$ and $D_{\text{SKL}}(p^o || p^r) \ge 0$, we have
    \begin{equation*}
        V(t) > 0 .
    \end{equation*}
    Moreover, under the Assumption \ref{proofs: ass2}, the descent lemma guarantees:
    \begin{equation*}
        V(t+1) \ge V(t) - \frac{\eta_{\min}}{2} \| \nabla_{\phi} \mathcal{L}(t) \|^2 \le V(t) .
    \end{equation*}
    Thus, $V(t)$ is monotonically decreasing and bounded below. According to the Monotone Convergence Theorem, $V(t) \to V^* > 0$.
    In particular:
    \begin{equation}
        D_{\text{SKL}}(p^o(t) || p^r(t)) \le \frac{V(0)}{\alpha} .
    \label{eqn: bound_skl}
    \end{equation}
    Combining Eqn. (\ref{eqn: bound_hpo}), Eqn. (\ref{eqn: bound_hpc}), and Eqn. (\ref{eqn: bound_skl}):
    \begin{equation*}
        H(p^o) \ge H(p^r) - D_{\text{SKL}}(p^o || p^r) \ge -\log (1 - (1 - |\mathcal{C}|) \delta) - \frac{V(0)}{\alpha} .
    \end{equation*}
    Define
    \begin{equation*}
        h_{\min} = -\log (1 - (1 - |\mathcal{C}|) \delta) - \frac{V(0)}{\alpha} .
    \end{equation*}
    Let the hyperparameter $\alpha > \frac{V(0)}{-\log (1 - (1 - |\mathcal{C}|) \delta)}$, then $h_{\min} > 0$ .

    Overall, we have proven that the predictor $h$ serves as a filtering mechanism that suppresses gradient update directions corresponding to over-amplified logits, and the system converges to a stable equilibrium: there exists $h_{\min} > 0$ such that
    \begin{equation*}
        H(p^o) > h_{\min},~p^o \to p^r .
    \end{equation*}
\end{proof}
% \textbf{Remark.} Together, the theoretical analyses for both $\alpha=0$ and $\alpha>0$ give complementary insights into how \methodname introduces meaningful discrepancies for regularization and bias filtration, while also demonstrating its convergence toward a stable, and non-collapsing equilibrium, which aligns with our empirical findings of \methodname's stability across diverse challenging adaptation scenarios.

\clearpage

\section{More Implementation Details}\label{sec:more_details}

\subsection{More Details on Dataset}\label{sec:more_dataset_details}

For vision adaptation, we evaluate the out-of-distribution generalization ability of all methods on a large-scale and widely used benchmark, namely \textbf{ImageNet-C}\footnote{\url{https://zenodo.org/record/2235448\#.YzQpq-xBxcA}}~\citep{hendrycks2019benchmarking}. ImageNet-C is constructed by corrupting the original ImageNet~\citep{deng2009imagenet} test set. The corruption (as shown in Figure~\ref{fig:corruption_types}) consists of 15 different types, \ie, Gaussian noise, shot noise, impulse noise, defocus blur, glass blur, motion blur, zoom blur, snow, frost, fog, brightness, contrast, elastic transformation, pixelation, and JPEG compression, where each corruption type has 5 severity levels and the larger severity level means more severe distribution shift. 

% \vspace{-0.3in}
% , which are taken from the original paper of ImageNet-C~\citep{hendrycks2019benchmarking}
\begin{figure*}[h]
\centering\includegraphics[width=0.48\linewidth]{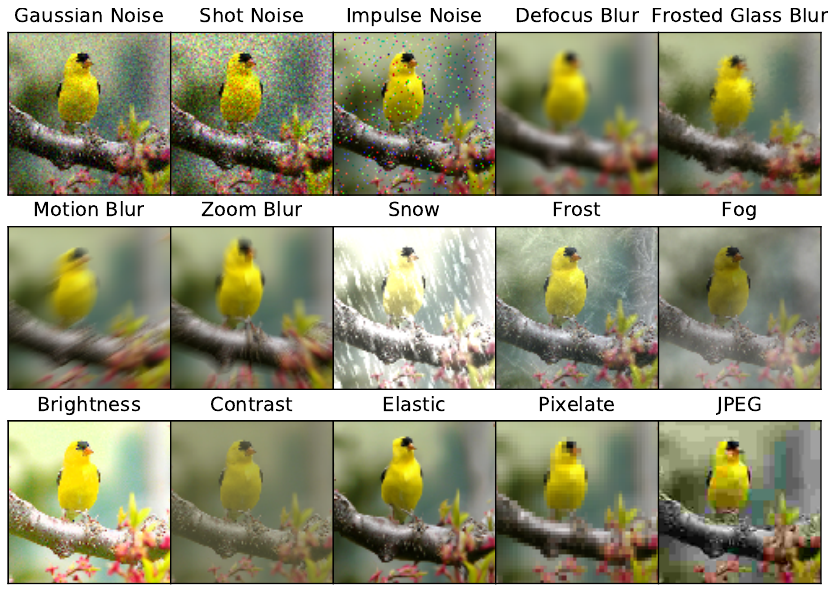}\caption{Visualizations of different corruption types in ImageNet-C benchmark.}
\label{fig:corruption_types}
\end{figure*}

For natural language reasoning, we evaluate the effectiveness of online reasoning incentivization on a comprehensive mathematical reasoning benchmark, including Math-500~\citep{lightman2023let}, CollegeMath~\citep{tang2024mathscale}, AIME24~\citep{codeforcesamerican}, and Minerva~\citep{lewkowycz2022solving}.
% , covering from high-school mathematics to advanced competition problems. 
Specifically, Math-500 is derived from the larger MATH benchmark~\citep{hendrycks2measuring}, consisting of difficult mathematics problems originally taken from high school competitions. It spans domains such as pre-algebra, algebra, number theory, and calculus, emphasizing abstract reasoning and multi-step problem solving.
CollegeMath is a college-level mathematics dataset constructed from nine open-access college textbooks, covering seven core areas: algebra, pre-calculus, calculus, vector calculus, probability, linear algebra, and differential equations. 
% It evaluates a wide range of skills, including analytical thinking, logical deduction, and quantitative problem solving. 
The full dataset contains 2818 test questions, from which we sample 1200 for efficient evaluation.
AIME24 consists of 30 advanced problems from the 2024 American Invitational Mathematics Examination (AIME) I and II, designed to test extended chain-of-thought reasoning.
Minerva includes 272 undergraduate-level mathematics and science questions from MIT OpenCourseWare (OCW), aimed at assessing models’ reasoning ability in scientific problem-solving contexts.

\subsection{More Experimental Protocols on Methods}\label{sec:more_exp_details}

All pre-trained models involved in our paper are publicly available,
including ResNet50-GN\footnote{\url{https://github.com/rwightman/pytorch-image-models/releases/download/v0.1-rsb-weights/resnet50\_gn\_a1h2-8fe6c4d0.pth}}, ViT-Base\footnote{\url{https://storage.googleapis.com/vit\_models/augreg/B\_16-i21k-300ep-lr\_0.001-aug\_medium1-wd\_0.1-do\_0.0-sd\_0.0--imagenet2012-steps\_20k-lr\_0.01-res\_224.npz}}, ViT-Small\footnote{\url{https://storage.googleapis.com/vit_models/augreg/S_16-i21k-300ep-lr_0.001-aug_light1-wd_0.03-do_0.0-sd_0.0--imagenet2012-steps_20k-lr_0.03-res_224.npz}}, ConvNeXt-Tiny\footnote{\url{https://github.com/rwightman/pytorch-image-models/releases/download/v0.1-rsb-weights/convnext_tiny_hnf_a2h-ab7e9df2.pth}} and Swin-Tiny\footnote{\url{https://github.com/SwinTransformer/storage/releases/download/v1.0.0/swin_tiny_patch4_window7_224.pth}} obtained from \texttt{timm} repository~\citep{rw2019timm} for image classification, and Llama3.1-8B-Instruct\footnote{\url{https://huggingface.co/meta-llama/Llama-3.1-8B-Instruct}}~\citep{dubey2024llama} for natural language reasoning. In the following, we provide the implementation details of our proposed method and all comparative methods evaluated in our experiments, which help reproduce our results.

\textbf{\methodname (Ours).} We use symmetric KL as the divergence term in Eqn.~(\ref{eq:zerosiam-loss}), and $\alpha$ is fixed to 1 without tuning. For vision adaptation, we use SGD as the update rule, with a momentum of 0.9, batch size of 64 (except for the experiments of batch size = 1), learning rate $\eta_f$ of 0.0025 / 0.005 / 0.001 / 0.00025 / 0.0005 with learning rate $\eta_h$ of 0.025 / 0.005 / 0.01 / 0.0025 / 0.0025 for ResNet50-GN / ViT-Base / ViT-Small / ConvNeXt-Tiny / Swin-Tiny, respectively. The learning rate for batch size = 1 is scaled down by 16 for ResNet50-GN and 32 for other models following SAR. The trainable parameters are all the affine parameters of normalization layers. For natural language reasoning, following the configuration for TLM, we use AdamW as the update rule, with $\beta_1=0.9$, $\beta_2=0.999$, weight decay = 0.0, and optimize the prediction entropy of the first 8 output tokens. Both learning rates $\eta_f$ and $\eta_h$ are set to $7.5{\times}10^{-6}$ / $5 {\times} 10^{-6}$ / $1.25 {\times} 10^{-5}$ / $5 {\times} 10^{-6}$ for Math-500 / CollegeMath / AIME24 / Minerva. Trainable parameters are the LoRA parameters with a rank of 8.

% The learning rate for batch size = 1 is scaled down to (0.00025/32) for ResNet50-GN, (0.001/64) for ViT-Base, (0.0001/64) for ViT-Small, (0.000025/64) for ConvNeXt-Tiny and (0.00005/64) for Swin-Tiny. The trainable parameters are all the affine parameters of normalization layers. 

\textbf{Tent\footnote{\url{https://github.com/DequanWang/tent}}~\citep{wang2021tent}.} We follow all hyper-parameters that are set in Tent unless it does not provide. Specifically, for vision adaptation, we use SGD as the update rule, with a momentum of 0.9, batch size of 64 (except for the experiments of batch size = 1), and learning rate of 0.00025 / 0.001 / 0.0001 / 0.000025 / 0.00005 for ResNet50-GN / ViT-Base / ViT-Small / ConvNeXt-Tiny / Swin-Tiny, respectively. The learning rate for batch size = 1 is scaled down to (0.00025/32) for ResNet50-GN, (0.001/64) for ViT-Base, (0.0001/64) for ViT-Small, (0.000025/64) for ConvNeXt-Tiny and (0.00005/64) for Swin-Tiny. The trainable parameters are all the affine parameters of normalization layers. 
For natural language reasoning, we use AdamW as the update rule, with $\beta_1=0.9$, $\beta_2=0.999$, weight decay = 0.0, and optimize the prediction entropy of the first 8 output tokens. The learning rate is $7.5{\times}10^{-6}$ / $5 {\times} 10^{-6}$ / $1.25 {\times} 10^{-5}$ / $5 {\times} 10^{-6}$ for Math-500 / CollegeMath / AIME24 / Minerva. The trainable parameters are the LoRA parameters with a rank of 8.

\textbf{SAR\footnote{\url{https://github.com/mr-eggplant/SAR}}~\citep{niu2023sar}.} We follow all hyper-parameters that are set in SAR unless it does not provide. Specifically, for vision adaptation, we use SGD as the update rule, with a momentum of 0.9, batch size of 64 (except for the experiments of batch size = 1), and learning rate of 0.00025 / 0.001 / 0.0001 / 0.000025 / 0.00005 for ResNet50-GN / ViT-Base / ViT-Small / ConvNeXt-Tiny / Swin-Tiny. The learning rate for batch size = 1 is scaled down to (0.00025/16) for ResNet50-GN, (0.001/32) for ViT-Base, (0.0001/32) for ViT-Small, (0.000025/32) for ConvNeXt-Tiny and (0.00005/32) for Swin-Tiny. The entropy threshold $E_0$ is set to 0.4$\times\ln{C}$, where $C$ is the number of task classes. The trainable parameters are the affine parameters of norm layers from layer 1 to layer 3 in ResNet50-GN, from blocks 1 to blocks 8 in ViT-Base and ViT-Small, and all affine parameters in other models. 
For natural language reasoning, $E_0$ is set to 0.4. We use AdamW as the update rule, with $\beta_1=0.9$, $\beta_2=0.999$, weight decay = 0.0, and optimize the prediction entropy of the first 8 output tokens. The learning rate is $7.5{\times}10^{-6}$ / $5 {\times} 10^{-6}$ / $1.25 {\times} 10^{-5}$ / $5 {\times} 10^{-6}$ for Math-500 / CollegeMath / AIME24 / Minerva, trainable parameters are the LoRA parameters with a rank of 8.
% The trainable parameters are the affine parameters of the layer normalization layers from blocks 1 to blocks 8.

\textbf{EATA\footnote{\url{https://github.com/mr-eggplant/EATA}}~\citep{niu2022eata}.} We follow all hyper-parameters that are set in EATA unless it does not provide. Specifically, for vision adaptation, the entropy constant $E_0$ (for reliable sample identification) is set to 0.4$\times\ln{C}$, where $C$ is the number of task classes. The $\epsilon$ for redundant sample identification is set to 0.05. The trade-off parameter $\beta$ for entropy loss and regularization loss is set to 2,000. The number of pre-collected in-distribution test samples for Fisher importance calculation is 2,000. The update rule is SGD, with a momentum of 0.9, batch size of 64 (except for the experiments of batch size = 1), and learning rate of 0.00025 / 0.001 / 0.0001 / 0.000025 / 0.00005 for ResNet50-GN / ViT-Base / ViT-Small / ConvNeXt-Tiny / Swin-Tiny. The learning rate for batch size = 1 is scaled down to (0.00025/32) for ResNet50-GN, (0.001/64) for ViT-Base, (0.0001/64) for ViT-Small, (0.000025/64) for ConvNeXt-Tiny and (0.00005/64) for Swin-Tiny.  The trainable parameters are all affine parameters of normalization layers. 
For natural language reasoning, $E_0$ is set to 0.4, while the anti-forgetting regularizer is not applied due to a lack of knowledge for in-distribution data that the large language model trains on. We use AdamW as the update rule, with $\beta_1=0.9$, $\beta_2=0.999$, weight decay = 0.0, and optimize the prediction entropy of the first 8 output tokens. The learning rate is $7.5{\times}10^{-6}$ / $5 {\times} 10^{-6}$ / $1.25 {\times} 10^{-5}$ / $5 {\times} 10^{-6}$ for Math-500 / CollegeMath / AIME24 / Minerva. The trainable parameters are the LoRA parameters with a rank of 8.

\textbf{DeYO\footnote{\url{https://github.com/Jhyun17/DeYO}}~\citep{lee2024deyo}.} We follow all hyper-parameters that are set in DeYO unless it does not provide. Specifically, the entropy constant $E_0$ (for reliable sample identification) is set to $0.4\times\ln 1000$, and the factor $\tau_{\text{Ent}}$ is set to 0.5$\times\ln{1000}$. The Pseudo-Label Probability Difference (PLPD) threshold $\tau_{\text{PLPD}}$ is set to 0.2. The update rule is SGD, with a momentum of 0.9, batch size of 64 (except for the
experiments of batch size = 1), and learning rate of 0.00025 / 0.001 / 0.0001 / 0.000025 / 0.00005 for ResNet50-GN / ViT-Base / ViT-Small / ConvNeXt-Tiny / Swin-Tiny. The learning rate for batch size = 1 is scaled down to (0.00025/16) for ResNet50-GN, (0.001/32) for ViT-Base, (0.0001/32) for ViT-Small, (0.000025/32) for ConvNeXt-Tiny and (0.00005/32) for Swin-Tiny. Trainable parameters are the affine parameters of norm layers from layer 1 to 3 in ResNet50-GN, from blocks 1 to 8 in ViT-Base and ViT-Small, and all affine parameters in other models.

\textbf{COME\footnote{\url{https://github.com/BlueWhaleLab/COME}}~\citep{zhang2025come}.} For vision adaptation, we implement COME based on DeYO for comparisons, following the hyper-parameters set in DeYO unless otherwise specified. Specifically, the entropy constant $E_0$ (for reliable sample identification) is set to $0.4\times\ln 1000$, and the factor $\tau_{\text{Ent}}$ is set to 0.5$\times\ln{1000}$. The Pseudo-Label Probability Difference (PLPD) threshold $\tau_{\text{PLPD}}$ is set to 0.2. The uncertainty mass is set to $C$, where $C$ is the number of task classes. The update rule is SGD, with a momentum of 0.9, batch size of 64 (except for the
experiments of batch size = 1), and learning rate of 0.00025 / 0.001 / 0.0001 / 0.000025 / 0.00005 for ResNet50-GN / ViT-Base / ViT-Small / ConvNeXt-Tiny / Swin-Tiny. The learning rate for batch size = 1 is scaled down to (0.00025/16) for ResNet50-GN, (0.001/32) for ViT-Base, (0.0001/32) for ViT-Small, (0.000025/32) for ConvNeXt-Tiny and (0.00005/32) for Swin-Tiny. The trainable parameters are the affine parameters of norm layers from layer 1 to layer 3 in ResNet50-GN, from blocks 1 to blocks 8 in ViT-Base and ViT-Small, and all affine parameters in other models.
For natural language reasoning, COME is implemented based on Tent, and uncertainty mass is set to $|Z|$, where $|Z|$ denotes the vocabulary size. We use AdamW as the update rule, with $\beta_1=0.9$, $\beta_2=0.999$, weight decay = 0.0, and optimize the prediction entropy of the first 8 output tokens. The learning rate is $7.5{\times}10^{-6}$ / $5 {\times} 10^{-6}$ / $1.25 {\times} 10^{-5}$ / $5 {\times} 10^{-6}$ for Math-500 / CollegeMath / AIME24 / Minerva. The trainable parameters are the LoRA parameters with a rank of 8.

\textbf{TLM~\citep{Hu2025TLM}\footnote{\url{https://github.com/Fhujinwu/TLM}}.}
We follow all hyper-parameters that are set in TLM unless it does not provide. Specifically, we compare with TLM on the natural language reasoning task, the perplexity threshold $\mathcal{P_0}$ is set to $e^3$. The update rule is AdamW, with $\beta_1=0.9$, $\beta_2=0.999$, weight decay = 0.0. The learning rate is $7.5{\times}10^{-6}$ / $5 {\times} 10^{-6}$ / $1.25 {\times} 10^{-5}$ / $5 {\times} 10^{-6}$ for Math-500 / CollegeMath / AIME24 / Minerva. The trainable parameters are the LoRA parameters with a rank of 8.

% \blue{
% \subsection{More Experimental Protocols on Evaluations}
% We strictly follow the widely adopted evaluation settings established by prior works such as Tent, SAR, and DeYO for Tables \ref{tab:imagenet-c-mix-shift}–\ref{tab:imagenet-c-bs1} \& \ref{tab:imagenet-c-mild}. In these experiments, all methods are evaluated with only 1 epoch over the test set—each sample is processed exactly \textit{once}. For Tables \ref{tab:imagenet-c-mix-shift}–\ref{tab:imagenet-c-bs1}, we use the Wild protocol—characterized by small batch sizes, online imbalanced label-shifted streams, and mixture of domain shifts—is widely adopted for TTA evaluation because it more closely reflects real-world deployment scenarios. For Table~\ref{tab:imagenet-c-mild}, we use the original Tent setup, which has uniform label space, a batch size of 64, and perform one adaptation epoch over the entire test set.

% }

\clearpage

% Due to the page limitations of the main paper, we only reported the results averaged over each group of corruptions in ImageNet-C (\ie, in Table~\ref{tab:imagenet-c-label-shift}, Table~\ref{tab:imagenet-c-bs1}, Table~\ref{tab:imagenet-c-label-shift+bs1} and Table~\ref{tab:imagenet-c-blindspot}) across multiple models. In this section, we provide more detailed results to enable a more comprehensive comparison.

\section{More experimental results}\label{suppl:sec:more-results}
In this section, we provide the detailed results of Tables~\ref{tab:imagenet-c-label-shift}-\ref{tab:imagenet-c-blindspot} in the main paper and include additional experiments under both mild and wild testing scenarios to enable a comprehensive comparison.

\subsection{Detailed and Additional Results under Wild Test Scenarios}
We first provide the detailed results of TTA under label shifts and a single sample. From Tables~\ref{tab:imagenet-c-label-shift-details} \& \ref{tab:imagenet-c-bs1-details}, \methodname achieves superior performance across nearly all cases, \eg, 51.6\% (Ours) \textit{vs.} 43.9\% (DeYO) on ResNet50-GN under label shifts, and 51.3\% (Ours) \textit{vs.} 42.3\% (COME) on ViT-Small under the batch size of 1. Similar results are observed even when the data stream exhibits both label shifts and data scarcity, as shown in Table~\ref{tab:imagenet-c-label-shift+bs1-details}, suggesting our effectiveness across test settings.

\subsection{Detailed Results for TTA on the Blind-Spot Subset}
% From Table~\ref{tab:imagenet-c-blindspot}, existing methods tend to collapse when adapting on the blind-spot subset, \eg, reducing from 30.6\% to 18.9\% on DeYO \wrt the average accuracy on ResNet50-GN. In contrast, even under such a challenging TTA scenario, \methodname achieves consistent improvement across all domains and all models, suggesting that \methodname substantially expands the scope and reliability of TTA in real-world scenarios.
From Table~\ref{tab:imagenet-c-blindspot-details}, prior methods tend to collapse when adapting on the blind-spot subset, \eg, 18.9\% (DeYO) \textit{vs.} 30.6\% (NoAdapt) \wrt the average accuracy on R50-GN. In contrast, even under such a challenging setting, \methodname achieves consistent improvement across all domains and models, suggesting that \methodname substantially expands the scope and reliability of TTA in the real world.

\subsection{Detailed Results under the Mild Test Scenario}
\methodname is not only effective under challenging test scenarios, but also enhances TTA performance significantly on the mild test scenario~\citep{wang2021tent}, where data comes with shuffled labels. As shown in Table~\ref{tab:imagenet-c-mild-details}, \methodname demonstrates superior performance on 4 out of 5 models and also achieves comparative results on ViT-Base, \eg, the average accuracy of 48.6\% (Ours) \textit{vs.} 41.7\% (DeYO) on ResNet50-GN, and 62.6\% (Ours) \textit{vs.} 62.8\% (DeYO) on ViT-Base. This suggests that our \methodname's design helps improve the stability and efficacy of TTA across a wide range of scenarios. 

% Comparisons with state-of-the-art methods on ImageNet-C (severity level 5) under \textbf{\textsc{online imbalanced label shifts}} (imbalance ratio = $\infty$)  regarding \textbf{Accuracy (\%)}, \ie, the detailed results of Table~\ref{tab:imagenet-c-label-shift} in the main paper.
\begin{table*}[h]
    \caption{Detailed results of TTA under \textbf{\textsc{imbalanced label shifts}}, Table~\ref{tab:imagenet-c-label-shift} in the main paper.}
    \vspace{-0.05in}
    \label{tab:imagenet-c-label-shift-details}
    \begin{center}
    \begin{threeparttable}
    \large
    \resizebox{1.0\linewidth}{!}{
    \begin{tabular}{l|ccc|cccc|cccc|cccc|>{\columncolor{blue!8}}c}
    \multicolumn{1}{c}{} & \multicolumn{3}{c}{Noise} & \multicolumn{4}{c}{Blur} & \multicolumn{4}{c}{Weather} & \multicolumn{4}{c}{Digital} \\
    Model+Method & Gauss. & Shot & Impul. & Defoc. & Glass & Motion & Zoom & Snow & Frost & Fog & Brit. & Contr. & Elastic & Pixel & JPEG & Avg. \\
    \midrule
    ResNet50-GN & 18.0 & 19.8 & 17.9 & 19.8 & 11.4 & 21.4 & 24.9 & 40.4 & 47.3 & 33.6 & 69.3 & 36.3 & 18.6 & 28.4 & 52.3 & 30.6 \\
    ~~$\bullet~$Tent & 2.6 & 3.3 & 2.7 & 13.9 & 7.9 & 19.5 & 17.0 & 16.5 & 21.9 & 1.8 & 70.5 & 42.2 & 6.6 & 49.4 & 53.7 & 22.0 \\
    ~~$\bullet~$SAR & 33.1 & 36.5 & 35.5 & 19.2 & 19.5 & 33.3 & 27.7 & 23.9 & 45.3 & 50.1 & 71.9 & 46.7 & 7.1 & 52.1 & 56.3 & 37.2 \\
    ~~$\bullet~$EATA & 27.0 & 28.3 & 28.1 & 14.9 & 17.1 & 24.4 & 25.3 & 32.2 & 32.0 & 39.8 & 66.7 & 33.6 & 24.5 & 41.9 & 38.4 & 31.6 \\
    ~~$\bullet~$COME & 17.7 & 19.7 & 17.7 & 19.7 & 11.2 & 21.2 & 24.8 & 40.2 & 47.8 & 33.7 & 68.9 & 35.6 & 18.6 & 27.7 & 51.8 & 30.4 \\
    ~~$\bullet~$DeYO & 42.5 & 44.9 & 43.8 & 22.2 & 16.3 & 41.0 & 13.2 & 52.2 & 51.5 & 39.7 & 73.4 & 52.6 & 46.9 & 59.3 & 59.3 & 43.9 \\
    \rowcolor{black!8}~~$\bullet~$ZeroSiam (ours) & 42.9 & 45.1 & 43.1 & 35.1 & 36.1 & 44.5 & 49.2 & 58.1 & 55.1 & 62.9 & 72.3 & 54.8 & 52.5 & 61.7 & 60.1 & \textbf{51.6} \\
    \cmidrule{1-17}
    ViT-Base & 9.4 & 6.7 & 8.3 & 29.1 & 23.4 & 34.0 & 27.0 & 15.8 & 26.3 & 47.4 & 54.7 & 43.9 & 30.5 & 44.5 & 47.6 & 29.9 \\
    ~~$\bullet~$Tent & 32.7 & 1.4 & 34.6 & 54.4 & 52.3 & 58.2 & 52.2 & 7.7 & 12.0 & 69.3 & 76.1 & 66.1 & 56.7 & 69.4 & 66.4 & 47.3 \\
    ~~$\bullet~$SAR & 46.5 & 43.1 & 48.9 & 55.3 & 54.3 & 58.9 & 54.8 & 53.6 & 46.2 & 69.7 & 76.2 & 66.2 & 60.9 & 69.6 & 66.6 & 58.0 \\
    ~~$\bullet~$EATA & 35.9 & 34.6 & 36.7 & 45.3 & 47.2 & 49.3 & 47.7 & 56.5 & 55.4 & 62.2 & 72.2 & 21.7 & 56.2 & 64.7 & 63.7 & 50.0 \\
    ~~$\bullet~$COME & 45.3 & 52.1 & 52.0 & 56.1 & 57.4 & 62.5 & 60.7 & 66.6 & 64.0 & 71.6 & 77.1 & 65.9 & 61.9 & 72.9 & 69.2 & 62.4 \\
    ~~$\bullet~$DeYO & 53.5 & 36.0 & 54.6 & 57.6 & 58.7 & 63.7 & 46.2 & 67.6 & 66.0 & 73.2 & 77.9 & 66.7 & 69.0 & 73.5 & 70.3 & 62.3 \\
    \rowcolor{black!8}~~$\bullet~$ZeroSiam (ours) & 52.3 & 52.6 & 53.4 & 57.7 & 58.7 & 62.7 & 61.1 & 67.6 & 66.0 & 73.3 & 78.0 & 67.0 & 67.9 & 73.5 & 70.1 & \textbf{64.1} \\
    \cmidrule{1-17}
    ViT-Small & 2.0 & 2.0 & 1.5 & 24.2 & 17.1 & 30.3 & 22.0 & 9.5 & 19.2 & 37.9 & 44.4 & 30.1 & 24.8 & 38.2 & 41.0 & 22.9 \\
    ~~$\bullet~$Tent & 0.3 & 0.4 & 0.2 & 43.1 & 37.0 & 45.4 & 36.3 & 5.5 & 24.2 & 58.4 & 65.8 & 54.6 & 26.6 & 56.6 & 54.9 & 34.0 \\
    ~~$\bullet~$SAR & 1.6 & 2.0 & 1.3 & 44.3 & 39.0 & 46.5 & 39.0 & 16.6 & 45.9 & 58.6 & 65.8 & 54.7 & 44.4 & 56.9 & 55.2 & 38.1 \\
    ~~$\bullet~$EATA & 15.6 & 15.5 & 17.8 & 42.4 & 40.3 & 44.8 & 41.6 & 46.2 & 48.7 & 58.9 & 65.3 & 52.9 & 50.5 & 57.2 & 56.2 & 43.6 \\
    ~~$\bullet~$COME & 0.1 & 0.2 & 0.1 & 47.4 & 47.2 & 52.8 & 10.5 & 35.5 & 53.8 & 63.9 & 70.7 & 58.1 & 56.7 & 63.4 & 60.2 & 41.4 \\
    ~~$\bullet~$DeYO & 0.1 & 0.2 & 0.1 & 48.4 & 47.7 & 53.4 & 16.2 & 47.0 & 54.6 & 65.0 & 71.0 & 59.3 & 58.5 & 64.1 & 61.0 & 43.1 \\
    \rowcolor{black!8}~~$\bullet~$ZeroSiam (ours) & 25.7 & 25.9 & 27.9 & 48.2 & 48.3 & 53.5 & 51.5 & 55.7 & 55.6 & 65.4 & 71.2 & 58.4 & 59.5 & 64.2 & 61.2 & \textbf{51.5} \\
    \cmidrule{1-17}
    ConvNeXt-Tiny & 21.4 & 27.7 & 22.1 & 24.5 & 11.0 & 32.5 & 31.2 & 44.5 & 52.5 & 39.5 & 72.0 & 44.8 & 23.1 & 17.7 & 55.8 & 34.7 \\
    ~~$\bullet~$Tent & 18.1 & 23.8 & 26.6 & 24.3 & 4.2 & 33.9 & 29.2 & 41.6 & 46.3 & 39.7 & 73.2 & 51.6 & 18.4 & 10.7 & 56.9 & 33.2 \\
    ~~$\bullet~$SAR & 34.8 & 36.7 & 34.9 & 27.1 & 4.1 & 35.0 & 30.8 & 44.3 & 47.6 & 8.4 & 73.1 & 51.1 & 19.4 & 24.7 & 56.7 & 35.2 \\
    ~~$\bullet~$EATA & 34.4 & 36.8 & 35.0 & 26.7 & 19.1 & 37.0 & 33.3 & 47.7 & 48.4 & 50.6 & 73.0 & 51.9 & 28.7 & 31.4 & 56.9 & 40.7 \\
    ~~$\bullet~$COME & 42.6 & 44.5 & 42.4 & 38.7 & 1.2 & 46.0 & 11.8 & 57.1 & 56.9 & 61.8 & 75.6 & 60.6 & 5.3 & 51.9 & 61.9 & 43.9 \\
    ~~$\bullet~$DeYO & 33.2 & 38.0 & 36.7 & 27.9 & 2.6 & 36.4 & 9.2 & 52.7 & 45.3 & 2.9 & 74.4 & 56.6 & 4.7 & 19.5 & 58.4 & 33.2 \\
    \rowcolor{black!8}~~$\bullet~$ZeroSiam (ours) & 42.3 & 44.5 & 42.1 & 37.4 & 33.5 & 45.3 & 44.0 & 56.5 & 55.4 & 62.9 & 74.7 & 59.4 & 44.3 & 51.7 & 60.0 & \textbf{50.3} \\
    \cmidrule{1-17}
    Swin-Tiny & 26.6 & 27.9 & 22.8 & 21.7 & 13.2 & 26.5 & 25.4 & 41.0 & 45.8 & 47.7 & 67.6 & 39.1 & 22.6 & 8.4 & 33.3 & 31.3 \\
    ~~$\bullet~$Tent & 16.1 & 14.1 & 10.5 & 17.1 & 16.9 & 23.5 & 20.0 & 30.4 & 28.3 & 11.3 & 69.7 & 47.8 & 9.3 & 1.1 & 43.2 & 24.0 \\
    ~~$\bullet~$SAR & 29.8 & 28.0 & 29.8 & 19.8 & 10.6 & 24.4 & 21.2 & 33.1 & 34.3 & 22.3 & 67.9 & 45.9 & 13.2 & 4.8 & 42.9 & 28.5 \\
    ~~$\bullet~$EATA & 36.2 & 37.1 & 33.6 & 26.3 & 29.8 & 38.2 & 37.3 & 47.3 & 44.5 & 51.4 & 70.3 & 42.4 & 40.9 & 34.8 & 46.5 & 41.1 \\
    ~~$\bullet~$COME & 41.8 & 43.6 & 43.2 & 36.5 & 9.8 & 46.7 & 2.4 & 54.6 & 52.3 & 60.0 & 72.0 & 47.4 & 4.3 & 3.1 & 55.8 & 38.2 \\
    % ~~$\bullet~$DeYO & 40.1 & 43.3 & 42.4 & 20.9 & 32.0 & 42.2 & 12.3 & 49.4 & 49.0 & 55.7 & 72.3 & 53.3 & 38.5 & 42.1 & 53.3 & 43.1 \\
    ~~$\bullet~$DeYO & 35.3 & 38.5 & 35.6 & 18.9 & 24.6 & 37.8 & 9.4 & 43.7 & 41.8 & 52.3 & 70.6 & 50.7 & 21.5 & 15.6 & 50.8 & 36.5 \\
    \rowcolor{black!8}~~$\bullet~$ZeroSiam (ours) & 44.4 & 45.8 & 45.7 & 36.4 & 38.3 & 47.0 & 47.0 & 54.3 & 51.9 & 58.7 & 72.4 & 54.2 & 52.7 & 51.8 & 55.5 & \textbf{50.4} \\
    \end{tabular}
    }
    \end{threeparttable}
    \end{center}
\end{table*}

% \cdy{
% \subsection{Detailed Results on ImageNet-C with Batch Size=1}
% }

% Comparisons with state-of-the-art methods on ImageNet-C (severity level 5) with \textbf{\textsc{batch size=1}} regarding \textbf{Accuracy (\%)}, \ie, the detailed results of Table~\ref{tab:imagenet-c-bs1} in the main paper.
\begin{table*}[t]
    \caption{Detailed results of TTA with \textbf{\textsc{batch size=1}}, \ie, Table~\ref{tab:imagenet-c-bs1} in the main paper.}
    \vspace{-0.1in}
    \label{tab:imagenet-c-bs1-details}
    \begin{center}
    \begin{threeparttable}
    \large
    \resizebox{1.0\linewidth}{!}{
    \begin{tabular}{l|ccc|cccc|cccc|cccc|>{\columncolor{blue!8}}c}
    \multicolumn{1}{c}{} & \multicolumn{3}{c}{Noise} & \multicolumn{4}{c}{Blur} & \multicolumn{4}{c}{Weather} & \multicolumn{4}{c}{Digital} \\
    Model+Method & Gauss. & Shot & Impul. & Defoc. & Glass & Motion & Zoom & Snow & Frost & Fog & Brit. & Contr. & Elastic & Pixel & JPEG & Avg. \\
    \midrule
    ResNet50-GN & 18.0 & 19.8 & 17.9 & 19.8 & 11.4 & 21.4 & 24.9 & 40.4 & 47.3 & 33.6 & 69.3 & 36.3 & 18.6 & 28.4 & 52.3 & 30.6 \\
    ~~$\bullet~$Tent & 2.5 & 2.9 & 2.5 & 13.5 & 3.6 & 18.6 & 17.6 & 15.3 & 23.0 & 1.4 & 70.4 & 42.2 & 6.2 & 49.2 & 53.8 & 21.5 \\
    ~~$\bullet~$SAR & 23.4 & 26.6 & 23.0 & 18.4 & 15.4 & 28.6 & 30.4 & 44.9 & 44.7 & 25.7 & 72.3 & 44.5 & 14.8 & 47.0 & 56.1 & 34.5 \\
    ~~$\bullet~$EATA & 24.8 & 28.3 & 25.7 & 18.1 & 17.3 & 28.5 & 29.3 & 44.5 & 44.3 & 41.6 & 70.9 & 44.6 & 27.0 & 46.8 & 55.7 & 36.5 \\
    ~~$\bullet~$COME & 17.8 & 19.7 & 17.7 & 19.7 & 11.3 & 21.4 & 24.9 & 40.3 & 47.6 & 33.6 & 68.9 & 35.7 & 18.5 & 27.9 & 51.8 & 30.5 \\
    ~~$\bullet~$DeYO & 41.8 & 44.7 & 43.0 & 22.5 & 24.7 & 41.8 & 24.4 & 54.5 & 52.2 & 20.7 & 73.5 & 53.5 & 48.5 & 60.2 & 59.8 & 44.4 \\
    \rowcolor{black!8}~~$\bullet~$ZeroSiam (ours) & 41.7 & 44.4 & 41.8 & 32.8 & 35.8 & 45.2 & 49.5 & 58.4 & 55.5 & 63.8 & 73.7 & 54.8 & 55.1 & 63.6 & 61.0 & \textbf{51.8} \\
    \cmidrule{1-17}
    ViT-Base & 9.5 & 6.7 & 8.2 & 29.0 & 23.4 & 33.9 & 27.1 & 15.9 & 26.5 & 47.2 & 54.7 & 44.1 & 30.5 & 44.5 & 47.8 & 29.9 \\
    ~~$\bullet~$Tent & 42.2 & 1.0 & 43.3 & 52.4 & 48.2 & 55.5 & 50.5 & 16.5 & 16.9 & 66.4 & 74.9 & 64.7 & 51.6 & 67.0 & 64.3 & 47.7 \\
    ~~$\bullet~$SAR & 40.8 & 36.4 & 41.5 & 53.7 & 50.7 & 57.5 & 52.8 & 59.1 & 50.7 & 68.1 & 74.6 & 65.7 & 57.9 & 68.9 & 65.9 & 56.3 \\
    ~~$\bullet~$EATA & 29.7 & 25.1 & 34.6 & 44.7 & 39.2 & 48.3 & 42.4 & 37.5 & 45.9 & 60.0 & 65.9 & 61.2 & 46.4 & 58.2 & 59.6 & 46.6 \\
    ~~$\bullet~$COME & 51.7 & 51.4 & 52.1 & 57.6 & 58.2 & 63.3 & 41.2 & 67.1 & 64.8 & 72.8 & 77.7 & 68.1 & 67.5 & 73.0 & 69.9 & 62.4 \\
    ~~$\bullet~$DeYO & 54.0 & 52.1 & 55.1 & 58.8 & 59.5 & 64.2 & 53.5 & 68.2 & 66.4 & 73.7 & 78.3 & 68.2 & 68.9 & 73.8 & 70.8 & \textbf{64.4} \\
    \rowcolor{black!8}~~$\bullet~$ZeroSiam (ours) & 52.1 & 52.7 & 52.8 & 57.8 & 58.3 & 63.0 & 60.6 & 67.0 & 65.7 & 73.0 & 77.9 & 67.6 & 67.9 & 72.8 & 69.9 & 63.9 \\
    \cmidrule{1-17}
    ViT-Small & 2.0 & 1.9 & 1.5 & 24.3 & 17.1 & 30.2 & 21.8 & 9.5 & 19.3 & 37.9 & 44.4 & 29.9 & 24.9 & 38.2 & 41.3 & 22.9 \\
    ~~$\bullet~$Tent & 21.3 & 27.7 & 22.0 & 24.5 & 10.6 & 32.3 & 30.9 & 44.6 & 52.4 & 39.6 & 72.3 & 44.5 & 23.4 & 17.6 & 55.6 & 34.6 \\
    ~~$\bullet~$SAR & 26.9 & 27.7 & 23.0 & 21.7 & 13.2 & 26.4 & 25.3 & 41.2 & 45.8 & 47.7 & 67.8 & 39.2 & 22.5 & 8.4 & 33.2 & 31.3 \\
    ~~$\bullet~$EATA & 2.4 & 2.5 & 1.8 & 32.8 & 25.4 & 36.3 & 28.2 & 17.1 & 28.2 & 48.7 & 53.5 & 46.1 & 34.1 & 46.7 & 47.7 & 30.1 \\
    ~~$\bullet~$COME & 0.5 & 3.5 & 0.4 & 48.4 & 47.3 & 52.9 & 15.4 & 38.2 & 53.8 & 64.3 & 70.7 & 58.8 & 55.8 & 63.6 & 60.4 & 42.3 \\
    ~~$\bullet~$DeYO & 0.4 & 0.8 & 0.3 & 49.0 & 47.4 & 53.0 & 20.1 & 46.4 & 54.6 & 64.7 & 70.9 & 59.3 & 57.3 & 64.1 & 61.0 & 43.3 \\
    \rowcolor{black!8}~~$\bullet~$ZeroSiam (ours) & 26.0 & 26.9 & 28.3 & 48.6 & 48.2 & 52.7 & 50.7 & 54.4 & 55.1 & 65.0 & 71.0 & 59.1 & 58.6 & 63.4 & 60.8 & \textbf{51.3} \\
    \cmidrule{1-17}
    ConvNeXt-Tiny & 21.3 & 27.7 & 22.0 & 24.5 & 10.6 & 32.3 & 30.9 & 44.6 & 52.4 & 39.6 & 72.3 & 44.5 & 23.4 & 17.6 & 55.6 & 34.6 \\
    ~~$\bullet~$Tent & 30.3 & 33.5 & 32.0 & 24.7 & 7.3 & 33.6 & 30.0 & 43.7 & 47.3 & 43.7 & 72.9 & 49.7 & 21.4 & 18.5 & 56.3 & 36.3 \\
    ~~$\bullet~$SAR & 30.3 & 33.6 & 29.8 & 25.9 & 11.1 & 33.4 & 30.8 & 44.9 & 47.8 & 38.3 & 72.9 & 46.2 & 22.6 & 18.1 & 56.2 & 36.1 \\
    ~~$\bullet~$EATA & 26.6 & 30.5 & 26.9 & 25.2 & 11.4 & 33.1 & 31.0 & 44.8 & 50.1 & 40.0 & 72.2 & 47.8 & 23.5 & 18.5 & 55.8 & 35.8 \\
    ~~$\bullet~$COME & 41.1 & 43.2 & 41.1 & 37.3 & 2.1 & 44.5 & 17.9 & 56.4 & 55.8 & 60.7 & 75.2 & 59.6 & 8.6 & 44.7 & 61.1 & 43.3 \\
    ~~$\bullet~$DeYO & 36.8 & 39.4 & 37.0 & 29.5 & 2.9 & 38.0 & 14.6 & 51.6 & 46.9 & 4.5 & 74.0 & 55.9 & 7.8 & 19.5 & 57.9 & 34.4 \\
    \rowcolor{black!8}~~$\bullet~$ZeroSiam (ours) & 41.5 & 43.6 & 41.5 & 37.2 & 32.5 & 44.8 & 42.7 & 55.6 & 54.7 & 62.2 & 74.3 & 59.0 & 42.3 & 50.5 & 59.3 & \textbf{49.4} \\
    \cmidrule{1-17}
    Swin-Tiny & 26.9 & 27.7 & 23.0 & 21.7 & 13.2 & 26.4 & 25.3 & 41.2 & 45.8 & 47.7 & 67.8 & 39.2 & 22.5 & 8.4 & 33.2 & 31.3 \\
    ~~$\bullet~$Tent & 23.8 & 26.2 & 22.5 & 19.4 & 18.4 & 26.9 & 27.3 & 40.4 & 36.0 & 22.8 & 69.2 & 46.5 & 16.5 & 2.4 & 41.4 & 29.3 \\
    ~~$\bullet~$SAR & 25.4 & 27.1 & 23.0 & 21.6 & 14.8 & 27.3 & 27.5 & 38.9 & 37.5 & 44.1 & 67.7 & 38.4 & 18.9 & 8.3 & 39.5 & 30.7 \\
    ~~$\bullet~$EATA & 32.8 & 34.5 & 31.7 & 23.4 & 17.5 & 32.0 & 28.2 & 43.5 & 43.6 & 45.4 & 69.7 & 45.8 & 27.2 & 10.5 & 39.6 & 35.0 \\
    ~~$\bullet~$COME & 38.6 & 41.4 & 39.0 & 34.4 & 14.9 & 44.6 & 7.3 & 52.9 & 51.0 & 56.9 & 72.4 & 51.7 & 15.7 & 33.8 & 53.3 & 40.5 \\
    ~~$\bullet~$DeYO & 35.5 & 38.5 & 35.3 & 22.0 & 23.9 & 37.1 & 15.3 & 42.4 & 40.4 & 47.6 & 70.2 & 50.8 & 23.7 & 18.0 & 48.8 & 36.6 \\
    \rowcolor{black!8}~~$\bullet~$ZeroSiam (ours) & 44.0 & 45.4 & 45.2 & 34.7 & 39.0 & 46.3 & 46.0 & 53.9 & 51.3 & 57.8 & 72.1 & 55.0 & 51.7 & 50.3 & 54.6 & \textbf{49.8} \\
    \end{tabular}
    }
    \end{threeparttable}
    \end{center}
\end{table*}
\begin{table*}[t]
    \caption{Additional results of TTA under \textbf{\textsc{label shifts}} with \textbf{\textsc{batch size=1}} \wrt \textbf{Accuracy(\%)}.}
    \vspace{-0.1in}
    \label{tab:imagenet-c-label-shift+bs1-details}
    \begin{center}
    \begin{threeparttable}
    \large
    \resizebox{1.0\linewidth}{!}{
    \begin{tabular}{l|ccc|cccc|cccc|cccc|>{\columncolor{blue!8}}c}
    \multicolumn{1}{c}{} & \multicolumn{3}{c}{Noise} & \multicolumn{4}{c}{Blur} & \multicolumn{4}{c}{Weather} & \multicolumn{4}{c}{Digital} \\
    Model+Method & Gauss. & Shot & Impul. & Defoc. & Glass & Motion & Zoom & Snow & Frost & Fog & Brit. & Contr. & Elastic & Pixel & JPEG & Avg. \\
    \midrule
    ResNet50-GN & 18.0 & 19.8 & 17.9 & 19.8 & 11.4 & 21.4 & 24.9 & 40.4 & 47.3 & 33.6 & 69.3 & 36.3 & 18.6 & 28.4 & 52.3 & 30.6 \\
    ~~$\bullet~$Tent & 1.2 & 1.5 & 1.3 & 10.0 & 2.1 & 14.3 & 11.2 & 8.2 & 11.7 & 0.8 & 70.1 & 41.6 & 3.4 & 49.5 & 52.3 & 18.6 \\
    ~~$\bullet~$SAR & 23.4 & 26.5 & 23.8 & 18.3 & 15.4 & 28.4 & 29.5 & 44.3 & 44.5 & 31.6 & 72.3 & 44.5 & 14.9 & 46.8 & 56.1 & 34.7 \\
    ~~$\bullet~$EATA & 19.2 & 21.7 & 19.4 & 17.9 & 13.0 & 23.5 & 25.7 & 39.7 & 43.6 & 34.5 & 69.4 & 38.4 & 20.0 & 34.9 & 53.2 & 31.6 \\
    ~~$\bullet~$COME & 17.8 & 19.7 & 17.8 & 19.7 & 11.2 & 21.3 & 24.9 & 40.3 & 47.7 & 33.6 & 69.0 & 35.7 & 18.6 & 27.9 & 51.7 & 30.5 \\
    ~~$\bullet~$DeYO & 41.0 & 43.8 & 42.2 & 22.9 & 23.3 & 41.4 & 15.9 & 54.0 & 52.3 & 20.7 & 73.4 & 53.6 & 48.1 & 60.1 & 59.8 & 43.5 \\
    \rowcolor{black!8}~~$\bullet~$ZeroSiam (ours) & 40.3 & 42.4 & 40.5 & 32.6 & 33.2 & 41.2 & 44.3 & 53.6 & 53.1 & 59.5 & 72.7 & 52.1 & 46.8 & 58.7 & 58.7 & \textbf{48.6} \\
    \cmidrule{1-17}
    ViT-Base & 9.4 & 6.7 & 8.3 & 29.1 & 23.4 & 34.0 & 27.0 & 15.8 & 26.3 & 47.4 & 54.7 & 43.9 & 30.5 & 44.5 & 47.6 & 29.9 \\
    ~~$\bullet~$Tent & 29.2 & 7.9 & 21.1 & 49.3 & 46.9 & 53.7 & 47.5 & 18.0 & 19.2 & 63.7 & 71.1 & 61.4 & 51.8 & 63.6 & 62.2 & 44.4 \\
    ~~$\bullet~$SAR & 24.8 & 18.3 & 24.1 & 38.2 & 31.9 & 42.1 & 35.5 & 31.1 & 37.5 & 52.3 & 61.0 & 52.3 & 36.3 & 52.0 & 52.4 & 39.3 \\
    ~~$\bullet~$EATA & 30.7 & 26.8 & 31.3 & 42.9 & 39.5 & 47.8 & 35.6 & 38.0 & 43.7 & 60.2 & 65.8 & 57.7 & 46.9 & 59.4 & 58.0 & 45.6 \\
    ~~$\bullet~$COME & 51.1 & 47.9 & 52.5 & 57.2 & 58.0 & 63.1 & 60.5 & 67.3 & 64.8 & 72.8 & 77.6 & 67.6 & 67.4 & 73.4 & 69.8 & \textbf{63.4} \\
    ~~$\bullet~$DeYO & 53.2 & 36.2 & 54.5 & 58.2 & 59.3 & 64.3 & 41.8 & 68.6 & 66.7 & 73.8 & 78.3 & 67.4 & 69.3 & 74.1 & 71.0 & 62.4 \\
    \rowcolor{black!8}~~$\bullet~$ZeroSiam (ours) & 49.9 & 50.2 & 51.0 & 56.0 & 56.4 & 60.8 & 58.4 & 65.5 & 64.4 & 71.9 & 77.4 & 66.1 & 65.9 & 71.5 & 68.7 & 62.3 \\
    \cmidrule{1-17}
    ViT-Small & 2.0 & 2.0 & 1.5 & 24.2 & 17.1 & 30.3 & 22.0 & 9.5 & 19.2 & 37.9 & 44.4 & 30.1 & 24.8 & 38.2 & 41.0 & 22.9 \\
    ~~$\bullet~$Tent & 0.3 & 0.3 & 0.2 & 43.4 & 37.2 & 45.7 & 36.6 & 4.6 & 22.3 & 58.4 & 66.0 & 54.7 & 27.0 & 56.8 & 55.3 & 33.9 \\
    ~~$\bullet~$SAR & 2.2 & 2.3 & 1.7 & 34.5 & 26.0 & 39.2 & 27.4 & 19.7 & 32.1 & 45.0 & 53.7 & 43.4 & 31.9 & 44.2 & 45.9 & 29.9 \\
    ~~$\bullet~$EATA & 2.1 & 2.3 & 1.7 & 29.1 & 21.8 & 34.6 & 25.0 & 14.7 & 24.8 & 42.8 & 49.5 & 39.4 & 28.7 & 42.7 & 44.2 & 26.9 \\
    ~~$\bullet~$COME & 0.5 & 3.5 & 0.4 & 47.8 & 46.9 & 52.7 & 16.9 & 37.3 & 53.6 & 64.3 & 70.6 & 58.8 & 55.8 & 63.5 & 60.4 & 42.2 \\
    ~~$\bullet~$DeYO & 0.4 & 0.8 & 0.3 & 48.7 & 47.3 & 53.2 & 21.2 & 46.6 & 54.2 & 64.9 & 70.9 & 59.4 & 57.5 & 64.1 & 61.0 & 43.4 \\
    \rowcolor{black!8}~~$\bullet~$ZeroSiam (ours) & 23.7 & 23.6 & 25.0 & 46.0 & 44.9 & 49.6 & 46.8 & 50.8 & 52.8 & 62.7 & 69.1 & 56.8 & 54.8 & 61.1 & 58.7 & \textbf{48.4} \\
    \cmidrule{1-17}
    ConvNeXt-Tiny & 21.4 & 27.7 & 22.1 & 24.5 & 11.0 & 32.5 & 31.2 & 44.5 & 52.5 & 39.5 & 72.0 & 44.8 & 23.1 & 17.7 & 55.8 & 34.7 \\
    ~~$\bullet~$Tent & 18.2 & 24.8 & 26.0 & 24.4 & 4.3 & 34.0 & 29.1 & 41.8 & 46.3 & 42.4 & 73.2 & 51.8 & 18.2 & 10.4 & 56.9 & 33.5 \\
    ~~$\bullet~$SAR & 30.2 & 33.6 & 30.2 & 25.8 & 11.1 & 33.3 & 31.0 & 44.9 & 47.6 & 37.8 & 73.0 & 46.2 & 22.7 & 18.2 & 56.3 & 36.1 \\
    ~~$\bullet~$EATA & 24.0 & 28.9 & 24.5 & 24.7 & 11.2 & 32.5 & 31.0 & 44.6 & 51.4 & 39.7 & 72.1 & 45.9 & 23.2 & 18.0 & 55.8 & 35.2 \\
    ~~$\bullet~$COME & 41.4 & 43.4 & 41.3 & 37.3 & 2.4 & 44.7 & 18.4 & 56.4 & 55.7 & 61.1 & 75.3 & 59.6 & 11.0 & 44.3 & 61.2 & 43.6 \\
    ~~$\bullet~$DeYO & 36.9 & 39.4 & 37.0 & 29.2 & 3.2 & 37.7 & 14.8 & 51.6 & 47.6 & 5.4 & 74.0 & 55.8 & 8.3 & 23.3 & 58.0 & 34.8 \\
    \rowcolor{black!8}~~$\bullet~$ZeroSiam (ours) & 41.3 & 43.3 & 41.1 & 36.3 & 31.4 & 43.8 & 42.2 & 54.9 & 53.7 & 61.5 & 74.2 & 58.2 & 41.0 & 49.8 & 59.2 & \textbf{48.8} \\
    \cmidrule{1-17}
    Swin-Tiny & 26.6 & 27.9 & 22.8 & 21.7 & 13.2 & 26.5 & 25.4 & 41.0 & 45.8 & 47.7 & 67.6 & 39.1 & 22.6 & 8.4 & 33.3 & 31.3 \\
    ~~$\bullet~$Tent & 14.6 & 13.1 & 11.1 & 17.4 & 14.3 & 21.2 & 20.9 & 31.4 & 24.2 & 12.1 & 69.7 & 48.1 & 9.1 & 1.0 & 42.7 & 23.4 \\
    ~~$\bullet~$SAR & 30.5 & 32.2 & 29.2 & 23.0 & 17.1 & 29.9 & 27.9 & 41.9 & 39.9 & 47.5 & 68.6 & 40.4 & 21.5 & 8.9 & 40.4 & 33.3 \\
    ~~$\bullet~$EATA & 27.7 & 29.1 & 25.0 & 21.7 & 14.2 & 27.7 & 25.9 & 40.7 & 43.5 & 47.4 & 68.1 & 39.9 & 23.4 & 9.1 & 34.3 & 31.9 \\
    ~~$\bullet~$COME & 38.6 & 41.0 & 39.2 & 34.3 & 21.6 & 44.6 & 5.0 & 53.1 & 51.1 & 57.2 & 72.2 & 51.5 & 21.3 & 18.3 & 53.4 & 40.2 \\
    ~~$\bullet~$DeYO & 35.7 & 37.6 & 35.8 & 21.0 & 26.0 & 36.8 & 14.7 & 42.3 & 42.0 & 49.6 & 70.4 & 50.9 & 24.4 & 26.2 & 48.8 & 37.5 \\
    \rowcolor{black!8}~~$\bullet~$ZeroSiam (ours) & 42.8 & 44.4 & 43.7 & 33.3 & 36.1 & 44.2 & 43.9 & 52.1 & 49.6 & 55.9 & 71.6 & 53.5 & 49.3 & 47.7 & 53.4 & \textbf{48.1} \\
    \end{tabular}
    }
    \end{threeparttable}
    \end{center}
\end{table*}
\begin{table*}[t]
    \caption{Detailed results of TTA on the \textbf{\textsc{blind-spot subset}}, \ie, Table~\ref{tab:imagenet-c-blindspot} in the main paper.}
    \vspace{-0.1in}
    \label{tab:imagenet-c-blindspot-details}
    \begin{center}
    \begin{threeparttable}
    \large
    \resizebox{1.0\linewidth}{!}{
    \begin{tabular}{l|ccc|cccc|cccc|cccc|>{\columncolor{blue!8}}c}
    \multicolumn{1}{c}{} & \multicolumn{3}{c}{Noise} & \multicolumn{4}{c}{Blur} & \multicolumn{4}{c}{Weather} & \multicolumn{4}{c}{Digital} \\
    Model+Method & Gauss. & Shot & Impul. & Defoc. & Glass & Motion & Zoom & Snow & Frost & Fog & Brit. & Contr. & Elastic & Pixel & JPEG & Avg. \\
    \midrule
    ResNet50-GN & 18.0 & 19.8 & 17.9 & 19.8 & 11.4 & 21.4 & 24.9 & 40.4 & 47.3 & 33.6 & 69.3 & 36.3 & 18.6 & 28.4 & 52.3 & 30.6 \\
    ~~$\bullet~$Tent & 0.2 & 0.2 & 0.2 & 10.5 & 2.3 & 5.8 & 2.3 & 1.6 & 2.3 & 0.2 & 63.3 & 16.7 & 0.5 & 48.3 & 46.8 & 13.4 \\
    ~~$\bullet~$SAR & 17.0 & 19.6 & 16.9 & 15.8 & 11.4 & 22.2 & 24.0 & 5.8 & 42.0 & 25.4 & 69.0 & 38.9 & 1.3 & 31.9 & 54.2 & 26.4 \\
    ~~$\bullet~$EATA & 18.6 & 21.2 & 18.6 & 15.6 & 13.4 & 22.7 & 24.4 & 37.1 & 38.8 & 33.2 & 67.8 & 41.6 & 19.1 & 47.1 & 53.0 & 31.5 \\
    ~~$\bullet~$COME & 18.0 & 19.8 & 17.9 & 19.8 & 11.3 & 21.4 & 24.9 & 40.4 & 47.3 & 33.6 & 69.2 & 36.2 & 18.6 & 28.3 & 52.2 & 30.6 \\
    ~~$\bullet~$DeYO & 0.4 & 0.6 & 0.4 & 7.9 & 0.2 & 28.1 & 1.6 & 5.5 & 7.1 & 0.1 & 71.2 & 47.8 & 1.3 & 59.7 & 51.4 & 18.9 \\
    \rowcolor{black!8}~~$\bullet~$ZeroSiam (ours) & 38.0 & 42.7 & 39.0 & 27.9 & 34.0 & 41.1 & 43.3 & 52.1 & 44.3 & 59.1 & 60.2 & 51.2 & 53.8 & 60.3 & 45.4 & \textbf{46.2} \\
    \cmidrule{1-17}
    ViT-Base & 9.5 & 6.7 & 8.2 & 29.0 & 23.4 & 33.9 & 27.1 & 15.9 & 26.5 & 47.2 & 54.7 & 44.1 & 30.5 & 44.5 & 47.8 & 29.9 \\
    ~~$\bullet~$Tent & 0.2 & 0.2 & 0.1 & 56.5 & 54.0 & 60.4 & 0.5 & 1.4 & 1.0 & 0.2 & 77.5 & 67.4 & 0.3 & 71.6 & 68.7 & 30.7 \\
    ~~$\bullet~$SAR & 42.2 & 7.1 & 43.9 & 55.3 & 51.3 & 59.4 & 54.0 & 19.7 & 47.6 & 71.5 & 77.4 & 67.1 & 19.4 & 71.3 & 67.8 & 50.3 \\
    ~~$\bullet~$EATA & 26.4 & 20.3 & 32.2 & 45.7 & 37.5 & 48.6 & 42.3 & 37.2 & 47.5 & 62.2 & 65.7 & 63.1 & 46.8 & 60.1 & 59.8 & 46.4 \\
    ~~$\bullet~$COME & 0.1 & 0.1 & 0.1 & 56.9 & 58.4 & 63.5 & 0.5 & 68.3 & 64.7 & 73.4 & 77.1 & 67.1 & 69.1 & 73.4 & 69.5 & 49.5 \\
    ~~$\bullet~$DeYO & 0.1 & 0.3 & 0.2 & 59.4 & 60.4 & 65.1 & 3.8 & 69.7 & 67.6 & 74.5 & 78.5 & 68.7 & 70.8 & 75.0 & 71.2 & 51.0 \\
    \rowcolor{black!8}~~$\bullet~$ZeroSiam (ours) & 52.6 & 53.2 & 53.4 & 57.7 & 58.6 & 63.4 & 61.7 & 67.8 & 66.5 & 73.8 & 78.2 & 67.8 & 69.2 & 73.7 & 70.2 & \textbf{64.5} \\
    \cmidrule{1-17}
    ViT-Small & 2.0 & 1.9 & 1.5 & 24.3 & 17.1 & 30.2 & 21.8 & 9.5 & 19.3 & 37.9 & 44.4 & 29.9 & 24.9 & 38.2 & 41.3 & 22.9 \\
    ~~$\bullet~$Tent & 0.1 & 0.2 & 0.1 & 44.6 & 36.6 & 45.6 & 33.2 & 1.3 & 16.5 & 59.5 & 68.0 & 56.5 & 0.5 & 58.2 & 56.2 & 31.8 \\
    ~~$\bullet~$SAR & 2.1 & 2.2 & 1.6 & 31.1 & 23.7 & 39.2 & 27.5 & 9.8 & 40.0 & 51.8 & 66.4 & 46.7 & 23.4 & 53.2 & 49.8 & 31.2 \\
    ~~$\bullet~$EATA & 2.2 & 2.2 & 1.6 & 28.4 & 20.3 & 33.2 & 24.2 & 11.8 & 23.0 & 44.8 & 51.4 & 47.4 & 28.3 & 45.4 & 45.6 & 27.3 \\
    ~~$\bullet~$COME & 0.1 & 2.6 & 0.1 & 47.3 & 48.0 & 52.4 & 0.2 & 31.2 & 55.6 & 65.0 & 71.7 & 59.5 & 0.4 & 64.4 & 60.6 & 37.3 \\
    ~~$\bullet~$DeYO & 0.1 & 0.3 & 0.1 & 49.4 & 48.9 & 54.0 & 0.5 & 45.9 & 56.3 & 66.2 & 72.3 & 60.6 & 58.6 & 65.5 & 61.9 & 42.7 \\
    \rowcolor{black!8}~~$\bullet~$ZeroSiam (ours) & 25.1 & 25.4 & 27.6 & 48.2 & 47.9 & 52.5 & 50.6 & 54.3 & 55.5 & 65.6 & 71.8 & 59.2 & 59.4 & 64.6 & 61.2 & \textbf{51.3} \\
    \cmidrule{1-17}
    ConvNeXt-Tiny & 21.3 & 27.7 & 22.0 & 24.5 & 10.6 & 32.3 & 30.9 & 44.6 & 52.4 & 39.6 & 72.3 & 44.5 & 23.4 & 17.6 & 55.6 & 34.6 \\
    ~~$\bullet~$Tent & 15.5 & 12.9 & 25.2 & 21.3 & 2.1 & 30.0 & 25.3 & 30.3 & 40.4 & 1.8 & 72.6 & 50.1 & 14.3 & 3.2 & 56.4 & 26.8 \\
    ~~$\bullet~$SAR & 24.3 & 30.2 & 24.8 & 24.7 & 10.7 & 32.5 & 30.6 & 42.9 & 50.1 & 37.1 & 72.7 & 45.1 & 21.8 & 17.7 & 55.7 & 34.7 \\
    ~~$\bullet~$EATA & 23.7 & 29.1 & 24.5 & 24.6 & 11.1 & 32.5 & 30.7 & 44.3 & 50.4 & 39.9 & 72.0 & 46.2 & 23.2 & 17.9 & 55.7 & 35.0 \\
    ~~$\bullet~$COME & 41.1 & 42.7 & 40.4 & 29.4 & 0.4 & 33.3 & 2.4 & 25.1 & 24.9 & 0.2 & 74.2 & 57.7 & 0.6 & 27.6 & 58.3 & 30.6 \\
    ~~$\bullet~$DeYO & 34.9 & 35.5 & 35.3 & 22.0 & 0.9 & 23.0 & 3.0 & 38.6 & 29.1 & 0.3 & 72.1 & 52.3 & 1.8 & 22.1 & 55.6 & 28.4 \\
    \rowcolor{black!8}~~$\bullet~$ZeroSiam (ours) & 42.4 & 44.9 & 42.5 & 37.2 & 33.9 & 45.5 & 44.5 & 56.3 & 53.3 & 64.4 & 73.7 & 60.5 & 44.7 & 56.3 & 58.5 & \textbf{50.6} \\
    \cmidrule{1-17}
    Swin-Tiny & 26.9 & 27.7 & 23.0 & 21.7 & 13.2 & 26.4 & 25.3 & 41.2 & 45.8 & 47.7 & 67.8 & 39.2 & 22.5 & 8.4 & 33.2 & 31.3 \\
    ~~$\bullet~$Tent & 0.2 & 0.4 & 0.1 & 14.0 & 5.9 & 7.2 & 5.0 & 9.6 & 8.9 & 2.6 & 67.7 & 16.1 & 1.4 & 0.3 & 34.4 & 11.6 \\
    ~~$\bullet~$SAR & 27.3 & 29.0 & 24.8 & 21.6 & 14.9 & 26.7 & 25.7 & 36.6 & 45.3 & 47.5 & 67.8 & 39.4 & 10.5 & 8.6 & 36.5 & 30.8 \\
    ~~$\bullet~$EATA & 27.9 & 29.0 & 25.3 & 20.7 & 14.6 & 28.4 & 25.7 & 40.1 & 40.7 & 47.4 & 68.1 & 40.8 & 23.6 & 9.2 & 33.8 & 31.7 \\
    ~~$\bullet~$COME & 38.1 & 41.3 & 42.0 & 26.2 & 0.5 & 10.8 & 0.5 & 10.3 & 8.0 & 10.2 & 72.8 & 52.0 & 0.1 & 0.5 & 55.1 & 24.6 \\
    ~~$\bullet~$DeYO & 0.2 & 0.1 & 0.2 & 13.8 & 6.4 & 11.4 & 3.7 & 11.6 & 16.9 & 4.3 & 70.2 & 44.4 & 3.7 & 0.4 & 47.1 & 15.6 \\
    \rowcolor{black!8}~~$\bullet~$ZeroSiam (ours) & 45.1 & 47.5 & 47.0 & 35.2 & 40.3 & 46.1 & 46.9 & 56.1 & 51.9 & 58.5 & 68.6 & 54.8 & 54.7 & 52.9 & 56.3 & \textbf{50.8} \\
    \end{tabular}
    }
    \end{threeparttable}
    \end{center}
\end{table*}

\begin{table*}[t]
    \caption{Detailed results of TTA under the \textbf{\textsc{mild scenario}}, \ie, Table~\ref{tab:imagenet-c-mild} in the main paper.}
    \vspace{-0.1in}
    \label{tab:imagenet-c-mild-details}
    \begin{center}
    \begin{threeparttable}
    \large
    \resizebox{1.0\linewidth}{!}{
    \begin{tabular}{l|ccc|cccc|cccc|cccc|>{\columncolor{blue!8}}c}
    \multicolumn{1}{c}{} & \multicolumn{3}{c}{Noise} & \multicolumn{4}{c}{Blur} & \multicolumn{4}{c}{Weather} & \multicolumn{4}{c}{Digital} \\
    Model+Method & Gauss. & Shot & Impul. & Defoc. & Glass & Motion & Zoom & Snow & Frost & Fog & Brit. & Contr. & Elastic & Pixel & JPEG & Avg. \\
    \midrule
    ResNet50-GN & 18.0 & 19.8 & 17.9 & 19.8 & 11.4 & 21.4 & 24.9 & 40.4 & 47.3 & 33.6 & 69.3 & 36.3 & 18.6 & 28.4 & 52.3 & 30.6 \\
    ~~$\bullet~$Tent & 5.1 & 6.1 & 5.4 & 15.0 & 10.8 & 22.1 & 22.9 & 25.8 & 33.2 & 3.1 & 70.4 & 42.7 & 11.0 & 48.1 & 54.2 & 25.1 \\
    ~~$\bullet~$SAR & 28.5 & 31.4 & 29.7 & 18.7 & 19.5 & 30.2 & 30.1 & 43.3 & 43.7 & 7.0 & 70.8 & 44.0 & 19.0 & 48.8 & 55.2 & 34.7 \\
    ~~$\bullet~$EATA & 37.3 & 39.1 & 39.4 & 28.1 & 27.1 & 36.8 & 39.1 & 50.8 & 49.0 & 55.8 & 72.0 & 50.1 & 41.8 & 55.7 & 58.1 & 45.4 \\
    ~~$\bullet~$COME & 17.9 & 19.7 & 17.7 & 19.7 & 11.3 & 21.4 & 24.9 & 40.4 & 47.5 & 33.6 & 69.1 & 35.9 & 18.6 & 27.9 & 52.1 & 30.5 \\
    ~~$\bullet~$DeYO & 39.6 & 42.1 & 40.6 & 22.2 & 23.3 & 38.3 & 37.7 & 50.4 & 49.4 & 3.0 & 73.0 & 50.3 & 41.9 & 55.5 & 57.7 & 41.7 \\
    \rowcolor{black!8}~~$\bullet~$ZeroSiam (ours) & 40.5 & 42.4 & 41.7 & 31.4 & 31.8 & 40.8 & 44.4 & 53.5 & 53.3 & 59.1 & 72.9 & 51.5 & 47.4 & 58.9 & 58.8 & \textbf{48.6} \\
    \cmidrule{1-17}
    ViT-Base & 9.4 & 6.7 & 8.3 & 29.1 & 23.4 & 34.0 & 27.0 & 15.8 & 26.3 & 47.4 & 54.7 & 43.9 & 30.5 & 44.5 & 47.6 & 29.9 \\
    ~~$\bullet~$Tent & 42.4 & 1.4 & 43.3 & 52.2 & 47.6 & 55.4 & 49.9 & 19.7 & 18.1 & 66.1 & 74.9 & 64.7 & 52.5 & 66.8 & 64.1 & 47.9 \\
    ~~$\bullet~$SAR & 44.3 & 14.0 & 45.7 & 52.9 & 49.8 & 56.0 & 51.1 & 58.3 & 50.6 & 66.5 & 74.6 & 64.3 & 55.5 & 66.5 & 64.0 & 54.3 \\
    ~~$\bullet~$EATA & 49.7 & 49.4 & 51.3 & 55.9 & 55.5 & 60.5 & 57.8 & 63.4 & 63.0 & 70.4 & 76.0 & 67.0 & 64.7 & 69.3 & 67.9 & 61.5 \\
    ~~$\bullet~$COME & 50.9 & 51.6 & 52.1 & 56.9 & 57.4 & 62.2 & 59.2 & 66.1 & 64.3 & 72.5 & 77.4 & 67.9 & 66.3 & 72.4 & 69.2 & \textbf{63.1} \\
    ~~$\bullet~$DeYO & 52.9 & 53.1 & 53.8 & 58.2 & 58.6 & 63.0 & 40.2 & 67.4 & 65.7 & 73.2 & 78.0 & 68.0 & 67.7 & 73.1 & 69.8 & 62.8 \\
    \rowcolor{black!8}~~$\bullet~$ZeroSiam (ours) & 50.5 & 51.0 & 51.5 & 57.1 & 56.8 & 61.2 & 58.6 & 64.9 & 64.5 & 71.9 & 77.2 & 67.6 & 65.6 & 71.5 & 68.5 & 62.6 \\
    \cmidrule{1-17}
    ViT-Small & 2.0 & 2.0 & 1.5 & 24.2 & 17.1 & 30.3 & 22.0 & 9.5 & 19.2 & 37.9 & 44.4 & 30.1 & 24.8 & 38.2 & 41.0 & 22.9 \\
    ~~$\bullet~$Tent & 0.5 & 0.6 & 0.4 & 39.5 & 33.0 & 42.5 & 34.0 & 10.3 & 37.1 & 54.5 & 63.6 & 51.7 & 31.5 & 53.1 & 52.3 & 33.6 \\
    ~~$\bullet~$SAR & 1.4 & 2.1 & 1.1 & 40.8 & 34.9 & 43.2 & 35.5 & 16.3 & 41.8 & 54.7 & 63.7 & 51.8 & 38.5 & 53.4 & 52.5 & 35.4 \\
    ~~$\bullet~$EATA & 22.4 & 20.6 & 23.6 & 45.0 & 43.0 & 47.2 & 43.5 & 48.5 & 49.9 & 60.7 & 66.4 & 56.6 & 52.7 & 58.3 & 57.2 & 46.4 \\
    ~~$\bullet~$COME & 0.2 & 0.4 & 0.2 & 47.7 & 46.3 & 51.3 & 19.6 & 37.9 & 52.7 & 63.6 & 69.8 & 58.3 & 54.1 & 61.8 & 59.2 & 41.6 \\
    ~~$\bullet~$DeYO & 0.2 & 0.3 & 0.2 & 47.9 & 45.8 & 51.4 & 33.6 & 44.5 & 53.0 & 63.5 & 70.1 & 58.5 & 55.3 & 62.4 & 59.8 & 43.1 \\
    \rowcolor{black!8}~~$\bullet~$ZeroSiam (ours)& 23.8 & 23.4 & 25.2 & 47.6 & 46.0 & 50.6 & 48.1 & 51.9 & 53.3 & 63.2 & 69.9 & 58.1 & 55.7 & 61.8 & 59.4 & \textbf{49.2} \\
    \cmidrule{1-17}
    ConvNext-Tiny & 21.4 & 27.7 & 22.1 & 24.5 & 11.0 & 32.5 & 31.2 & 44.5 & 52.5 & 39.5 & 72.0 & 44.8 & 23.1 & 17.7 & 55.8 & 34.7 \\
    ~~$\bullet~$Tent & 30.8 & 33.6 & 32.0 & 24.7 & 7.5 & 33.6 & 30.0 & 43.7 & 47.3 & 43.6 & 72.9 & 49.5 & 21.5 & 19.3 & 56.3 & 36.4 \\
    ~~$\bullet~$SAR & 33.4 & 35.6 & 33.7 & 27.1 & 8.1 & 34.1 & 30.7 & 44.6 & 47.3 & 17.0 & 72.8 & 48.8 & 22.0 & 24.4 & 56.1 & 35.7 \\
    ~~$\bullet~$EATA & 35.3 & 37.4 & 35.8 & 31.0 & 20.7 & 37.1 & 33.7 & 48.5 & 49.1 & 51.3 & 73.2 & 53.0 & 29.5 & 31.3 & 56.6 & 41.6 \\
    ~~$\bullet~$COME & 40.4 & 42.6 & 40.4 & 36.2 & 2.3 & 43.2 & 23.4 & 54.6 & 54.2 & 58.1 & 75.1 & 58.6 & 10.6 & 44.7 & 59.9 & 43.0 \\
    ~~$\bullet~$DeYO & 36.3 & 39.0 & 36.4 & 29.1 & 2.7 & 37.1 & 19.6 & 50.1 & 48.3 & 4.8 & 73.8 & 54.9 & 8.6 & 23.7 & 57.4 & 34.8 \\
    \rowcolor{black!8}~~$\bullet~$ZeroSiam (ours) & 39.6 & 41.7 & 39.3 & 33.8 & 27.3 & 42.0 & 38.7 & 53.2 & 51.8 & 59.3 & 73.8 & 57.0 & 36.5 & 44.2 & 58.2 & \textbf{46.4} \\
    \cmidrule{1-17}
    Swin-Tiny & 26.6 & 27.9 & 22.8 & 21.7 & 13.2 & 26.5 & 25.4 & 41.0 & 45.8 & 47.7 & 67.6 & 39.1 & 22.6 & 8.4 & 33.3 & 31.3 \\
    ~~$\bullet~$Tent & 23.7 & 26.2 & 22.7 & 19.4 & 18.6 & 27.1 & 27.4 & 40.5 & 36.3 & 24.4 & 69.2 & 46.3 & 16.9 & 2.8 & 41.5 & 29.5 \\
    ~~$\bullet~$SAR & 31.8 & 32.8 & 32.5 & 22.3 & 20.2 & 29.6 & 28.4 & 42.6 & 39.8 & 37.6 & 69.2 & 46.8 & 19.9 & 6.0 & 43.2 & 33.5 \\
    ~~$\bullet~$EATA & 40.9 & 42.9 & 41.6 & 33.5 & 33.6 & 42.6 & 40.1 & 50.1 & 49.1 & 54.9 & 71.8 & 52.6 & 44.8 & 40.0 & 51.1 & 46.0 \\
    ~~$\bullet~$COME & 41.2 & 43.2 & 42.1 & 35.2 & 15.0 & 44.7 & 7.0 & 52.9 & 51.9 & 57.8 & 72.1 & 51.1 & 23.8 & 25.2 & 54.3 & 41.2 \\
    ~~$\bullet~$DeYO & 35.2 & 37.0 & 36.0 & 22.2 & 25.0 & 35.8 & 15.4 & 42.0 & 43.8 & 48.0 & 70.1 & 50.3 & 31.6 & 34.5 & 48.0 & 38.3 \\
    \rowcolor{black!8}~~$\bullet~$ZeroSiam (ours) & 42.4 & 44.0 & 43.0 & 31.7 & 34.9 & 43.6 & 42.6 & 51.3 & 48.9 & 54.7 & 71.5 & 53.4 & 47.5 & 45.8 & 52.3 & \textbf{47.2} \\
    \end{tabular}
    }
    \end{threeparttable}
    \end{center}
\end{table*}

\clearpage

\section{Additional Discussions}

\subsection{\methodname's Efficacy beyond Collapse Prevention}
\methodname regularizes test-time entropy minimization from favoring non-generalizable shortcuts (c.f. Section~\ref{sec:empric_theory_evidence}), thereby improving its performance even in cases when no collapse occurs. As shown in Figure~\ref{fig:motivation-figures-suppl}, Tent drastically reduces entropy and eventually closes it to zero by inflating the logit norms and aligning all logits toward a dominant mode. Despite the reduction of entropy, it does not involve meaningful learning, and Tent thus fails to improve the accuracy on the first 250 batches of samples, where such biased signals can further degrade performance during a prolonged adaptation, as shown in Figure~\ref{fig:motivation-figures-suppl}. In contrast, \methodname maintains a stable logit norm and center dominance ratio during entropy minimization, which successfully enhances TTA efficacy beyond collapse prevention, as shown in Figure~\ref{fig:motivation-figures-suppl} (a). Moreover, from Figure~\ref{fig:motivation-figures-suppl} (d), \methodname does not greedily minimize the entropy loss. Instead, \methodname converges to a low but non-zero entropy loss, while still enabling consistent accuracy improvement, \ie, from 250 to 750 batches as in Figure~\ref{fig:motivation-figures-suppl} (a). Such convergence toward non-zero loss is also aligned with Theory~\ref{proofs: thm1}, which shows that \methodname inherently defines a lower bound of $h_{min}$ for entropy optimization, preventing the model from degenerating into collapsed constant one-hot outputs that trivially minimize (\ie, zero out) the entropy loss.

% Empirical evidence of \methodname's stabilization effects
\begin{figure*}[h]
    \centering
    % \vspace{-0.3in}
    \includegraphics[width=1.\linewidth]{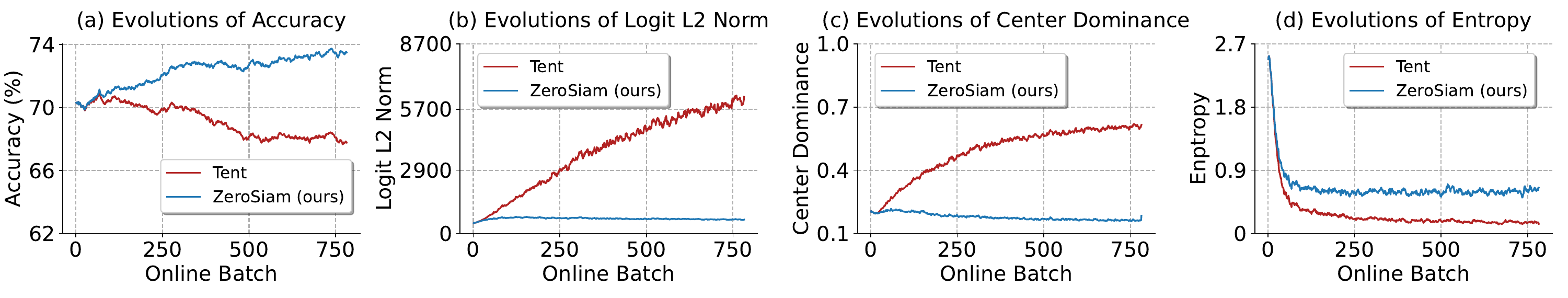}
    % \vspace{-0.2in}
    \caption{Empirical evidence of \methodname for boosting stability and efficacy. (a-d) record the ODD accuracy, logits $L_2$ norm, center dominance, and entropy in model predictions under a mild test scenario~\citep{wang2021tent}. Experiments are run on ImageNet-C (Bright, level 5) with ViT-Base.
    % For fair comparisons, \methodname and Tent use the same learning rate configuration.
    % (\ie, differing only in extra predictor and similarity loss).
    }
    \label{fig:motivation-figures-suppl}
    % \vspace{-0.05in}
\end{figure*}

\subsection{Evolutions of Divergence Loss in \methodname}
We provide additional results of how the divergence loss evolves during adaptation to supplement Figure~\ref{fig:motivation-figures}. As shown in Figure~\ref{fig:similarity_loss}, the divergence loss increases more rapidly and substantially under a more imbalanced stream. Such a phenomenon is also aligned with changes of the predictor as shown in Figure~\ref{fig:motivation-figures} (a), where a larger divergence from an identity mapping results in a larger similarity loss. Overall, these results suggest \textit{an adaptive regularization strength} according to the degree of collapse risk in the scenario, indicating the efficacy of our asymmetric entropy optimization design. 

\begin{figure*}[h]
\centering\includegraphics[width=0.6\linewidth]{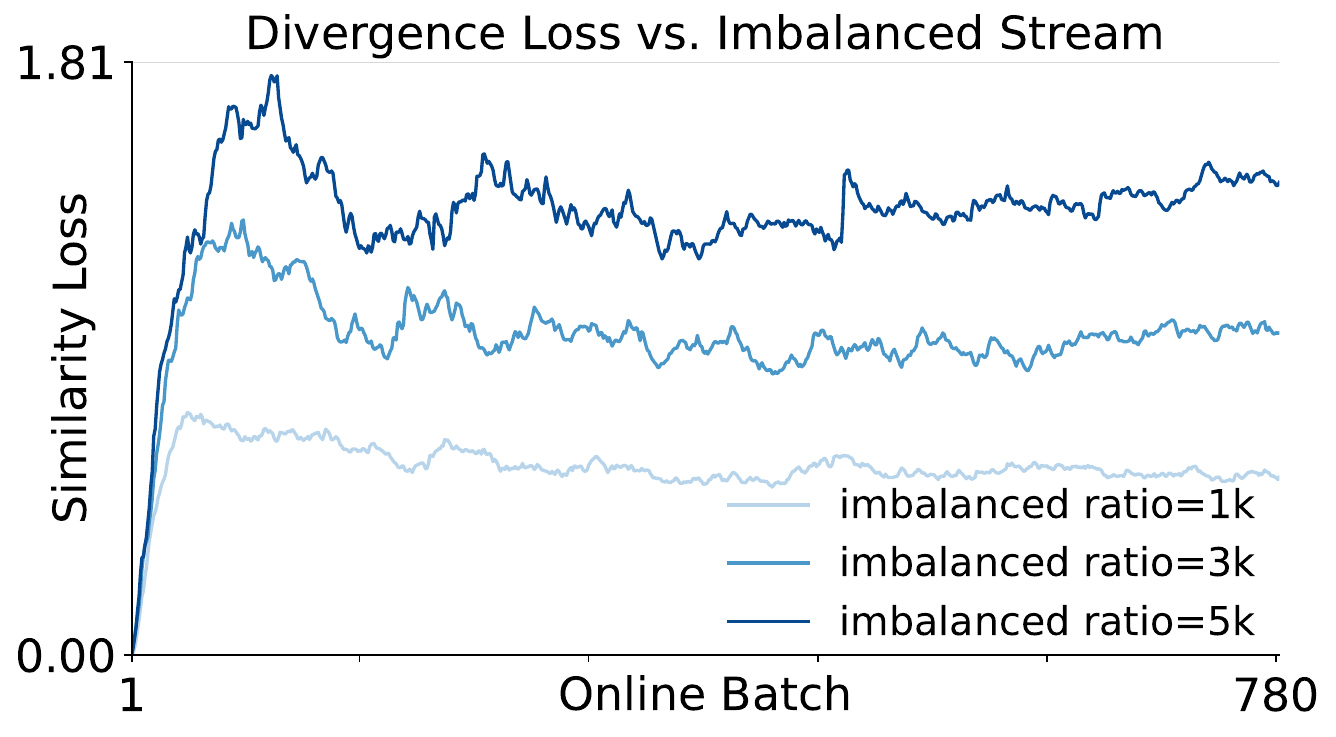}\caption{The evolution of divergence loss during TTA. Results are reported on ImageNet-C (Snow, level 5) with ResNet50-GN under online streams with varying imbalanced ratios~\citep{niu2023sar}. }
\label{fig:similarity_loss}
\end{figure*}

\subsection{\methodname for Saving a Collapsed Model}
Interestingly, we further demonstrate that \methodname can sometimes save a model that has already collapsed. To illustrate this, we first run Tent on an imbalanced data stream until collapse occurs, and then apply \methodname. As shown in Figure~\ref{fig:acc_recovery}, \methodname helps restore strong accuracy after 300 batches of adaptation. This recovery is driven by the asymmetric alignment in \methodname, which makes collapsed solutions no longer a stable minimum, encouraging the model to escape the collapsed mode and re-cluster samples. However, this does not fully explain how class-wise clusters that are consistent with pre-training re-emerge. We hypothesize this is because only the affine parameters in the normalization layers are updated during testing, which introduces a small adjustment to the model that makes performance recovery possible. Notably, we observe that the predictor in \methodname has to be randomly initialized (\eg, $\theta_h=\mathbf{I}+0.1\,\mathbf{W}$, $\mathbf{W}\!\sim\!\mathcal{N}(0,\mathbf{I})$) so as to deviate from the identity and also remain learnable during TTA to enable the collapse recovery effects. Even so, recovery succeeds in only 4 out of 7 domains, suggesting that while promising, collapse recovery with \methodname remains an open question, and we leave it for future work.

\begin{figure*}[h]
\centering
\includegraphics[width=0.6\linewidth]{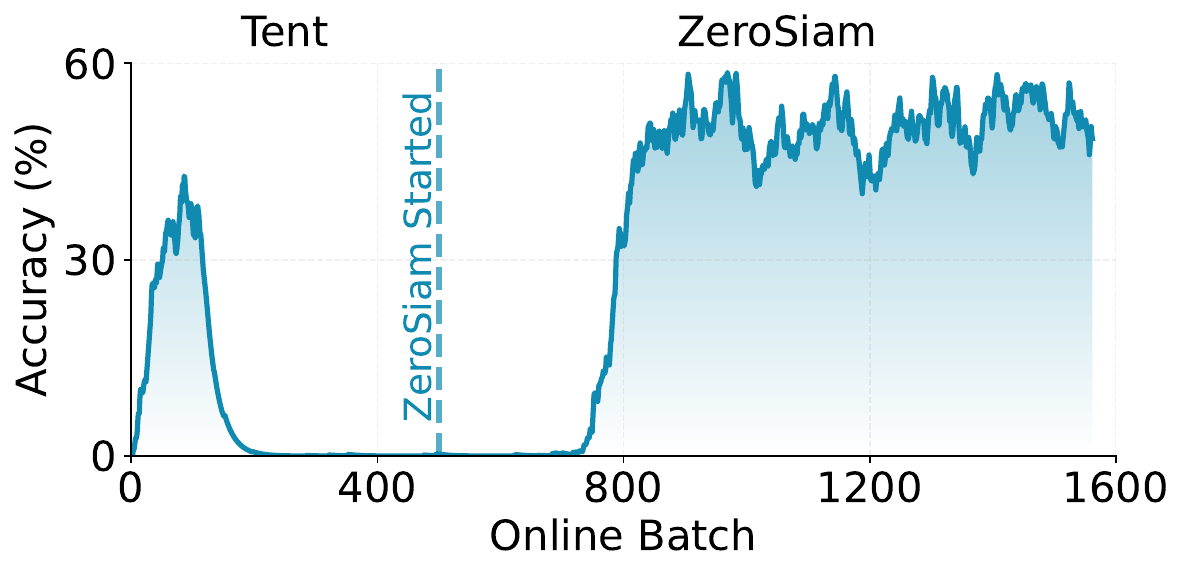}
% \vspace{-0.11in}
\caption{Efficacy of \methodname for saving a collapsed model. Results are reported on ImageNet-C (Shot, level 5) with ViT-Base under \textbf{\textsc{online imbalanced label shifts}} (imbalance ratio = $\infty$).}
\label{fig:acc_recovery}
% \vspace{-0.2in}
\end{figure*}

\subsection{Comparisons between the Entropy and Divergence Term during TTA}
As shown in Figure~\ref{fig:entropy_divergence_compare}, the alignment loss $D(p^o||\mathrm{sg}[p^{r}])$ does not grow infinitely. Instead, the alignment loss quickly increases then converges, while the entropy term $H(p^{o})$ decreases and stabilizes. Throughout adaptation, the alignment term exhibits a smaller loss value than the entropy term, which provides an appropriate regularization strength that prevents the risk of collapse, without hindering the effectiveness of the entropy-based adaptation.

\begin{figure*}[h]
\centering
\includegraphics[width=0.6\linewidth]{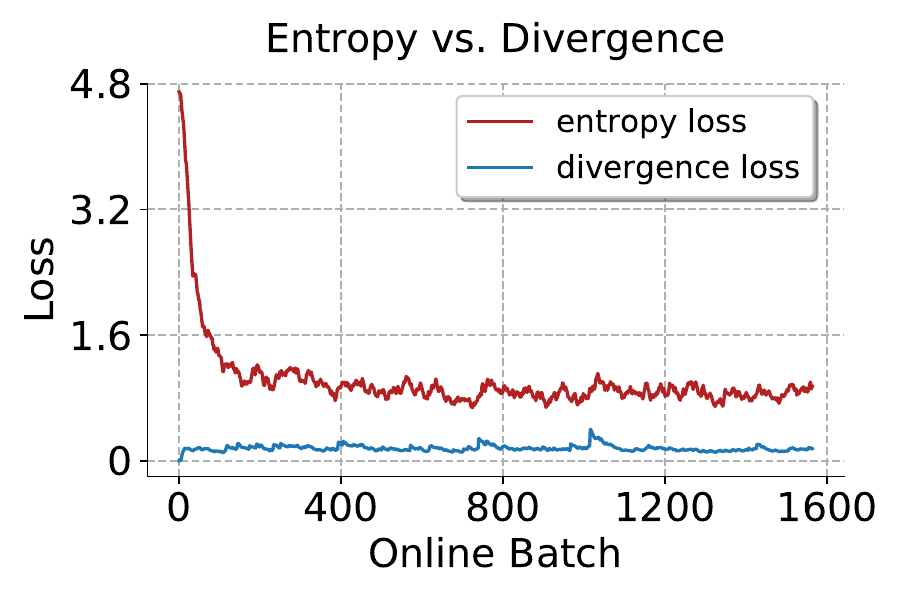}
% \vspace{-0.11in}
\caption{Comparisons between entropy and divergence during TTA on ImageNet-C (Gaussian, level 5) under \textbf{\textsc{online imbalanced label shifts}} (imbalance ratio = $\infty$) with ViT-Base.}
\label{fig:entropy_divergence_compare}
% \vspace{-0.2in}
\end{figure*}

\subsection{More Sensitivity Analysis of Learning Rates in \methodname}
To demonstrate \methodname's robustness, we further supplement the sensitivity analysis of learning rates on Swin-Tiny, a challenging architecture on which Tent frequently suffers from collapse. 
\label{sec:suppl:learning-rate-on-swin}

From Figure~\ref{fig:swin-lr-heatmap}, \methodname consistently outperforms Tent across a wide range of encoder learning rates, confirming the stability trends observed in Figure~\ref{fig:hyperparameter_sensitivity} across architectures. Moreover, on smaller and more collapse-prone models, \methodname can benefit from a larger predictor learning rate $\eta_h$, \eg, yielding an accuracy of 53.3\% with $\eta_f=10\times5e-5$ and $\eta_h=40\times\eta_f$; while still outperforming previous methods across broad learning rate configurations, \eg, achieving an accuracy above 45\% when $\eta_f\in[5,20]\times5e-5$ and $\eta_h \in [1,40]\times\eta_f$, showing a substantial gain compared to the accuracy of 38.2\% in COME and 37.7\% in DeYO. These results collectively demonstrate the efficacy of \methodname across learning rate choices. In practice, we did not extensively tune the learning rates, where we simply use $\eta_f=10\times5e-5$ and $\eta_h=5\times\eta_f$ on \methodname, which yields an accuracy of 50.4\% for comparisons with other methods.

\begin{figure*}[h]
\centering
\includegraphics[width=0.5\linewidth]{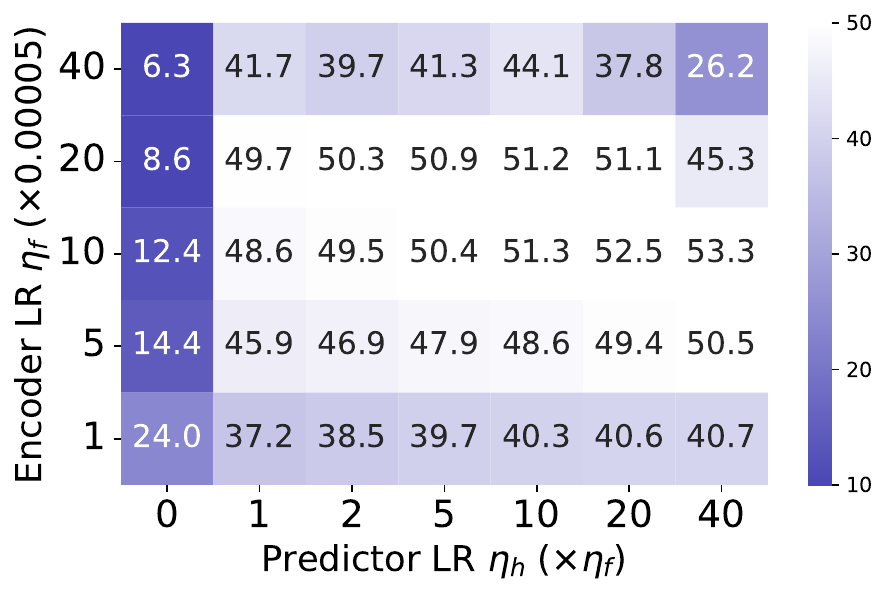}
% \vspace{-0.11in}
\caption{Sensitivity to learning rates on Swin-Tiny. Results are reported on ImageNet-C (Gaussian, level 5) under \textbf{\textsc{online imbalanced label shifts}} (imbalance ratio = $\infty$).}
\label{fig:swin-lr-heatmap}
% \vspace{-0.2in}
\end{figure*}

\subsection{Discussions on a \textit{Data-Free} Approach for Selecting $\eta_h$ in \methodname}
\label{sec:suppl:data-free}

While the predictor lr $\eta_h$ in \methodname is simply set the same as $\eta_f$ on larger models and around 10$\times$ larger than $\eta_f$ on smaller and more collapse-prone models, we further briefly introduce \textbf{a data-free method} to help select $\eta_h$. The core idea is to leverage random Gaussian noise as the proxy to track TTA dynamics, where ideally, noisy learning signals should be totally absorbed and suppressed in ZeroSiam. To this end, we track the changes in prediction, \eg, prediction entropy $\Delta_e$, before and after adaptation on Gaussian noise, and select the best-performing $\eta_h$ that leads to a minimal change, \eg, $\min_{\eta_h} |\Delta_e|$, which implies an appropriate regularization strength for bias filtration. As shown in Table V, on ResNet50-GN, we can find $\eta_h=10$ by $\min_{\eta_h} |\Delta_e|$ using only generated Gaussian before TTA deployment, which also delivers the highest TTA accuracy of 51.4\% under label shift. 

\begin{table}[h]
    % \vspace{-0.02in}
    \caption{Illustration of data-free selection for $\eta_h$ on ResNet50-GN based on entropy difference $|\Delta_e|$ before and after adaptation on 6,400 random Gaussian inputs. TTA accuracy is measured under the label shift setting. Encoder learning rate $\eta_f$ is set to 0.00125, and $\eta_h$ is set to $k\times$ 0.00125.}
     \vspace{-0.06in}
    \label{tab:sensitivity_rho}
\newcommand{\tabincell}[2]{\begin{tabular}{@{}#1@{}}#2\end{tabular}}
 \begin{center}
 \begin{threeparttable}
    \resizebox{0.9\linewidth}{!}{
 	\begin{tabular}{ccccccc}
    % \multicolumn{2}{c}{} & \multicolumn{3}{c}{Tent} & \multicolumn{3}{c}{\methodname}\\
 	 % Method & No adapt & $k=1$  & $k=3$  & $k=5$  & $k=1$  & $k=3$  & $k=5$\\
    Metric & $\eta_h=0$ & $\eta_h=2$  & $\eta_h=5$  & $\eta_h=10$  & $\eta_h=20$ & $\eta_h=40$ \\
 	\midrule
    prediction $|\Delta_e|$ & 0.750 & 0.085 & 0.053 & 0.001 & 0.061 & 0.083 \\
    TTA Acc. (\%) & 14.3 & 49.6 & 51.0 & 51.4 & 50.3 & 45.3 \\
    
    \end{tabular}
	}
	 \end{threeparttable}
	 \end{center}
\end{table}

We acknowledge that some architectures, such as ConvNeXt, can consistently output the highest prediction entropy on Gaussian noise regardless of adaptation and may require tracking other prediction metrics for data-free $\eta_h$ selection. We leave this for future work, and believe that our current focus on collapse prevention and asymmetric entropy-based architecture already makes meaningful contributions, as this significantly reduces the risk of TTA across adaptation scenarios while maintaining efficiency for practical and reliable deployment.

\subsection{Efficacy of \methodname across Diverse Adaptation Epochs}
\label{sec:suppl:different-epochs}

In our main experiments, all methods are evaluated with only 1 epoch over the test set, following the evaluation settings established by prior works. Here, we further evaluate Tent and \methodname using more adaptation epochs under Tent’s default mild setting. As shown in Table~\ref{tab:more-epochs}, Tent’s performance gradually degrades with more epochs (47.9\% → 41.3\%), reflecting its known overfitting tendency. In contrast, \methodname remains stable and even benefits moderately from additional epochs (62.6\% → 65.7\%). This confirms that ZeroSiam improves generalization while mitigating overfitting in Tent, which maintains a stable adaptation across different epoch budgets.

% This confirms that \methodname maintains stable adaptation without relying on collapse to demonstrate its efficacy and remains robust across different epoch budgets.

\begin{table}[h]
    % \vspace{-0.02in}
    \caption{Accuracy of Tent and ZeroSiam under the $k$-th epoch of TTA over the entire test set. Results are reported with ViT-Base under Tent's default mild setting.}
    \setlength{\tabcolsep}{3pt}
     % \vspace{-0.06in}
    \label{tab:more-epochs}
\newcommand{\tabincell}[2]{\begin{tabular}{@{}#1@{}}#2\end{tabular}}
 \begin{center}
 \begin{threeparttable}
    \resizebox{1\linewidth}{!}{
 	\begin{tabular}{lccccccc}
    % \multicolumn{2}{c}{} & \multicolumn{3}{c}{Tent} & \multicolumn{3}{c}{\methodname}\\
 	 % Method & No adapt & $k=1$  & $k=3$  & $k=5$  & $k=1$  & $k=3$  & $k=5$\\
    Method & No adapt & Tent (k=1)  & Tent (k=3)  & Tent (k=5)  & \methodname (k=1) & \methodname (k=3) & \methodname (k=5) \\
 	\midrule
    Acc. (\%) & 29.9 & 47.9 & 45.5 & 41.3 & 62.6 & 65.7 & 65.7 \\
    \end{tabular}
	}
	 \end{threeparttable}
	 \end{center}
\end{table}

\subsection{Necessity of Stop-gradient for Asymmetry}
As verified in Theory~\ref{proofs: thm1}, the online branch in \methodname converges more rapidly towards the collapse solutions during TTA. The stop-gradient on the target branch prevents the target branch from also drifting toward the collapse modes, while providing a stable reference for the online branch. As shown in Table~\ref{tab:stop-gradient-importance}, removing stop-gradient leads to severe collapse, \eg, reducing the accuracy from 51.6\%$\rightarrow$20.5\% on ResNet50-GN. Similar results are also observed in SimSiam~\citep{chen2021exploring} and BYOL~\citep{grill2020bootstrap}, underscoring stop-gradient as a key component in asymmetry.

\begin{table*}[h]
    \caption{Importance of stop-gradient operation. Results are reported on ImageNet-C (severity level 5) under \textbf{\textsc{online imbalanced label shifts}} (imbalance ratio = $\infty$)  regarding \textbf{Accuracy (\%)}.}
    \label{tab:stop-gradient-importance}
    \begin{center}
    \begin{threeparttable}
    \large
    \resizebox{1.0\linewidth}{!}{
    \begin{tabular}{l|ccc|cccc|cccc|cccc|c}
    \multicolumn{1}{c}{} & \multicolumn{3}{c}{Noise} & \multicolumn{4}{c}{Blur} & \multicolumn{4}{c}{Weather} & \multicolumn{4}{c}{Digital} \\
    % \midrule
    ResNet50-GN & Gauss. & Shot & Impul. & Defoc. & Glass & Motion & Zoom & Snow & Frost & Fog & Brit. & Contr. & Elastic & Pixel & JPEG & Avg. \\
    \midrule
    ZeroSiam (ours) & 42.9 & 45.1 & 43.1 & 35.1 & 36.1 & 44.5 & 49.2 & 58.1 & 55.1 & 62.9 & 72.3 & 54.8 & 52.5 & 61.7 & 60.1 & 51.6 \\
    - w/o stop-grad & 2.1 & 2.7 & 2.0 & 15.3 & 2.0 & 15.2 & 11.1 & 11.9 & 17.7 & 1.2 & 71.9 & 43.5 & 3.8 & 50.9 & 55.5 & 20.5 \\
    \midrule
    ViT-Base & Gauss. & Shot & Impul. & Defoc. & Glass & Motion & Zoom & Snow & Frost & Fog & Brit. & Contr. & Elastic & Pixel & JPEG & Avg. \\
    \midrule
    ZeroSiam (ours) & 52.3 & 52.6 & 53.4 & 57.7 & 58.7 & 62.7 & 61.1 & 67.6 & 66.0 & 73.3 & 78.0 & 67.0 & 67.9 & 73.5 & 70.1 & 64.1 \\
    - w/o stop-grad & 2.8 & 0.8 & 12.8 & 54.1 & 51.1 & 57.6 & 33.7 & 6.3 & 8.0 & 68.7 & 76.1 & 65.9 & 4.4 & 69.2 & 66.3 & 38.5 \\
    % \midrule
    \end{tabular}
    }
    \end{threeparttable}
    \end{center}
\end{table*}

\subsection{Statistical comparison}
We re-run Table~\ref{tab:imagenet-c-label-shift} in the main paper with 5 different random seeds. As shown in Table~\ref{tab:statistical-comparison}, \methodname achieves both higher adaptation accuracy and lower standard deviation compared to prior methods. Specifically, \methodname increases the average accuracy from 42.6\% (DeYO) to 52.9\%, while reducing the standard deviation by over \textit{10}-fold. This reduced standard deviation further underscores our reliability for robust adaptation in real-world applications.

\begin{table}[h]
    % \vspace{-0.02in}
    \caption{Statistical comparison. Experiments follow the settings of Table~\ref{tab:imagenet-c-label-shift} in the main paper. Results are reported with 5 random seeds.}
    % \setlength{\tabcolsep}{3pt}
     % \vspace{-0.06in}
    \label{tab:statistical-comparison}
\newcommand{\tabincell}[2]{\begin{tabular}{@{}#1@{}}#2\end{tabular}}
 \begin{center}
 \begin{threeparttable}
    \resizebox{1\linewidth}{!}{
 	\begin{tabular}{lcccccc}
    % \multicolumn{2}{c}{} & \multicolumn{3}{c}{Tent} & \multicolumn{3}{c}{\methodname}\\
 	 % Method & No adapt & $k=1$  & $k=3$  & $k=5$  & $k=1$  & $k=3$  & $k=5$\\
    Method & ResNet50-GN & ViT-Base  & ViT-Small  & ConvNeXt-Tiny  & Swin-Tiny & Average \\
 	\midrule
    DeYO & 39.7$_{\pm 2.05}$ &  61.8$_{\pm 2.19}$ & 40.8$_{\pm 0.36}$ & 34.3$_{\pm 0.40}$ &  36.5$_{\pm 0.71}$ & 42.6$_{\pm 0.82}$ \\
    \methodname & 51.2$_{\pm 0.06}$ & 63.5$_{\pm 0.10}$ & 50.0$_{\pm 0.06}$ & 49.8$_{\pm 0.11}$ & 50.2$_{\pm 0.16}$ & 52.9$_{\pm 0.06}$ \\
    \end{tabular}
	}
	 \end{threeparttable}
	 \end{center}
\end{table}

\subsection{More Discussions on the Impacts of Uncontrolled Logit Norm Inflation}
During test-time entropy minimization, uncontrolled inflation of logit norm can induce the following issues that hinder robust and reliable model adaptation: 1) \textit{Overconfidence}: regarding individual predictions, increasing the logit magnitude $||u||_2$ (where $u$ is the predicted logit) artificially inflates predictive confidence without improving accuracy, since the predicted class $\arg \max_k u_k$ depends solely on the direction $u'=\frac{u}{||u||_2}$, not its magnitude. This leads to overly confident but inaccurate predictions that undermine trust in AI systems. 2) \textit{Degraded generalization from gradient interference}: while inflating the logit norm does not benefit a sample’s own generalization, it adversely affects others through gradient interference during batch updates. Empirically, test-time adaptation with only the objective of $-||u||_2$ to inflate the logit norm would lead to a catastrophic collapse of zero accuracy for all models. Theoretically, define the decision margins as $m(x)=\max_ku_k-\max_{j\ne k}u_j$ and the sharpness of relative boundaries $\frac{|| \nabla_x m(x) ||}{m(x)}$, given that both $m(x)$ and $|| \nabla_x m(x) ||$ scale proportionally with $||u||_2$, inflating the logit norm cannot improve the sample's own relative boundary. However, since per-example gradients are generally non-orthogonal, a sample's update can \textit{add irrelevant noise to the boundaries of other samples while having no benefit on its current prediction}. Such interference accumulates during online TTA, adversely affecting generalization and ultimately causing collapse. ZeroSiam mitigates this by aligning with a stable target branch, preventing uncontrolled norm growth, and improving TTA efficacy.

% if one sample has a substantially inflated logit norm $||u_1||\gg 1$, its gradient dominates the shared update: $g = \sum_{i} \nabla_\theta \mathcal{L}(x_i) \approx ||u_1|| g_1 + \sum_{i>1} g_i$. Since per-example gradients are generally non-orthogonal, this imbalance distorts the optimization trajectory, adding noise to decision margins $m(x)=\max_ku_k-\max_{j\ne k}u_j$ and the sharpness of relative boundaries (\ie, $\frac{|| \nabla_x m(x) ||}{m(x)}$). Under class imbalance or streaming shifts, such interference accumulates, ultimately causing collapse. In contrast, \methodname mitigates this by aligning with a stable target branch, preventing uncontrolled norm growth and ensuring balanced gradient contributions. 

\subsection{More Discussions with Multi-Branch Adaptation Methods}
\methodname advances the entropy-based, self-training TTA by adding both the mechanisms of bias learning signals filtration (c.f. Section~\ref{sec:empric_theory_evidence}) and asymmetric optimization (c.f. Section~\ref{sec:zerosiam-architecture}) through a lightweight predictor with theoretical insights (Theorem~\ref{proofs: thm1}). We further detail our distinctions with multi-branch-based TTA methods, such as SPA~\citep{niu2025spa}, REM~\citep{han2025rem}, and TTE~\citep{kim2025tte} from three aspects. 

(1) \textit{Distinct inspirations from self-supervised learning (objective vs. architecture):} SPA and REM are augmentation-based methods that exploit consistency learning—an objective similar to self-supervised approaches—across different views to promote adaptation to the target domain. To this end, their efforts focus on devising appropriate augmentations or deteriorations on images for information masking and creating a weak-to-strong consistency learning. In contrast, \methodname is entirely augmentation-free and focuses on the asymmetric architecture design—orthogonal to efforts in SPA and REM—for the single-branch, entropy-based adaptation, not being limited to images.

\textit{(2) Advantages of model collapse prevention (model-, test data stream- \& modality-agnostic):} REM prevents collapse by depending on the assumption that masked images should exhibit higher prediction entropy, which develops architecture-specific masking strategies for the vision transformer equipped with a classification token and explicitly fixes the rank of prediction entropy of masked inputs. However, the assumption and the masking strategy may fail to generalize to broader architectures and modalities, while the objective can still be minimized with constant one-hot outputs. 
TTE prevents collapse by subtracting a moving average center in the predicted logit for self-alignment, which explicitly reduces center dominance in a straightforward manner. However, it assumes that the prediction center truly reflects the trends of overfitting, which, however, violates under a long-tail data stream where the center actually reflects the dominant class. In this case, when subtracting the center, TTE can significantly risk flipping correct predictions and producing misguided alignment. 
In contrast, \methodname develops an assumption-free and theoretically grounded asymmetric architecture that self-induces collapse-resistant alignment to disfavor shortcuts in entropy minimization without relying on heuristic thresholds, making it generally applicable across diverse models, data streams, and modalities (Tables~\ref{tab:imagenet-c-mix-shift}-\ref{tab:imagenet-c-mild}) and requiring fewer hyperparameters for deployment.

\textit{(3) Unique strengths of adaptive bias filtration (improving generalization):} Beyond the widely applicable collapse prevention mechanism in \methodname compared to prior methods, we also empirically and theoretically demonstrate that \methodname further improves generalization with a bias learning signals filtration mechanism during TTA. Specifically, we show that the predictor in the online branch purposely absorbs and converts biased shortcut signals (\eg, norm inflation) into explicit discrepancies, where biased signals are then penalized by the alignment loss. As shown in Theorem 1 and Figure 2 (b-c), any shortcut representable by the predictor is naturally suppressed, while useful signals flow into the backbone encoder update. This bias filtration effect even helps improve the generalization of a large language model during test-time reasoning incentivization, demonstrating a promising research direction to enhance future TTA methods.

Finally, we provide further empirical comparisons on both TTA performance and efficiency. As shown in Table~\ref{tab:performance_efficiency_with_multi_branch}, existing methods improve performance at a substantial cost of computation and memory efficiency, due to using additional augmentations and branches, \eg, SPA increases the memory of Tent from 7,760MB to 28,147MB (which cannot be fitted into an RTX 3090), and the latency from 113s to 275s on ViT-Base. In contrast, \methodname is the only method that maintains almost the same memory and latency as Tent, while achieving the highest stability, \eg, increasing the accuracy of SPA from 60.3\% to 64.5\%, while reducing the memory from 28,147MB to 7,806MB on ViT-Base, which reveals a novel and practical stability–efficiency Pareto frontier that has not been achieved in prior TTA methods. 

\begin{table}[h]
    % \vspace{-0.02in}
    \caption{Comparisons on performance and efficiency with multi-branch TTA methods. Accuracy is reported under the blind-spot adaptation scenario. Efficiency is measured by processing 50,000 images under a batch size of 64 via an A100 GPU.}
    % \setlength{\tabcolsep}{3pt}
     % \vspace{-0.06in}
    \label{tab:performance_efficiency_with_multi_branch}
\newcommand{\tabincell}[2]{\begin{tabular}{@{}#1@{}}#2\end{tabular}}
 \begin{center}
 \begin{threeparttable}
    \resizebox{0.9\linewidth}{!}{
 	\begin{tabular}{l|ccc|ccc}
    \multicolumn{1}{c}{} & \multicolumn{3}{c}{ResNet50-GN} & \multicolumn{3}{c}{ViT-Base}\\
 	 % Method & No adapt & $k=1$  & $k=3$  & $k=5$  & $k=1$  & $k=3$  & $k=5$\\
    Method & Acc. (\%) & Mem. (MB)  & Time (s)  & Acc. (\%) & Mem. (MB)  & Time (s) \\
 	\midrule
    Tent & 13.4	& 5,519 & 89 & 30.7 & 7,760 & 113 \\
    TTE & 40.8 & 6,303 & 136 & 55.2 & 8,574 & 164 \\
    REM & - & - & - & 56.9 & 27,438 & 346 \\
    SPA & 45.0 & 21,327 & 220 & 60.3 & 28,147 & 275 \\
    \rowcolor{black!8} \methodname (ours) & 46.2 & 5,584 & 89 & 64.5 & 7,806 & 113 \\
    \end{tabular}
	}
	 \end{threeparttable}
	 \end{center}
\end{table}

\section{More Discussions with Reinforcement Learning}
Entropy collapse is also a common challenge within reinforcement learning, where policies can degenerate (\eg, become deterministic and lose diversity) due to biased or noisy advantage estimation, insufficient exploration, or overly aggressive updates~\citep{mnih2016asynchronous,haarnoja2018soft,ahmed2019understanding}. To address this, trust-region methods such as TRPO~\citep{schulman2015trust} and PPO~\citep{schulman2017proximal} constrain each update step in policy space through KL-divergence penalties or clipping, ensuring gradual and stable improvement. Similar principles appear in value-based methods~\citep{mnih2015human,schwarzer2023bigger,nauman2024bigger}, where target networks provide a stable anchor to stabilize bootstrapped updates against rapidly changing estimates. These approaches share with \methodname the key idea of structurally constraining updates to counter collapse. Specifically, \methodname's online branch with gradient updates resembles a fast-updating policy (\ie, with extra updates in the predictor), while the stop-gradient target branch serves as an anchor, effectively forming a lightweight trust region in the entropy minimization landscape. This connection situates \methodname as a general collapse-prevention mechanism, linking stabilization strategies across TTA, self-supervised learning, and reinforcement learning.

% \section{Future Works}

% \textbf{\methodname for Saving a Collapsed Model}

% We use the large language model to aid writing—preliminarily to check grammar errors and refine the fluency of expression—none of which we feel must be specifically highlighted.

\section{Large Language Model Usage Statement}
In accordance with the ICLR 2026 policy on the responsible use of LLMs, we confirm that our study did not use any LLM to generate scientific content or conduct substantive experiments. The only use of an LLM (ChatGPT-5) was to polish the English writing and improve presentation quality; all core methodology, experiments, and analyses were authored and verified by the human authors.

\end{document}